\documentclass[twoside,11pt]{article}
\usepackage{authblk}

\author[1]{Burak Var\i c\i}
\author[3]{Karthikeyan Shanmugam \thanks{The project was started when the author was at IBM Research, Thomas J. Watson Research Center, Yorktown Heights, NY.}}
\author[2]{Prasanna Sattigeri}
\author[1]{Ali Tajer}

\affil[1]{Rensselaer Polytechnic Institute}
\affil[2]{MIT-IBM Watson AI Lab}
\affil[3]{Google Research India}

\usepackage{ISG_style_arxiv}
\usepackage[utf8]{inputenc} 
\usepackage[T1]{fontenc}    
\usepackage{hyperref}       
\usepackage{url}            
\usepackage{booktabs}       
\usepackage{amsfonts,amssymb}       
\usepackage{nicefrac}       
\usepackage{microtype}      
\usepackage{amssymb}
\usepackage{times}
\usepackage{graphicx}
\usepackage{subcaption}
\usepackage{bbm}
\usepackage{dsfont}
\usepackage{cite}
\usepackage{flushend}
\usepackage{float}
\usepackage[dvipsnames]{xcolor}
\usepackage{algorithm}
\usepackage{algpseudocode}
\usepackage{mathtools}
\usepackage{aligned-overset}
\usepackage{natbib}

\DeclareMathOperator*{\argmax}{arg\,max}

\DeclareMathOperator*{\argmin}{argmin}

\def \C {{\rm c }} 

\newcommand{\norm}[1]{\left\lVert#1\right\rVert}
\newcommand{\smin}[1]{\sigma_{\min}\left(#1\right)}
\newcommand{\smax}[1]{\sigma_{\max}\left(#1\right)}
\newcommand{\lmin}[1]{\lambda_{\min}\left(#1\right)}
\newcommand{\lmax}[1]{\lambda_{\max}\left(#1\right)}

\newcommand{\lminn}[2]{\lambda_{\min}^{#2}\big(#1\big)}
\newcommand{\lmaxx}[2]{\lambda_{\max}^{#2}\big(#1\big)}
\newcommand{\sminn}[2]{\sigma_{\min}^{#2}\big(#1\big)}
\newcommand{\smaxx}[2]{\sigma_{\max}^{#2}\big(#1\big)}

\newcommand{\kl}{\;\|\;}

\newcommand{\inner}[2]{\left\langle #1 ,\; #2 \right\rangle}

\newcommand{\floor}[1]{\left \lfloor #1 \right \rfloor}
\newcommand{\algonameTS}{LinSEM-TS}
\newcommand{\algonameUCB}{LinSEM-UCB}

\begin{document}

\title{Causal Bandits for Linear Structural Equation Models}
\date{}

\maketitle

\begin{abstract}
This paper studies the problem of designing an optimal sequence of interventions in a causal graphical model to minimize cumulative regret with respect to the best intervention in hindsight. This is, naturally, posed as a causal bandit problem. The focus is on causal bandits for linear structural equation models (SEMs) and soft interventions. It is assumed that the graph's structure is known and has $N$ nodes. Two linear mechanisms, one soft intervention and one observational, are assumed for each node, giving rise to $2^N$ possible interventions. Majority of the existing causal bandit algorithms assume that at least the interventional distributions of the reward node's parents are fully specified. However, there are $2^N$ such distributions (one corresponding to each intervention), acquiring which becomes prohibitive even in moderate-sized graphs. This paper dispenses with the assumption of knowing these distributions or their marginals. Two algorithms are proposed for the frequentist (UCB-based) and Bayesian (Thompson Sampling-based) settings. The key idea of these algorithms is to avoid directly estimating the $2^N$ reward distributions and instead estimate the parameters that fully specify the SEMs (linear in $N$) and use them to compute the rewards. In both algorithms, under boundedness assumptions on noise and the parameter space, the cumulative regrets scale as $\tilde{\cal O} (d^{L+\frac{1}{2}} \sqrt{NT})$, where $d$ is the graph's maximum degree, and $L$ is the length of its longest causal path. Additionally, a minimax lower of $\Omega(d^{\frac{L}{2}-2}\sqrt{T})$ is presented, which suggests that the achievable and lower bounds conform in their scaling behavior with respect to the horizon $T$ and graph parameters $d$ and $L$.
\end{abstract}

\section{Introduction}\label{sec:introduction}
Multi-armed bandit (MAB) settings provide a rich theoretical context for formalizing and analyzing sequential experimental design procedures. Each arm represents one experiment in a canonical MAB setting, the stochastic outcome of which is represented by a random reward. The objective of a learner is to design a sequence of arm selections (i.e., experiments) that maximizes the cumulative reward over a time horizon. Bandit problems have a growing list of applications in various domains such as marketing campaigns \citep{Sawant2018ContextualMB}, clinical trials \citep{liu2020reinforcement}, portfolio management \citep{shen2015portfolio}, recommender systems \citep{zhou2017large}. Various assumptions on the statistical dependencies among the arms have led to different classes of problems, such as linear bandits \citep{abbasi-yadkori2011,dani2008linear,rusmevichientong2010linearly}, combinatorial bandits \citep{cesa2012combinatorial}, and contextual bandits~\citep{tewari2017ads}. In this paper, we focus on \emph{causal bandits}~\citep{lattimore2016causal} as another instance of a bandit setting in which reward models are assumed to be the results of causal relationships.

Causal Bayesian networks represent the cause-effect relationships through a directed acyclic graph (DAG). Each node of the causal graph represents a random variable, the edges encode the statistical dependencies among them, and the directed edges signify causal relationships. Recently, the effectiveness of DAGs in encoding non-trivial dependencies among random variables has led to an interest in the study of causal bandits. In a causal bandit framework, \emph{interventions} on the nodes of a DAG are modeled as arms, and the post-intervention stochastic observations are modeled as the arm rewards. Instead of a cumulative reward, generally, the observation of one node is regarded as the reward value. This node is usually selected to be one without any descendants. The structure of the underlying DAG induces statistical dependencies among the arms' rewards and their sequential selection. Furthermore, different interventions on different subsets of nodes result in distinct reward distribution for the reward node. In such a causal bandit framework, the objective is to minimize the cumulative regret with respect to the best intervention in hindsight. 

Causal bandits are effective for modeling several complex real-world problems. For instance, in drug discovery, dosages of multiple drugs can be adjusted adaptively to identify a desirable clinical outcome \citep{liu2020reinforcement}. Likewise, in advertisements campaigns, an advertisement approach can adaptively adjust its strategies that can be modeled as interventions on their advertisement system to maximize their return on advertisement investment~\citep{lu2020regret,lu2021causal,nair2021budgeted}. In each of these applications, several variables affect the observed reward. In parallel, these variables also have cause-effect relationships among them. Hence, an optimal design of experiments in such settings often necessitate simultaneously performing interventions on multiple variables. \vspace{.1 in}

\noindent\textbf{Literature review.} Designing a causal bandit problem hinges critically on the extent of assumptions on the topology of the underlying DAG and the probability distribution of its random variables. The existing literature can be categorized based on different combinations of these assumptions. 

The majority of the existing literature assumes that both the topology and the distributions of the reward node's parents are fully specified under all possible intervention models~\citep{lattimore2016causal,sen17,lu2020regret,nair2021budgeted}. Among these works, the initial studies posed the causal bandit problem as best arm identification, where the learner does not incur a regret for exploration~\citep{bareinboim2015bandits,lattimore2016causal,sen17}. The impact of exploration was further accounted for by \citet{lu2020regret}. Specifically, \citet{lu2020regret} proposed algorithms that capitalize on the causal information to improve the cumulative regrets compared to the algorithms that do not use this information. \citet{nair2021budgeted} extended the previous work by providing an instance-dependent regret bound that, in the worst case, grows with the number of interventions provided. The shared assumption of these works is the a priori knowledge of interventional distributions of parents of the reward node. This is a rather strong assumption since the number of such distributions can be prohibitive. \textbf{In contrast, we assume that the interventional distributions, or their marginals, are unknown.}

The setting without the knowledge of the topology and the structure of the interventional distributions has been studied more recently \citep{de2022causal,lu2021causal,bilodeau2022adaptively}. However, the algorithm of \citet{de2022causal} relies on auxiliary separating set algorithms, and its improvement upon non-causal algorithms is shown only empirically. \citet{lu2021causal} provide an improved regret result for a similar setting, though only for the atomic interventions. \citet{bilodeau2022adaptively} consider a different case and establish adaptive regret guarantees with respect to an unknown causal structure. They assume access to an estimate of the interventional distributions, which can be learned offline. Each possible arm is played a certain number of times to ensure that these estimates are sufficiently accurate. Notably, no assumptions are made about the structure of the distributions. Furthermore, their proposed adaptive algorithm is agnostic to the causal structure and incurs sublinear regret even in the presence of unobserved confounders. The main distinction of our work compared to \citet{bilodeau2022adaptively} is that their focus is on the adaptivity of the regret guarantees, and their algorithm's regret upper bound grows with the cardinality of the intervention space.

One setting that lies in between the previous two settings assumes that the topology is known while the statistical models are unknown. This setting was first studied by \citet{yabe18causal}, and they obtained a simple regret guarantee that scales polynomially in the graph size. However, the scope of this study is limited to binary random variables and the best arm identification objective. \citet{maiti2022causal} propose an algorithm that takes a semi-Markovian graph as input and achieves a simple regret that is almost optimal for a certain class of graphs. They also propose an algorithm that achieves an improved cumulative regret compared to the non-causal bandit algorithms when all variables are observed. The setting, however, is focused on atomic interventions on binary random variables. \citet{xiong2022pure} propose adaptive algorithms with gap-dependent sample complexity results. However, they focus on atomic interventions and their main result is for parallel graphs in which all the non-reward nodes are parents of the reward node without other causal relationships. Their extension to more general graphs is limited to binary generalized linear models. Unlike \citet{maiti2022causal} and \citet{xiong2022pure}, in which the cardinality of the intervention space is at most $N$, in this paper, the intervention space consists of all subsets of the graph's nodes with cardinality $2^N$. \citet{feng2022combinatorial} posed the non-atomic interventions as a combinatorial bandit problem of finding the best set of nodes to intervene with a pre-specified cardinality. They consider only binary random variables similar to \citet{yabe18causal} and \citet{maiti2022causal}.
Furthermore, they assume that the causal mechanisms follow a generalized linear model with $\{0,1\}$ Boolean variables. In their paper, they consider only hard interventions, unlike our case. Therefore, the Gram matrix associated with their local least-squares estimation is always about estimating the observational causal mechanism. In our case, due to the soft intervention model, eigenvalues of the Gram matrices for estimating the different interventional mechanisms are not straightforward to control, given that one needs to account for other simultaneous interventions. In this paper, we assume that only the DAG topology is known, which is a setup similar to that of \citet{yabe18causal} and \citet{feng2022combinatorial}. Nevertheless, we have significant differences in the intervention model, the space of interventions, and the causal mechanisms assumed.

\paragraph{Motivation for soft interventions.} We note that most of the existing causal bandit algorithms assume deterministic hard interventions, i.e., $do()$ interventions, for minimizing cumulative regret. The exceptions include the study by \citet{sen17}, which considers soft interventions on a single node for best arm identification. \citet{yabe18causal} remark that it can be extended to Boolean variables too. For minimizing cumulative regret under unknown interventional distributions, which is the goal of this paper, the only exception is the adaptive algorithm of \citet{bilodeau2022adaptively}.

In hard interventions, the variables are forced to take fixed values, and parental effects are completely absent. In contrast, soft interventions, which modify the conditional distributions of the target variables, can be more realistic in certain applications and have found recent interest in causal inference literature \citep{jaber2020causal,varici2021scalable}. Hence, in this paper we focus on soft interventions. In addition, our results can be readily extended to hard interventions, which we discuss in Appendix~\ref{sec:hard}.

A motivating example of soft interventions is the computational advertising studied by \citet{bottou2013counterfactual}. The elements of an advertising system (e.g., query, prices, and click rates) and their interactions can be modeled by a causal graph with a known structure, and the reward variable accounts for the resulting revenue. However, the strengths of the interactions between the nodes are unknown. Interventions are a set of click-rate prediction algorithms. Figure~\ref{fig:soft-motivation} illustrates the described model. In this causal system, the interventions are not hard since they do not remove the causal effects.

\begin{figure}[t]
    \centering
    \includegraphics[width=0.7\linewidth]{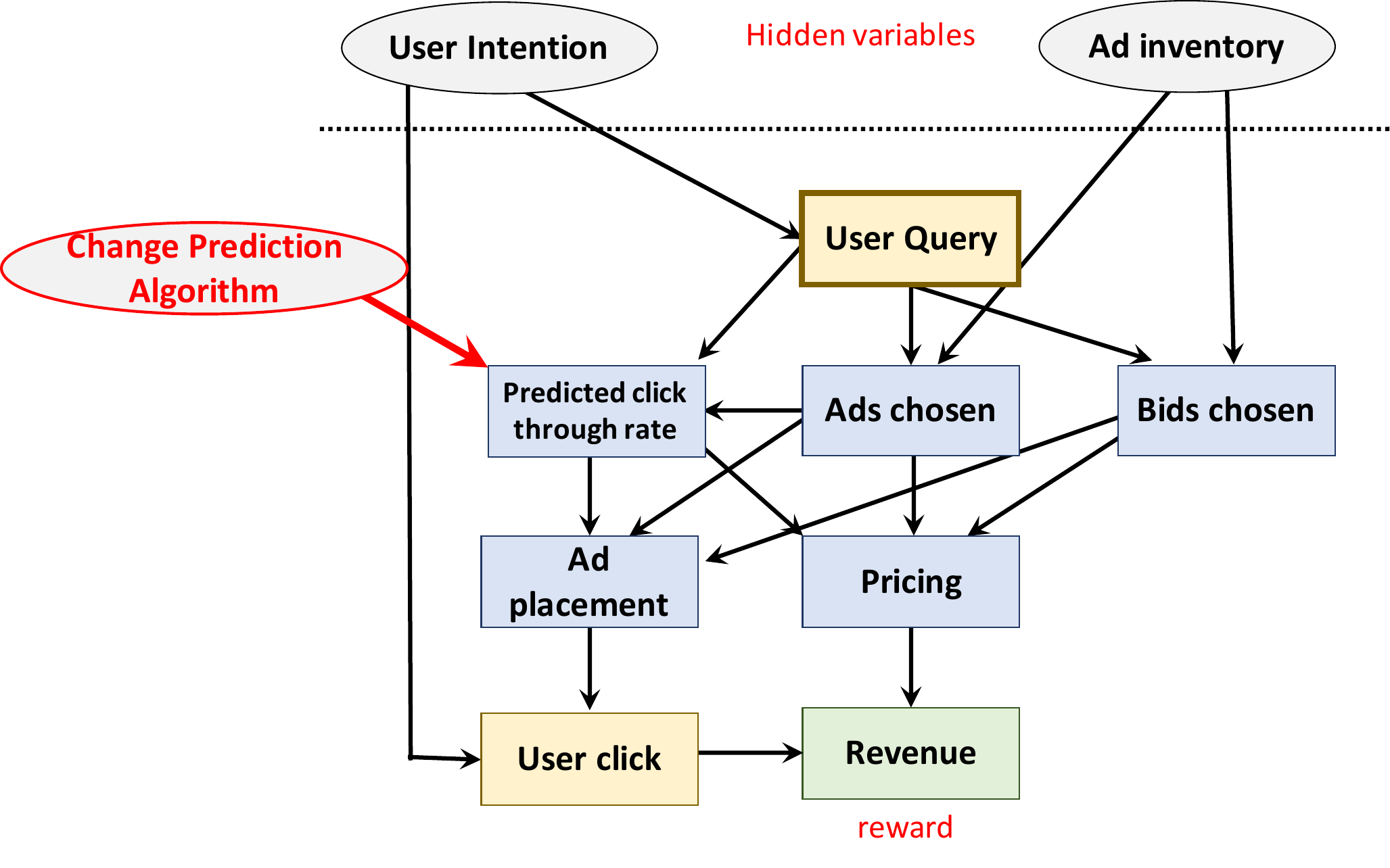}
    \caption{A computational advertising example borrowed from \citep{bottou2013counterfactual}. The interventions are changes in the prediction algorithm for the click rate. Since they do not remove any edges, they are soft interventions.}
    \label{fig:soft-motivation}
\end{figure}

\vspace{.1 in}
\noindent\textbf{Contributions.} Motivated by the limitations of the existing works, our objective is to answer the following question: \textbf{Can we use only the graph structure without the knowledge of interventional distributions, allow soft interventions on continuous variables, and obtain cumulative regret guarantees that scales optimally with the horizon ($\sqrt{T}$) while not growing with the cardinality of the intervention space?} We answer this question affirmatively when the causal system follows a linear structural equation model (SEM) with unknown parameters and the interventions inherit that structure.

We consider per-node interventional mechanisms, giving rise to~$2^N$ possible intervention models. Standard linear bandit algorithms also are inapplicable as feature vectors of arms are unknown in our case. We propose the {\bf Lin}ear {\bf SEM} {\bf U}pper {\bf C}onfidence {\bf B}ound the (\algonameUCB{}) algorithm that achieves the cumulative regret $\tilde{\cal O} (d^{L+\frac{1}{2}} \sqrt{NT})$ in which $d$ is the maximum degree and $L$ is the length of the longest causal path in the graph. Importantly, the number of interventions $2^N$ does not appear in our regret result. Instead, we have $d^{L+\frac{1}{2}}$ term that depends on the topology. We also propose the {\bf Lin}ear {\bf SEM} {\bf T}hompson {\bf S}ampling (\algonameTS{}) algorithm for the Bayesian setting, and its Bayesian regret scales similar to that of \algonameUCB{}, i.e., $\tilde{\cal O} (d^{L+\frac{1}{2}} \sqrt{NT})$. Finally, we establish a regret lower bound that is a constant factor of $d^{\frac{L}{2}-2}\sqrt{T}$, which matches the behavior of our achievable regret $\tilde \mcO(d^{L+\frac{1}{2}}\sqrt{NT})$ in terms of the horizon $T$ and graph parameters $d$ and $L$.

\vspace{.1 in}
\noindent\textbf{Key technical challenges.} 
The central piece of our analysis pertains to characterizing the hitherto uninvestigated cumulative regret when we have continuous variables, an intervention space that is the power set of the graphs' nodes, and unknown interventional distributions. This analysis faces a number of challenges while developing bandit algorithms with regret guarantees. First, the linear SEM, i.e., weight matrices and the noise model, are unknown. Since these pieces of information are needed for specifying the interventional distributions, not knowing them renders these distributions unknown. A direct (naive) approach to regret analysis involves estimating all the unknown distributions, the complexity of which grows quickly. Motivated by the fact that estimating $2^{N}$ interventional distributions explicitly bears significant redundancy, to circumvent the computational challenge, we start by expressing the reward as a function of the intervention's weight matrix and the random noise vector, which is independent of the action. We note that the measurements from the nodes are observed \emph{only after} an intervention is applied. Hence, it is important for these noise vectors to be independent of the actions to reduce the dependency on interventions to the weight matrices. This relationship is subsequently leveraged to estimate the unknown weight matrices containing at most $2Nd$ parameters. Finally, these estimates are leveraged to estimate the rewards.

A second challenge is due to the non-linearity of the relationship between the reward values observed and the unknown weight matrices. Specifically, even though the causal model is linear, interventions have a compounding effect on the mean values of the nodes inducing non-linearity. This renders the problem fundamentally different from the linear bandit problems. The third challenge is due to applying soft interventions. Under soft interventions, unlike hard interventions, the effects of the intervention targets' ancestors on the reward node are not nullified. In addition, having continuous variables (instead of binary, e.g., \citet{yabe18causal,feng2022combinatorial}) necessitates a careful examination of the compounded effect of the estimation errors through the causal paths. 

Our regret analysis naturally decomposes the total regret into two separate terms, one accounting for the effect of the topology through the maximum degree and the longest causal path of the graph, and the other term captures the effect of the intervention space and the underlying data distributions. For establishing lower bounds, we construct two bandit instances that differ by just one edge but have very different expected rewards due to the aggregation of the difference in one edge through $L$ layers.

\vspace{.1 in}
\noindent\textbf{Organization.} The rest of the paper is organized as follows. Section \ref{sec:problem_setup} introduces the graphical model, intervention model, and the causal bandit model that is built on them. Performance measures and the quantities that will affect the regret guarantees are also defined. Section \ref{sec:methodology} outlines our approach to the problem, and establishes the estimation procedure. Section \ref{sec:UCB} develops an upper confidence bound-based algorithm and regret guarantees. Section \ref{sec:TS} develops a Thompson Sampling-based algorithm and benefits from the results developed in Section \ref{sec:UCB} to present a similar regret result. Also, a more practical algorithm is proposed in Section \ref{sec:TS} that is used in the numerical studies in Section \ref{sec:simulations}. Section~\ref{sec:lower-bound} presents a minimax lower bound on the regret which shows that the scaling behavior of our regret guarantees is necessary. The central pieces of the proofs are provided in the main body of the paper, and the rest of the proofs are relegated to the appendix.

\section{Problem Setup}\label{sec:problem_setup}
{\bf Notations.} Vectors are denoted by upper-case letters, e.g., $X$, where $X_{i}$ denotes the $i$-th element of vector $X$. Matrices are denoted by bold upper-case letters, e.g., $\bA$, where $[\bA]_{i}$ denotes the $i$-th column vector of $\bA$ and $[\bA]_{i,j}$ denotes the entry at row $i$ and column $j$ of $\bA$. For a positive integer $N$, we define $[N] \triangleq \{1,\dots,N\}$. Sets and events are denoted by calligraphic letters. For a subset $\mcS \subseteq [N]$, we define $X_{\mcS} \triangleq X\odot\mathbf{1}(\mcS)$, where $\odot$ denotes the Hadamard product and the $N$-dimensional binary vector $\mathbf{1}({\mcS})\in \{0,1\}^N$ is specified such that its elements at the coordinates included in $\mcS$ are set to~1, and the rest are 0. Norm $\norm{\cdot}$ denotes the Euclidean norm when applied on a vector and the spectral norm when applied to a matrix. For the latter, the spectral norm of a matrix is equal to the largest eigenvalue of the matrix. The $\bA$-norm of a vector $X$ associated with the positive semidefinite matrix $\bA$ is denoted by $\norm{X}_{\bA}=\sqrt{X^{\top}\bA X}$.
We denote the singular values of a matrix $\bA \in \mathbb{R}^{M \times N}$, where $M \geq N$,  by
\begin{align}
    \sigma_{1}(\bA) &\geq \sigma_{2}(\bA) \geq \dots \geq \sigma_{N}(\bA) \ .
\end{align}
In this paper, we often work with zero-padded vectors. As a result, the matrices that contain these vectors have non-trivial \emph{null space} leading to zero singular values. In such cases, we use the \emph{effective} smallest singular value that is non-zero. We denote the largest and smallest eigenvalues that correspond to effective dimensions of a positive semidefinite matrix $\bA$ with rank $k$ by
\begin{align}
    \smax{\bA} \triangleq \sigma_{1}(\bA) \ , \quad \mbox{and} \quad \smin{\bA} \triangleq \sigma_{k}(\bA) \ .
\end{align}
Similarly, for a square matrix of the form $\bV = \bA \bA^{\top} \in \mathbb{R}^{N \times N}$, we denote the largest and smallest eigenvalues by
\begin{align}
    \lmax{\bV} &\triangleq \lmax{\bA \bA^{\top}} 
    = \sigma_{\max}^2(\bA) \ , \\
    \mbox{and} \qquad \lmin{\bV} &\triangleq \lmin{\bA \bA^{\top}} 
    = \sigma_{\min}^2(\bA) \ . 
\end{align}
Finally, the notation $\tilde \mcO$ is used to ignore constant and poly-logarithmic factors. We also have a table at Appendix~\ref{appendix:notations} that summarizes commonly used notations throughout the paper.

\vspace{.1 in}
\noindent {\bf Graphical model.} We consider a directed acyclic graph with the known structure $\mcG \triangleq (\mcV,\mcB)$, where $\mcV \triangleq [N]$ is the set of nodes, and $\mcB$ is the set of edges. The vector of random variables associated with the nodes is denoted by $X \triangleq (X_1,\dots,X_N)^\top$. We refer to the set of parents of node $i\in\mcV$ by $\Par(i)$. We denote the maximum degree of the graph by $d \triangleq \{\max_{i} |\Par(i)|\}$, and use $L$ to denote the length of the longest directed path in the graph. For the causal model, we consider a linear SEM, according to which
\begin{align}
    X = \bH^\top X + \nu + \epsilon \ ,  \label{eq:linear_sem_pre} 
\end{align}
where $\bH \in \mathbb{R}^{N\times N}$ is the edge weights matrix, and it is strictly upper triangular, $\nu$ is the constant vector of unknown affine terms, $\epsilon$ accounts for the zero-mean model noise, and it is independent of $X$ and $\nu$. The term $\nu+\epsilon$ can be alternatively viewed as a noise term with an unknown mean value. In this model, $[\bH]_{i}$ captures the weights of the causal effects of $\Par(i)$ on node $i$. Consequently, the edge weight $[\bH]_{j,i}$ is non-zero if and only if $j \in \Par(i)$. The elements of the noise vector $\epsilon\triangleq (\epsilon_1,\dots,\epsilon_N)^\top$ are statistically independent, and are assumed to be 1-sub-Gaussian. The noise vector is assumed to satisfy $\norm{\epsilon}\leq m_\epsilon$, where $m_\epsilon\in \mathbb{R}^{+}$ is an unknown constant.

Next, we re-arrange the terms in \eqref{eq:linear_sem_pre} such that the unknowns terms, i.e., $\bH$ and $\nu$, can be presented more compactly. We create a \emph{dummy} node $X_0 = 1$ and an associated dummy noise term $\epsilon_0 = 1$. Accordingly, we augment $X$ and $\epsilon$ and redefine them as $X \triangleq (1,X_1,\dots,X_N)^\top$ and $\epsilon\triangleq (1,\epsilon_1,\dots,\epsilon_N)^\top$, respectively. Hence, for any node $i \in [N]$, the dummy node $0$ acts as a parent with edge weight $\nu_i$. Subsequently, we create the matrix $\bB \in \mathbb{R}^{(N+1)\times (N+1)}$, in which we use $0$ to refer to the index of the dummy node, i.e.,
\begin{align}
    [\bB]_{i,j} = \begin{cases}
        [\bH]_{i,j} &\ , \quad \forall i \in [N], \ \forall j \in [N] \ , \\
        \nu_i &\ , \quad i = 0 , \ \forall j \ \in [N] \ , \\
        0 &\ , \quad j = 0 , \ \forall i \ \in [N] \cup \{0\} \ .
    \end{cases} 
\end{align}
Hence, \eqref{eq:linear_sem_pre} becomes
\begin{align}
    X = \bB^\top X + \epsilon \ .  \label{eq:linear_sem} 
\end{align}
Next, note that the constant $X_{0}=1$ acts as a parent for all nodes. Hence, we expand $\Par(i)$ with the dummy node and define 
\begin{align}
\Pa(i) &\triangleq \Par(i) \cup \{0\}  \ , \quad \forall i \in [N]\ .
\end{align}
Therefore, the effective maximum degree becomes $d+1$, the effective longest path-length becomes $L+1$, the weight between the dummy node $0$ and node $i$ is $\nu_i=[\bB]_{0,i}$.

\vspace{.1 in}
\paragraph{Soft intervention model.} The conditional distribution of $X_i$ given its parents is denoted by $\P(X_i | X_{\Par(i)})$. A \emph{soft intervention on node $i$} refers to an action that induces a change in the conditional distribution $\P(X_i | X_{\Par(i)})$. An intervention can impact one or more nodes simultaneously, and we distinguish the intervention actions based on the set of nodes that they impact. We denote the intervention space by $\mcA \triangleq 2^{\mcV}$.  

To formalize the impact of our soft interventions, we model the effect of a soft intervention on node $i$ as an alteration in $[\bB]_{i}$ such that $[\bB]_{i}$ changes to $[\bB^{*}]_{i} \neq [\bB]_{i}$. We define $\bB^{*}$ as a weight matrix formed by columns $\{[\bB^{*}]_{i}: i \in \mcV\}$. We refer to $\bB$ and $\bB^{*}$ as the observational and interventional weight matrices, respectively. We note that the changes can occur on either the weights $[\bH]_{i}$ or the affine term $\nu_i$, and this covers all possible soft interventions on linear models. The main difference between soft and hard interventions is that under a deterministic hard intervention on node $i$, we have a zero vector $[\bB^{*}]_{i}$, and $\epsilon_i$ is not random anymore but a constant value assigned by the intervention. Under a soft intervention, $[\bB^{*}]_{i}$ can take non-zero values and $\epsilon$ remains a random vector.

Since an intervention can impact multiple nodes simultaneously, each intervention action imposes changes in multiple columns of $\bB$. This leads to a distinct linear SEM model for each possible intervention action $a \in \mcA$. To capture the impact of a specific intervention action $a \in \mcA$, we define $\bB_{a}$ as its corresponding post-intervention weight matrix. $\bB_{a}$ is constructed according to:
\begin{align} \label{eq:Ba_construct}
    [\bB_{a}]_{i} = \mathbbm{1}_{\{i \in a\}} [\bB^{*}]_{i} + \mathbbm{1}_{\{i \notin a\}} [\bB]_{i} \ .
\end{align}
We use $\P_{a}$ to denote the distribution of the random variables under intervention $a \in \mcA$ induced by linear SEMs in \eqref{eq:linear_sem} with weight matrix $\bB_{a}$. For given matrices $\bB$ and $\bB^{*}$ we define
\begin{align}
    m_B\triangleq \max_{i\in\mcV,a\in\mcA}\{\norm{[\bB_{a}]_{i}}\}\ ,
\end{align}
Since the maximum degree of a node~$i$ is $d$ and we have augmented the weight vectors with $\nu_i$'s, we have $\norm{[\bB_{a}]_{i}}_0 \leq d+1$. The boundedness of $\|[\bB_{a}]_{i}\|$ and $\|\epsilon\|$ in conjunction with the SEM in~\eqref{eq:linear_sem} imply that there exists a constant $m\in\R^+$ such that $\|X\|\leq m$. We formalize the boundedness assumptions as follows.

\begin{assumption}[Boundedness]\label{assumption:boundedness}
For given matrices $\bB$ and $\bB^{*}$, $\max_{i\in\mcV,a\in\mcA}\{\norm{[\bB_{a}]_{i}}\} \leq m_B$. Furthermore, $X$ satisfies $\norm{X}\leq m$ for some known $m \in \mathbb{R}^{+}$.
\end{assumption}

\paragraph{Causal bandit model.} In the setting described, a learner performs a sequence of repeated interventions. Each intervention is represented by an arm, including the null intervention (i.e., a pure observation). The set of interventions can be abstracted by a multi-arm bandit setting with $2^N$ arms, one arm corresponding to one specific intervention. Following the related literature, e.g., \citep{lattimore2016causal,sen17,lu2020regret,nair2021budgeted,yabe18causal}, we assume the structure of the graph $\mcG$ is given, and without loss of generality, we designate node $N$ as the reward node and accordingly $X_N$ as the reward variable. Similar to linear bandits, the reward $X_N$ is a linear function of some other variables, namely the parents of the reward node, $X_{\Pa(N)}$. These parent variables $X_{\Pa(N)}$, in turn, depend on their causal ancestors according to the linear SEM in~\eqref{eq:linear_sem}. Even though each of the parents contributes linearly to its immediate descendants, their compounding effects induce a non-linearity in the overall model such that $X_N$ varies non-linearly with respect to the entries of $\bB$ and $\bB^{*}$. We denote the expected reward based on intervention (action) $a \in \mcA$ by 
\begin{align}
     \mu_a \triangleq \E_{a}[X_N] \ , \label{eq:expected_reward}
\end{align}
where $\E_{a}$ denotes expectation under $\P_a$. Accordingly, we define the optimal action $a^*$ as 
\begin{align}
    a^* &\triangleq \argmax_{a \in \mcA} \mu_a \ . \label{eq:optimal_action}
\end{align}
The sequence of interventions over time is denoted by $\{a_t\in\mcA: t\in\mathbb{N}\}$. Upon intervention $a_t$ in round $t$, the learner observes $X(t) \triangleq (1,X_1(t),\dots, X_N(t))^\top \sim \P_{a_t}$ and collects the reward $X_N(t)$. We denote the random noise vector in round $t$ by $\epsilon(t) \triangleq (1,\epsilon_1(t),\dots,\epsilon_N(t))^\top$, which is independent of the intervention $a_t$. The learner does not know the observational or interventional matrices $\bB$ and $\bB^{*}$. More importantly, the interventional distributions of the parents of the reward variable, i.e., $\{\P_{a}(\Pa(N)):a\in\mcA\}$ are unknown. This is in contrast to the earlier studies in~\citep{lattimore2016causal,sen17,lu2020regret,nair2021budgeted,bilodeau2022adaptively}, which assume that these distributions are known.

Let us denote the second moment of the parents of a node $i$ under intervention $a \in \mcA$ by 
\begin{align}
    \Sigma_{i,a} \triangleq \E_{X \sim \P_{a}}\left[X_{\Pa(i)} X_{\Pa(i)}^\top\right]\ . \label{eq:second_moment_definition}
\end{align}
This second moment matrix is a function of intervention $a$, and the unknown weight matrices $\bB$ and $\bB^{*}$. Accordingly, we denote the lower and upper bounds on the minimum and maximum singular values of these moments by
\begin{align}
    \kappa_{i,\min} &\triangleq \min_{a \in \mcA} \smin{\Sigma_{i,a}} \ , \quad  
    \kappa_{\min} \triangleq \min_{i \in [N]} \kappa_{i,\min}  \label{eq:kappa_min} \ , \\
    \kappa_{i,\max} &\triangleq \max_{a \in \mcA} \smin{\Sigma_{i,a}} \ , \quad  
    \kappa_{\max} \triangleq \min_{i \in [N]} \kappa_{i,\max}  \label{eq:kappa_max} \ .
\end{align}
We do not assume to know these moments. Note that having a zero singular value implies that there is a deterministic relationship among the elements of $X_{\Pa(i)}$. However, random variables of a causal model, such as the model described in this section, cannot have such a deterministic relationship. Furthermore, the addition of constant $X_0=1$ does not violate this property since it is the only non-random variable in vector $X$, and any random variable $X_i$ still does not possess a deterministic relationship with the other elements of $X_{\Pa(i)}$. Hence, in our setting we have $\kappa_{\min} > 0$. Furthermore, since $\norm{X}^2 \leq m^2$, we have 
\begin{align}
    0 < \kappa_{\min} \leq \kappa_{i,\min} \leq \kappa_{i,\max} \leq \kappa_{\max} \leq m^2 \ . \label{eq:kappa_m2}
\end{align}
The learner's objective is to design a policy for sequentially selecting the interventions over time so that a measure of cumulative regret is minimized. For this purpose, we define $r(t)\triangleq \mu_{a^*} - \mu_{a_t}$ as the average regret incurred at time $t$. In this paper, we consider the following two canonical expected cumulative regret measures. 
\begin{enumerate}
    \item \textbf{Frequentist regret:} By denoting the cumulative regret up to time $T$ by $R(T)\triangleq \sum_{t=1}^T r(t)$, the expected cumulative regret is given by
    \begin{align}\label{eq:frequentist_regret}
        \E[R(T)] \triangleq T\mu_{a^*} - \E \left[ \sum_{t=1}^T X_N(t) \right] \ .
    \end{align}
    \item \textbf{Bayesian regret:} We define $\bW \triangleq [\bB \ \ \bB^{*}] \in \R^{N \times 2N}$ to capture the entire parameterization of the observational and interventional distributions. We define the domain of $\bW$ by $\mcW \triangleq \{\Theta \in R^{N \times 2N} : \norm{[\Theta]_{i}}\leq 1, \ \ \forall i \in [2N]\}$. By denoting the cumulative regret associated with~$\bW$ up to time $T$ by $R_{\bW}(T)$, the Bayesian regret is given by
    \begin{align}\label{eq:bayesian_regret}
        {\rm BR}(T) \triangleq \E_{\mcW} \E_{\epsilon} [R_{\bW}(T)]\ ,
    \end{align}
    where $\E_{\mcW}$ is the expectation with respect to the Bayesian prior over $\mcW$.
\end{enumerate}

\section{Causal Bandits Methodology}\label{sec:methodology}

Before presenting the algorithms and performance guarantees, we note that despite similarities to the conventional bandit problems and the existing causal bandits' literature, treating our problem necessitates a distinct approach to algorithm design and analysis. As discussed in Section \ref{sec:introduction}, this is primarily due to the lack of information about the joint distribution of parents under different interventions  $\{\P_a(\Pa(N)):a\in\mcA\}$. Naively, the moments of these distributions can be estimated from the data. However, it becomes prohibitive even for moderate values of $N$, which gives rise to $2^N$ possible interventions.

We take a different approach and use the fact that reward variable $X_N$ under intervention $a$ is a function of the weight matrix $\bB_{a}$ and the noise vector $\epsilon$ (which is independent of $a$). We show that the reward $X_N$ is a linear function of the entries of $\epsilon$, and the coefficients of the noise terms are non-linear functions of the entries of $\bB_{a}$. Hence, causal relationships are captured by these coefficients. We compute these coefficients by solving each linear problem $X_i = X^\top [\bB_{a}]_{i} + \epsilon_i$ separately, and combining the estimates for $[\bB_{a}]_{i}$ according to the graph structure. We show how these steps work, propose algorithms, and provide their regret analysis in the rest of the paper.

\paragraph{Reward Modeling.} In linear SEMs, each random variable $X_i$ can be specified as a linear function of the exogenous noise variables $\epsilon$ via recursive substitution of the structural equations. This can easily be seen by rearranging \eqref{eq:linear_sem} to obtain $X = (I_{N+1}-\bB^{\top})^{-1}\epsilon$. The second observation is that, $ (I_{N+1}-\bB^{\top})^{-1}$ has a simple expansion since $\bB$ is strictly upper triangular. Specifically,
\begin{align}
     \left(I_{N+1}-\bB^{\top}\right)^{-1} &= (I_{N+1}-\bB)^{-\top} \\
     &= \left(I_{N+1} + \sum_{i=1}^\infty \bB^i \right)^{\top} \\
     & = \left(\sum_{\ell=0}^{L+1} \bB^{\ell} \right)^{\top} \ , \label{eq:expansion}
\end{align}
where the last equality holds since $B^{\ell}$ becomes a zero matrix for $\ell \geq L+2$. Finally, note that the entry $[\bB^{\ell}]_{i,j}$ is the sum of the weighted products along $\ell$-length directed paths from node $j$ to node $i$. Hence, using \eqref{eq:expansion}, the multiplier of $\epsilon_i$ in the expansion of $X_N$ has a simple description: the sum of the products of weights along a path that traces from an upstream node $i$ and ends at the reward node $N$. The following lemma characterizes this linear relationship between the reward node $X_N$ and the noise variables $\epsilon$. The relationship is further simplified for the expected reward.

\begin{lemma}\label{lm:expected_reward}
Consider the linear SEM associated with intervention $a$ with weight matrix $\bB_{a}$. The reward $X_N$ is related to the noise vector $\epsilon$ via 
\begin{align}
    X_N &= \sum_{\ell=0}^{L+1} \inner{\big[\bB_{a}^{\ell}\big]_{N}}{\epsilon} \ ,
\end{align}
in which, $L$ is defined as the length of the longest path in $\mcG$.
Furthermore, since $\{\epsilon_i:i\in [N]\}$ have zero mean values and $\epsilon_{0}=1$, the expected reward under intervention $a$ is
\begin{align}
    \mu_a & = f(\bB_{a}) \triangleq \sum_{\ell=1}^{L+1} \big[\bB_{a}^{\ell}\big]_{0,N}  \ .
\end{align}
\end{lemma}
\begin{proof}
    The first result immediately follows from \eqref{eq:expansion} as follows:
    \begin{align}
        X \overset{\eqref{eq:linear_sem}}&{=} (I_{N+1}-\bB_a^{\top})^{-1} \epsilon \overset{\eqref{eq:expansion}}{=} \left(\sum_{\ell=0}^{L+1} \bB_a^{\ell} \right)^{\top}\epsilon \ , \\
        X_N &= \sum_{\ell=0}^{L+1}\inner{[\bB_a^{\ell}]_N}{\epsilon} \ . \label{eq:lm1-1} 
    \end{align}
    Note that $\epsilon$ and $\bB_a$ are independent, the expectation of each $\epsilon_i$ is $0$ for $i \in [N]$, and dummy noise constant $\epsilon_0=1$. Then, we obtain
    \begin{align}
        \mu_a = \E[X_N] &= \sum_{\ell=0}^{L+1} \E\left[\inner{[\bB_a^{\ell}]_N}{\epsilon}\right] \\
        &=  \sum_{\ell=0}^{L+1} \sum_{i=0}^{N} \left([\bB_a^{\ell}]_{i,N} \E[\epsilon_i] \right) \\
        & = \sum_{\ell=0}^{L+1} [\bB_a^{\ell}]_{0,N} \ . \label{eq:lm1-2}
    \end{align}
    Note that $\bB_a^{0} = I_{N+1}$, and the summand for $\ell=0$ in \eqref{eq:lm1-2} is zero. Hence, by defining $f(\bB_a) = \sum_{\ell=1}^{L+1} [\bB_a^{\ell}]_{0,N}$, we obtain the desired result $\mu_a = f(\bB_a)$.
\end{proof}
Lemma \ref{lm:expected_reward} indicates that, given $\bB_{a}$, the expected reward can be computed from the sum of $L+1$ components. Next, we focus on estimating the weight vectors, which are for constructing the matrices $\{\bB_{a}:a \in\mcA\}$.

\paragraph{Estimating parameter vectors under causal bandit model.} 
We use the ordinary least-squares method to estimate vectors $\{[\bB]_{i}, [\bB^{*}]_{i} : i \in [N] \}$. For estimating $[\bB]_{i}$ and $[\bB^{*}]_{i}$, we should use the data samples from the rounds in which $i$ is non-intervened, and intervened, respectively. Hence, our estimates at time $t$, which are denoted by $\{ [\bB(t)]_{i}, [\bB^{*}(t)]_{i} : i \in [N] \}$, are computed according to:
\begin{align}
 [\bB(t)]_{i} & \triangleq [\bV_{i}(t)]^{-1} \sum_{s \in [t] : i \notin a_{s}} X_{\Pa(i)}(s) X_i(s) \ , \label{eq:estimate_obs} \\
\mbox{and} \qquad   [\bB^{*}(t)]_{i} & \triangleq [\bV^{*}_{i}(t)]^{-1}\sum_{s \in [t] : i \in a_{s}} X_{\Pa(i)}(s) X_i(s) \ , \label{eq:estimate_int}
\end{align}
where we have defined
\begin{align}
\bV_{i}(t) & \triangleq \sum_{s \in [t] : i \notin a_{s}} X_{\Pa(i)}(s) X_{\Pa(i)}^\top(s) + I_{N+1} , \label{eq:define_V_obs} \\
\mbox{and} \qquad  \bV^{*}_{i}(t) & \triangleq\sum_{s \in [t] : i \in a_{s}}  X_{\Pa(i)}(s) X_{\Pa(i)}^\top(s)+I_{N+1} \ . \label{eq:define_V_int}
\end{align}
Note that, first entry of $X_{\Pa(i)}=1$ since dummy node $0$ is contained in $\Pa(i) = \Par(i) \cup \{0\}$ for every node $i \in [N]$. Accordingly, we denote our estimate of $\bB_{a}$ at time $t$ by $\bB_{a}(t)$. These estimates are formed similarly to~\eqref{eq:Ba_construct} and according to:
\begin{align} \label{eq:Bat_construct}
    [\bB_{a}(t)]_{i} \triangleq \mathbbm{1}_{\{i \in a\}} [\bB^{*}(t)]_{i} + \mathbbm{1}_{\{i \notin a\}} [\bB(t)]_{i}  \ .
\end{align}
Note that for estimating $[\bB_{a}(t)]_{i}$, based on whether node $i$ is contained in $a$, we use either the observational data (when $i\notin a$)   or the interventional data (when $i\in a$). Furthermore, at time~$t$, we construct two distinct $t \times (N+1)$ data matrices $\bD_{i}(t)$ and $\bD^{*}_{i}(t)$ from the observational and interventional data, respectively, to store the data from these two cases separately. Specifically, if node $i$ is intervened at time $s \in [t]$, then the $s$-th row of $\bD^{*}_{i}(t)$ stores $X_{\Pa(i)}^{\top}(s)$, and the $s$-th row of $\bD_{i}(t)$ is a zero vector. This construction is reversed when $i$ is not intervened at time~$s$. Hence,
\begin{align}
    \big[\bD_{i}^{\top}(t)\big]_{s} &\triangleq \mathbbm{1}_{\{i \notin a_s\}} X_{\Pa(i)}^{\top}(s) \ , \label{eq:D_it_obs} \\
\mbox{and} \qquad     \big[{\bD^{*}_{i}}^{\top}(t)\big]_{s} &\triangleq \mathbbm{1}_{\{i \in a_s\}} X_{\Pa(i)}^{\top}(s) \ .  \label{eq:D_it_int}
\end{align}
Similarly to \eqref{eq:Ba_construct}, we denote the relevant data and Gram matrices for node $i$ under intervention $a$ by
\begin{align}
    \bD_{i,a}(t) &\triangleq \mathbbm{1}_{\{i \in a\}} \bD^{*}_{i}(t) +  \mathbbm{1}_{\{i \notin a\}} \bD_{i}(t) \ ,  \label{eq:D_ita} \\
 \mbox{and} \qquad    \bV_{i,a}(t) &\triangleq \mathbbm{1}_{\{i \in a\}} \bV^{*}_{i}(t) +  \mathbbm{1}_{\{i \notin a\}} \bV_{i}(t) \ . \label{eq:V_ita} 
\end{align}
The constructions in \eqref{eq:D_it_obs}-\eqref{eq:V_ita} yield 
\begin{align}
    \bV_{i,a}(t) &= \bD_{i,a}^{\top}(t) \bD_{i,a}(t) + I_{N+1} \ . \label{eq:V_ita_D_ita}
\end{align}
Note that matrices $\bV_{i}(t)$ and $\bV^{*}_{i}(t)$ are positive definite with their smallest eigenvalue being at least $1$ due to the regularization constant $I_{N+1}$. Hence, $\lmin{\bV_{i,a}(t)}\geq 1$.

Let $N^{*}_{i}(t)$ denote the number of times that node $i$ is intervened up to time $t$, for each $i \in [N]$ and $t \in [T]$. Similarly, denote the number of times that $i$ is not intervened by $N_{i}(t) = t - N^{*}_{i}(t)$. Formally,
\begin{align} \label{eq:def_N_it}
    N^{*}_{i}(t) &\triangleq \sum_{s=1}^t \mathbbm{1}_{\{i \in a_s\}} \ , \quad \mbox{and} \quad
    N_{i}(t) \triangleq \sum_{s=1}^t \mathbbm{1}_{\{i \notin a_s \}} = t - N^{*}_{i}(t) \ .
\end{align}
Accordingly, for node $i$ under intervention $a$ we define
\begin{align} \label{eq:def_N_iat}
    N_{i,a}(t) &\triangleq \mathbbm{1}_{\{i \in a\}} N^{*}_{i}(t) + \mathbbm{1}_{\{i \notin a\}} N_{i}(t) \ , 
\end{align}
and denote the estimation error at time $t$ for matrix $\bB_{a}$ and its columns by
\begin{align}\label{eq:def_delta_ait}
    \Delta_{a}(t) & \triangleq \bB_{a}(t) - \bB_{a} \ , \quad \mbox{and}  \quad [\Delta_{a}(t)]_{i} \triangleq [\bB_{a}(t)]_{i} - [\bB_{a}]_{i} \ , \qquad \forall i \in [N]\ .
\end{align} 

Our analysis will show that (Lemma \ref{lm:expected_reward}), regret analysis involves the powers of $\bB_{a}(t)$ matrices. To get insight into the matrix powers, consider a node $j \in \mcV \setminus \Pa(i)$ that is not a parent of node $i$. By construction, vector $X_{\Pa(i)}$ has a zero at its $j$-th entry. Then, the $j$-th rows of $[\bV_{i}(t)]^{-1}$ and $[\bV^{*}_{i}(t)]^{-1}$ will consist of only zeros, except for their $j$-th entries, which are $1$, accounting for the addition of $I_{N+1}$. Since $j \notin \Pa(i)$, we have $[X_{\Pa(i)}]_{j}(s)=0$ and $[\bB_{a}(t)]_{j,i}=0$ based on~\eqref{eq:estimate_obs}. Similarly, $[\bB^{*}]_{i}$ has non-zero entries at only entries $k \in \Pa(i)$, and $[\Delta_{a}(t)]_{i}$ will be at most $(d+1)$-sparse. We denote the estimation error of the $\ell$-th power of $\bB_{a}$ for each $\ell \in [L]$ and $i \in [N]$ by
\begin{align}\label{eq:def_delta_alit}
    \Delta_{a}^{(\ell)}(t) & \triangleq \bB_{a}^{\ell}(t) - \bB_{a}^{\ell} \ , \quad \mbox{and} \quad \big[\Delta_{a}^{(\ell)}(t)\big]_{i} \triangleq \big[\bB_{a}^{\ell}(t)\big]_{i} - \big[\bB_{a}^{\ell}\big]_{i}\ .
\end{align} 
Based on the approach described and the estimated quantities specified, in the next section, we present the main algorithm. 

\section{\algonameUCB{} Algorithm}\label{sec:UCB}
Upper confidence bound ({\rm UCB})-based algorithms are effective in a wide range of bandit settings. Their general principle is to, sequentially and adaptively to the data, compute upper confidence bounds on the reward of each arm. A learner, subsequently, in each round selects the arm with the largest upper confidence bound. Confidence intervals for the estimated parameters are leveraged to compute these bounds. In our problem, we have $2N$ unknown weight vectors $\{[\bB]_{i}, [\bB^{*}]_{i} : i \in [N] \}$ that specify the weight matrices $\{\bB_{a} : a \in \mcA\}$, which in turn, characterize the reward variables uniquely. Hence, we design a ${\rm UCB}$-based algorithm by maintaining confidence intervals for these $2N$ vectors. We describe the details of the algorithm next.

\paragraph{Algorithm details.} Algorithm \ref{alg:ucb_algorithm} presents our main bandit algorithm referred to as {\bf Lin}ear {\bf SEM} {\bf U}pper {\bf C}onfidence {\bf B}ound (\algonameUCB{}). As its inputs, it takes the graph structure $\mcG$ (set of parents for each node), action set $\mcA$, horizon $T$, and parameter $\beta_{T}$. We build the confidence intervals centered on the empirical estimates for observational and interventional weight vectors as follows:
\begin{align}
    \mcC_{i}(t) & \triangleq  \left\{\theta \in \mathbb{R}^N : \norm{\theta}\leq m_B , \norm{\theta - [\bB(t-1)]_{i}}_{\bV_{i}(t-1)} \leq \beta_{T} \right\}  \ , \label{eq:conf_obs} \\
    \mbox{and} \quad \mcC^*_{i}(t) & \triangleq \left\{ \theta \in \mathbb{R}^N : \norm{\theta}\leq m_B, \norm{\theta - [\bB^{*}(t-1)]_{i}}_{\bV^{*}_{i}(t-1)} \leq \beta_{T} \right\} \label{eq:conf_int} \ ,
\end{align}
where $\beta_{T}\in\R_+$ controls the size of the confidence intervals. Accordingly, we define the relevant confidence interval for node $i$ under intervention $a$ as 
\begin{align}
    \mcC_{i,a}(t) &\triangleq  \mathbbm{1}_{\{i \in a\}} \mcC^{*}_{i}(t) + \mathbbm{1}_{\{i \notin a\}} \mcC_{i}(t) \ . \label{eq:conf_C_iat}
\end{align}
Based on these, we define the upper confidence bound for intervention $a$ in round $t$ as follows:
\begin{align}\label{eq:ucb_definition}
    {\rm UCB}_a(t) \triangleq \left\{
    \begin{array}{ll}
       \max_{\Theta \in \mathbb{R}^{N\times N}}   & f(\Theta) \\ 
       &\\
       \text{s.t. }  &  [\Theta]_i \in \mcC_{i,a}(t) \, \quad  \forall \ i \in [N] 
    \end{array}\right. \ .
\end{align}
The \algonameUCB{} algorithm computes the upper confidence bounds in each round, and plays the action that has the largest upper confidence bound. The estimates $\{[\bB_{a}(t)]_{i} : i \in [N]\}$ are updated according to \eqref{eq:estimate_obs} and \eqref{eq:estimate_int}.
\begin{algorithm}[t]
\caption{\algonameUCB{}}
\label{alg:ucb_algorithm}
\begin{algorithmic}[1]
\State \textbf{Input:} Horizon $T$, causal graph $\mcG$, action set $\mcA$, parameter $\beta_{T}$.  
\State \textbf{Initialization:} Initialize parameters for $2N$ linear problems:
\State  $[\bB(0)]_{i} = \mathbf{0}_{(N+1)\times 1},  [\bB^{*}(0)]_{i} = \mathbf{0}_{(N+1) \times 1}, \ \forall i \in [N] \ $ \Comment{initiliaze estimates for parameter vectors}
\State  $\bV_{i}(0) = I_{N+1}, \ \bV^{*}_{i}(0) = I_{N+1}, \ g_{i}(0) = \mathbf{0}_{(N+1) \times 1}, \ g^{*}_{i}(0) = \mathbf{0}_{(N+1) \times 1}$ \Comment{initiliaze auxiliary parameters}
\For {$t = 1,2,\ldots,T$}
    \For {$a \in \mcA$}
        \State Compute ${\rm UCB}_a(t)$ according to \eqref{eq:ucb_definition} .
    \EndFor
    \State $a_t = \argmax_{a \in \mcA} {\rm UCB}_a(t)$ \Comment{select the action that maximizes {\rm UCB}} \label{line:select_best_arm_ucb}
    \State Pull $a_t$, observe $X(t) = (1,X_1(t),\dots, X_N(t))^\top$. \label{line:play_best_arm}
    \For {$i \in \{1,\dots,N\}$} \label{line:ucb_update_start}
        \If {$i \in a_t$} \Comment{update interventional parameters}
        \State $\bV^{*}_{i}(t) = \bV^{*}_{i}(t-1) + X_{\Pa(i)}(t) X^{\top}_{\Pa(i)}(t)$ \ \  \text{and} \ \ $\bV_{i}(t) = \bV_{i}(t-1)$  \label{line:update_V_int}
        \State $g^{*}_{i}(t) = g^{*}_{i}(t-1) + X_{\Pa(i)}(t) X_{i}(t)$ \ \ \text{and} \ \ $g_{i}(t) = g_{i}(t-1)$  \label{line:update_g_int}
        \State $[\bB^{*}(t)]_{i} = [\bV^{*}_{i}(t)]^{-1} g^{*}_{i}(t)$ \ \ \text{and} \ \ $[\bB(t)]_{i}=[\bB(t-1)]_{i}$  \label{line:update_B_int} 
        \Else \Comment{update observational parameters}
        \State $\bV_{i}(t) = \bV_{i}(t-1) + X_{\Pa(i)}(t) X^{\top}_{\Pa(i)}(t)$ \quad \text{and} \quad $\bV^{*}_{i}(t) = \bV^{*}_{i}(t-1)$ \label{line:update_V_obs}
        \State $g_{i}(t) = g_{i}(t-1) + X_{\Pa(i)}(t) X_{i}(t)$ \ \ \text{and} \ \ $g^{*}_{i}(t) = g^{*}_{i}(t-1)$   \label{line:update_g_obs}
        \State $[\bB(t)]_{i} = [\bV_{i}(t)]^{-1}g_{i}(t)$ \ \  \text{and} \ \ $[\bB^{*}(t)]_{i}=[\bB^{*}(t-1)]_{i}$ \label{line:update_B_obs} 
        \EndIf
    \EndFor \label{line:ucb_update_end}
\EndFor
\end{algorithmic}
\end{algorithm}%

\paragraph{Regret analysis.} The regret in round $t$ depends on the closeness of the estimates $\{[\bB(t)]_{i}, [\bB^{*}(t)]_{i} : i \in [N]\}$ to the true parameters $\{[\bB]_{i}, [\bB^{*}]_{i} : i \in [N]\}$. The estimation errors will become smaller as more data samples are observed. Since $\beta_{T}$ controls the size of the confidence sets, a proper choice of $\beta_{T}$ would guarantee that the true parameters lie in the defined confidence intervals with high probability. Consider an intervention $a$, and suppose that constant $\beta_{T}$ is chosen such that it ensures the closeness of the estimates of each vector $\{[\bB_{a}(t)]_{i}: i \in [N]\}$ to the true parameters. Lemma~\ref{lm:expected_reward} indicates that the expected reward $\mu_{a}$ is composed of $L+1$ components, corresponding to the set of paths of length $\ell \in [L+1]$. Therefore, the mismatch in the computed rewards can be decomposed into $L+1$ components, where each term accounts for the estimation error of a certain path length. From~\eqref{eq:def_delta_alit} recall that the estimation error term corresponding to paths of length $\ell$ is denoted by $\Delta_{a}^{(\ell)}$. We do not form direct estimates for the error terms $\Delta_{a}^{(\ell)}$ and, subsequently, cannot directly find bounds on them. Hence, to find such bounds, we leverage the estimation error bounds of the individual estimation guarantees on weight vectors and then aggregate them according to ~\eqref{eq:def_delta_alit}. The following lemma delineates a data-dependent bound on the estimation error $\Delta_{a}^{(\ell)}(t)$, and will be pivotal in our regret analysis.  
\begin{lemma} \label{lm:bound_l_paths}
If $\norm{[\Delta_{a}(t)]_{i}}_{\bV_{i,a}(t)}\leq \beta_{T}$ for all $i \in [N]$ and $t \in [T]$, then for all $\ell \in [L+1]$ we have 
\begin{align}
       \norm{\big[\Delta_{a}^{(\ell)}(t)\big]_{N}}_{\bV_{N,a}(t)} < (d+1)^{\frac{\ell-1}{2}} (\beta_{T}+m_B)^{\ell} \left[\lmaxx{\bV_{N,a}(t)}{1/2} \max_{i \in [N]} \lminn{\bV_{i,a}(t)}{-1/2} \right] \ .
\end{align}
\end{lemma}
\begin{proof}
See Appendix~\ref{appendix:proofs}.
\end{proof}
Besides the true parameters $\bB_{a}$, and the estimated parameters $\bB_{a}(t)$, there is one more set of parameters relevant for analyzing the performance of \algonameUCB{}, namely the parameters that attain the upper confidence bounds. The following corollary will be used in parallel to Lemma~\ref{lm:bound_l_paths} for treating parameter estimations in our analysis.
\begin{corollary}\label{corollary_lemma2}
For all $\bA \in \mathbb{R}^{N\times N}$ define
\begin{align}
    \Delta_{\bA}(t) \triangleq \bA - \bB_{a}(t) \ , \quad \mbox{and} \quad \Delta_{\bA}^{(\ell)}(t) \triangleq \bA^{\ell} - \bB_{a}^{\ell}(t)\ .
\end{align}
If $\forall i \in [N]$ and $\forall t \in [T] $, $\bA$ satisfies the following conditions:
\begin{enumerate}
    \item $\norm{[\bA]_{i}}\leq m_B$; 
    \item $\norm{[\bA]_{i}}_0 \leq d+1$; 
    \item $\norm{\Delta_{\bA}(t)}_0 \leq d+1$;
    \item $\norm{[\Delta_{\bA}(t)]_{i}}_{\bV_{i,a}(t)}\leq~\beta_{T} \  $;
\end{enumerate}
then for all $\ell \in [L]$ we have
\begin{align}
   \norm{\big[\Delta_{\bA}^{(\ell)}(t)\big]_{N}}_{\bV_{N,a}(t)} <  (d+1)^{\frac{\ell-1}{2}} (\beta_{T}+m_B)^{\ell} \big[\lmaxx{\bV_{N,a}(t)}{1/2} \max_{i \in [N]} \lminn{\bV_{i,a}(t)}{-1/2} \big] \ . \label{eq:corollary_lemma2}
\end{align}
\end{corollary}
\begin{proof}
See Appendix~\ref{appendix:proofs}.
\end{proof}
For Lemma \ref{lm:bound_l_paths} to be useful in the regret analysis of \algonameUCB{}, we need to ensure that the conditions of its statement hold with high probability. We note that estimating $[\bB_{a}]_{i}$ by regressing from $X_{\Pa(i)}$ to $X_{i}$ is a linear problem. Then, the condition $\norm{[\Delta_{a}(t)]_{i}}_{\bV_{i,a}(t)}\leq~\beta_{T}$ follows the form of a uniform confidence interval guarantee for the linear bandits, which has been studied extensively~\citep{dani2008linear,abbasi-yadkori2011}. Specifically, the results of \citet[Theorem 20.5]{lattimore2020bandit} are relevant to our confidence intervals defined in \eqref{eq:conf_obs}~and~\eqref{eq:conf_int}. We adopt these results, and modify them to account for the effect of including  $I_{N+1}$ (in the definition of $\bV_{i,a}(t)$) as a regularizer  for the least-squares estimator of $[\bB_{a}]_{i}$. This is formalized in the next theorem.

\begin{theorem}{\citet[Theorem 20.5]{lattimore2020bandit}}
Let $\delta \in (0,1)$. If $\norm{[\bB_{a}]_{i}}\leq m_B$ for a node $i$, then $\P(\exists t \in \mathbb{N} : [\bB_{a}]_{i} \notin \mcC^{'}_{i,a}(t)) \leq \delta$, where we have defined
\begin{align}
    \mcC^{'}_{i,a}(t) & \triangleq \left\{\theta \in \mathbb{R}^N : \norm{\theta - [\bB_{a}(t-1)]_{i}}_{\bV_{i,a}(t-1)} \leq m_B + \sqrt{2\log(1/\delta)+\log \left( \det \bV_{i,a}(t-1)\right)} \right\}  \ .
\end{align}
\end{theorem}
Note that the confidence set $\mcC^{'}_{i,a}(t)$ resembles $\mcC_{i,a}(t)$ that we defined in \eqref{eq:conf_C_iat}. Since $\norm{[\bB_{a}]_{i}}\leq m_B$, the proper choice of $\beta_{T}$ would make the both confidence sets equal.

\paragraph{Sketch of the result:} We are ready to present the first part of the regret result. The events that do not satisfy the conditions of Lemma \ref{lm:bound_l_paths} will incur a constant term in the regret. Since $\epsilon$ is assumed to be $1$-sub-Gaussian, we can use \citet[Theorem 20.5]{lattimore2020bandit} to control the probability of such events. Subsequently, we will use Lemma \ref{lm:bound_l_paths} to analyze each of $L+1$ components of the regret. The structure of the graph $\mcG$ will affect the regret through degree $d$ and the longest causal path length $L$ as implied by Lemma \ref{lm:bound_l_paths}. Note that these results are for a time instance $t \in [T]$, but we aim to find bounds on the cumulative regret. Furthermore, the properties of the intervention space have not been accounted for yet. The term $\lambda_{T}$, which will be defined shortly, captures these effects and is analyzed in a separate result (Theorem \ref{th:regret_ucb_part2}).

\begin{theorem}\label{th:regret_ucb_part1}
Under Assumption~\ref{assumption:boundedness}, the regret of \algonameUCB{} is bounded by  
\begin{align}
    \E[R(T)]\leq 2m + 2 (\beta_{T}+m_B)^{L+1} (d+1)^{\frac{L}{2}} \lambda_{T} \ , 
\end{align}
where we have set
\begin{align}
 \beta_{T} & = m_B + \sqrt{2\log(2NT)+(d+1)\log(1+m^2T/(d+1))}\ , \label{eq:def-beta-in-theorem}\\
 \mbox{and} \quad 
     \lambda_{T} & =  \E\left[ \sum_{t=1}^T  \sqrt{\frac{\lmax{\bV_{N,a_t}(t)}}{\lmin{\bV_{N,a_t}(t)} \min_{i \in [N]} \lmin{\bV_{i,a_t}(t)}}}  \right] \ . \label{eq:def_lambda_T}
\end{align}
\end{theorem}

\begin{proof}
For the cumulative regret specified in~\eqref{eq:frequentist_regret} we have
\begin{align}
    \E[R(T)] = T\mu_{a^*} - \E \left[ \sum_{t=1}^T X_N(t) \right] = \E\left[ \sum_{t=1}^{T} (\mu_{a^*} - \mu_{a_t}) \right] \ .
\end{align}
From Lemma \ref{lm:expected_reward} we have $\mu_{a_t}=f(\bB_{a_t})$. Therefore, $\E[R(T)]$ can be equivalently stated as
\begin{align} \label{eq:ucb_proof_regret}
    \E[R(T)] = \E\left[\sum_{t=1}^T  (f(\bB_{a^*}) - f(\bB_{a_t})) \right]\ , 
\end{align}
Note that we were able use Lemma \ref{lm:expected_reward} since $\epsilon(t)$ is independent of the data, which governs the choice of $a_t$. Since $\beta_{T}=m_B+\sqrt{2\log(2NT)+(d+1)\log(1+m^2T/(d+1))}$ is independent of $t \in [T]$, we use the shorthand $\beta$ to replace it. Next, we define the events $\{\mcE_{i}, \mcE^{*}_{i} : \forall i \in [N]\}$ as
\begin{align}
    \mcE_{i} &\triangleq \biggl\{ \forall t \in [T] : \norm{[\bB(t-1)]_{i}-[\bB]_{i}}_{\bV_{i}(t-1)} \leq \beta \biggr\}  \ ,  \\
   \mbox{and} \quad  \mcE^{*}_{i} &\triangleq \biggl\{ \forall t \in [T] : \norm{[\bB^{*}(t-1)]_{i}-[\bB^{*}]_{i}}_{\bV^{*}_{i}(t-1)} \leq \beta \biggr\}  \ .
\end{align}
We will show that the specified choice of $\beta$ ensures that the events $\{\mcE_{i}, \mcE^{*}_{i} : \forall i \in [N]\}$ hold with a high probability. In other words, the confidence intervals of ${\rm UCB}$ contain the true parameters with high probability. To this end, we first bound $\bV_{i}(t)$. Since $\bV_{i}(t)$ is a positive definite matrix, we can use the arithmetic-geometric mean inequality (AM-GM) to upper bound its determinant through its trace. Furthermore, since $\norm{X}\leq m$, the trace of $V_i(t)$ will be upper bounded by $d_i+T m^2$. Therefore, we have
\begin{align}
    \det \bV_{i}(t) \overset{{\rm (AM-GM)}}&{\leq} \left(\frac{1}{d_i}{\rm tr}\left(\bV_{i}(t)\right)\right)^{d_i} \leq \left( 1+ \frac{Tm^2}{d_i} \right)^{d_i} \leq \left( 1+ \frac{Tm^2}{d} \right)^{d} \  .
\end{align}
By noting that the norms of vectors $\{[\bB]_{i}, [\bB^{*}]_{i} : \forall i \in [N]\}$ are bounded by $m_B$ (Assumption~\ref{assumption:boundedness}), and setting $\delta = \frac{1}{2NT}$, \citet[Theorem 20.5]{lattimore2020bandit} yields
\begin{align}
    \P(\mcE_{i}^{\C}) \leq \frac{1}{2NT} \ , \quad \P({\mcE^{*}_{i}}^{\C}) \leq \frac{1}{2NT}\ , \quad \forall i \in [N] \ .\label{eq:CI_error_prob_1}
\end{align}
Let $\mcE_{\cap}$ denote the event that all of the events $\{\mcE_{i}, \mcE^{*}_{i}: i \in [N]\}$ occur simultaneously, i.e., 
\begin{align}
    \mcE_{\cap} \triangleq \left(\bigcap_{i=1}^N  \mcE_{i} \right) \bigcap \left(\bigcap_{i=1}^N  \mcE^{*}_{i} \right)  \ .  \label{eq:ucb_conf_interval_event_union}
\end{align}
By invoking the union bound we have
\begin{align}
    \P \left( \mcE_{\cap}^{\C} \right) &\leq  \sum_{i=1}^{N} \P(\mcE_{i}^{\C}) + \sum_{i=1}^{N} \P({\mcE^{*}_{i}}^{\C}) 
     \overset{\eqref{eq:CI_error_prob_1}} {\leq} \sum_{i=1}^N \left(\frac{1}{2NT} + \frac{1}{2NT} \right) 
      = \frac{1}{T} \ . \label{eq:bound_Ec}
\end{align}
Next, we analyze the regret under the complementary events $\mcE_{\cap}$ and $\mcE_{\cap}^{\C}$. Note that the regret at any time $t$ can be at most $2m$ since $|X_{N}| \leq \norm{X} \leq m$. Therefore, for the expected regret we have
\begin{align} 
 \E[R(T)] &= \E\left[ \sum_{t=1}^{T} (f(\bB_{a^*})-f(\bB_{a_t})) \right] \\
 & = \E\left[\mathbbm{1}_{\mcE_{\cap}^{\C}} \sum_{t=1}^{T} \underset{\leq 2m}{\underbrace{(f(\bB_{a^*})-f(\bB_{a_t})}} \right] + \E\left[\mathbbm{1}_{\mcE_{\cap}} \sum_{t=1}^{T} (f(\bB_{a^*})-f(\bB_{a_t}) \right] \\
    &\leq 2m T \P(\mcE_{\cap}^{\C}) +\E\left[\mathbbm{1}_{\mcE_{\cap}} \sum_{t=1}^{T} (f(\bB_{a^*})-f(\bB_{a_t}) \right] \\
 \overset{\eqref{eq:bound_Ec}}&{\leq} 2m +\E\left[\mathbbm{1}_{\mcE_{\cap}} \sum_{t=1}^{T} (f(\bB_{a^*})-f(\bB_{a_t})) \right] \ . \label{eq:ucb_proof_mid1}
\end{align}
The algorithm selects $a_t = \argmax_{a \in \mcA} {\rm UCB}_{a}(t)$ in round $t$. Let $\tilde \bB_{a}$ denote the parameters that attain ${\rm UCB}_{a}(t)$, i.e., $ f(\tilde \bB_{a}) ={\rm UCB}_{a}(t)$. By definition, ${\rm UCB}_{a^*}(t) \leq {\rm UCB}_{a_t}(t)$. Under the event $\mcE_{\cap}$, we have
\begin{align}
    f(\bB_{a^*}) \leq{\rm UCB}_{a^*}(t) &\leq{\rm UCB}_{a_t}(t) =  f(\tilde \bB_{a_t})  \ , \\
    \mbox{and} \qquad f(\bB_{a^*})-f(\bB_{a_t}) &\leq f(\tilde \bB_{a_t})-f(\bB_{a_t}) \ . \label{eq:ucb_proof_mid2}
\end{align}
For the term $f(\tilde \bB_{a_t})-f(\bB_{a_t})$, based on the definition of $f$ in Lemma \ref{lm:expected_reward}, by applying the Cauchy-Schwarz (CS) inequality we have
\begin{align}
    f(\tilde \bB_{a_t})-f(\bB_{a_t}) &= \sum_{\ell=1}^{L+1} \left(\big[\tilde \bB_{a_t}^{\ell}\big]_{0,N} - \big[\bB_{a_t}^{\ell}\big]_{0,N}\right) \\
    &\leq  \sum_{\ell=1}^{L+1} \norm{\big[\tilde \bB_{a_t}^{\ell}\big]_{N} - \big[\bB_{a_t}^{\ell}\big]_{N}} \ , \label{eq:ucb_proof_mid5_1} \\ 
    \overset{\rm (CS)}&{\leq} \sum_{\ell=1}^{L+1} \norm{\big[\tilde \bB_{a_t}^{\ell}\big]_{N} - \big[\bB_{a_t}^{\ell}\big]_{N}}_{\bV_{N,a_t}(t)} \lminn{\bV_{N,a_t}(t)}{-1/2} \ .
\end{align}
Note that $[\bB_{a_t}(t)]_{N}$ is an estimate of $[\bB_{a_t}]_{N}$, and $[\tilde \bB_{a_t}]_{N}$ lies in the confidence interval that is centered on $[\bB_{a_t}(t)]_{N}$. We decompose $\big[\tilde \bB_{a_t}^{\ell}\big]_{N} - \big[\bB_{a_t}^{\ell}\big]_{N}$ into two parts by adding and subtracting the term $\big[\bB_{a_t}^{\ell}(t)\big]_{N}$ as follows
\begin{align}
& \big[\tilde \bB_{a_t}^{\ell}\big]_{N} - \big[\bB_{a_t}^{\ell}\big]_{N} = \left(\big[\tilde \bB_{a_t}^{\ell}\big]_{N} - \big[\bB_{a_t}^{\ell}(t)\big]_{N}\right) + \left(\big[\bB_{a_t}^{\ell}(t)\big]_{N} - \big[\bB_{a_t}^{\ell}\big]_{N}\right)  \ , 
\end{align}
and due to the triangle inequality we have
\begin{align}
& \norm{\big[\tilde \bB_{a_t}^{\ell}\big]_{N} - \big[\bB_{a_t}^{\ell}\big]_{N}}_{\bV_{N,a_t}(t)} \\
& \hspace{.5 in}\leq \norm{\big[\tilde \bB_{a_t}^{\ell}\big]_{N} - \big[\bB_{a_t}^{\ell}(t)\big]_{N}}_{\bV_{N,a_t}(t)} + \norm{\big[\bB_{a_t}^{\ell}(t)\big]_{N} - \big[\bB_{a_t}^{\ell}\big]_{N}}_{\bV_{N,a_t}(t)} \ . \label{eq:ucb_proof_mid6}
\end{align}
Using \eqref{eq:ucb_proof_mid6} in the right-hand side of \eqref{eq:ucb_proof_mid1} we obtain
\begin{align}
    \E[R(T)] \overset{\eqref{eq:ucb_proof_mid2}}&{\leq} 2m+ \E\left[\mathbbm{1}_{\mcE_{\cap}} \sum_{t=1}^{T} (f(\tilde \bB_{a_t})-f(\bB_{a_t})) \right] \\
        \overset{\eqref{eq:ucb_proof_mid6}}&{\leq} 2m + \E\left[\mathbbm{1}_{\mcE_{\cap}} \sum_{t=1}^T \sum_{\ell=1}^{L+1} \norm{\big[\tilde \bB_{a_t}^{\ell}\big]_{N} - \big[\bB_{a_t}^{\ell}(t)\big]_{N}}_{\bV_{N,a_t}(t)} \lminn{\bV_{N,a_t}(t)}{-1/2} \right] \label{eq:ucb_proof_mid7} \\
    & \qquad + \E\left[\mathbbm{1}_{\mcE_{\cap}} \sum_{t=1}^T \sum_{\ell=1}^{L+1} \norm{\big[\bB_{a_t}^{\ell}(t)\big]_{N} - \big[\bB_{a_t}^{\ell}\big]_{N}}_{\bV_{N,a_t}(t)} \lminn{\bV_{N,a_t}(t)}{-1/2} \right] \ . \label{eq:ucb_proof_mid8}
\end{align}
Under the event $\mcE_{\cap}$, the conditions of Lemma \ref{lm:bound_l_paths} are satisfied for matrices $\Delta_{a}^{(\ell)}(t)$ and $\bB_{a_t}$. Similarly, the conditions of Corollary \ref{corollary_lemma2} are satisfied for matrices $\tilde \bB_{a_t}^{\ell} - \bB_{a_t}^{\ell}(t)$, and~$\tilde \bB_{a_t}$. Therefore, by applying Lemma \ref{lm:bound_l_paths} to each term in \eqref{eq:ucb_proof_mid8}, and Corollary \ref{corollary_lemma2} to each term in \eqref{eq:ucb_proof_mid7}, we obtain
\begin{align}
    \E[R(T)] &\leq 2m + \E\left[\mathbbm{1}_{\mcE_{\cap}} \sum_{t=1}^T  \sqrt{\frac{\lmax{\bV_{N,a_t}(t)}}{\lmin{\bV_{N,a_t}(t)} \min_{i \in [N]} \lmin{\bV_{i,a_t}(t)}}}  \right] \notag \\
    & \qquad \qquad \times 2\sum_{\ell=1}^{L+1} (d+1)^{\frac{\ell-1}{2}}(\beta + m_B)^{\ell}   \label{eq:ucb_proof_mid9} \\
     &\leq 2m + \E\left[\sum_{t=1}^T  \sqrt{\frac{\lmax{\bV_{N,a_t}(t)}}{\lmin{\bV_{N,a_t}(t)} \min_{i \in [N]} \lmin{\bV_{i,a_t}(t)}}}  \right] \notag \\
    & \qquad \qquad \times 2\sum_{\ell=1}^{L+1} (d+1)^{\frac{\ell-1}{2}}(\beta + m_B)^{\ell}   \label{eq:ucb_proof_mid9_2} \\
    & =  2m + 2\lambda_T \frac{1}{\sqrt{d+1}} \sum_{\ell=1}^{L+1} ((\beta + m_B)\sqrt{d+1})^{\ell} \ , \label{eq:ucb_proof_mid11}
\end{align}
in which, $\lambda_T \triangleq \E\left[ \sum_{t=1}^T  \sqrt{\frac{\lmax{\bV_{N,a_t}(t)}}{\lmin{\bV_{N,a_t}(t)} \min_{i \in [N]} \lmin{\bV_{i,a_t}(t)}}} \right]$.
Note that, for $c \geq 1$, 
\begin{align}
    \sum_{\ell=1}^{L+1} c^\ell = \frac{c^{L+2}-1}{c-1} - 1 \leq 2c^{L+1} \ . \label{eq:snip1}
\end{align}
Since $(\beta + m_B) \sqrt{d+1} \overset{\eqref{eq:def-beta-in-theorem}}{>} \sqrt{2}\sqrt{2\log 2} > 1$, by using \eqref{eq:snip1} in \eqref{eq:ucb_proof_mid11} we obtain
\begin{align}
    \E[R(T)] &\leq 2m + \frac{2 \lambda_{T}}{\sqrt{d+1}}((\beta + m_B)\sqrt{d+1})^{L+1} \\
    &= 2m + 2 \lambda_{T} (\beta+m_B)^{L+1} (d+1)^{\frac{L}{2}} \ .
\end{align}
\end{proof}
Next, we analyze the $\lambda_{T}$ term in \eqref{eq:def_lambda_T}, which in conjunction with Theorem \ref{th:regret_ucb_part1} characterizes our desired regret bound.

\begin{theorem}\label{th:regret_ucb_part2}
Under Assumption~\ref{assumption:boundedness}, $\lambda_T$ specified as
\begin{align}
    \lambda_{T} \triangleq  \E\left[ \sum_{t=1}^T  \sqrt{\frac{\lmax{\bV_{N,a_t}(t)}}{\lmin{\bV_{N,a_t}(t)} \min_{i \in [N]} \lmin{\bV_{i,a_t}(t)}}}  \right]\ ,
\end{align}
is bounded according to
\begin{align}
    \lambda_{T} < \frac{4g(\tau)}{\sqrt{\kappa_{\min}}}\sqrt{NT} + 2\sqrt{2}(N+1)\tau g(\tau) + \frac{2\sqrt{2}N\sqrt{\tau} g(\tau)}{\sqrt{\kappa_{\min}}} \log\left(\frac{T}{2N}\right) + \frac{m}{T} + \frac{2m}{3}  + 1 \ ,
\end{align}
where $\tau=\frac{\alpha^2 m^4}{\kappa_{\min}^2}$, $\alpha=\sqrt{\frac{16}{3}\log((d+1)NT^{T/2}(T+1))}$, and $g(\tau)=\sqrt{2}(\sqrt{\tau \kappa_{\max}}+\sqrt{\tau \kappa_{\min}}+1)$. Furthermore, since $\alpha = \mcO(\sqrt{\log(T)}), \tau = \mcO(\log(T))$, and $g(\tau) = \mcO(\sqrt{\log(T)})$, we can write
\begin{align}
    \lambda_{T} = K_1 \sqrt{NT} + K_2 (\log(T))^2 + K_3 \ ,
\end{align}
where $K_1 = \frac{4 g(\tau)}{\sqrt{\kappa_{\min}}}$, and $K_2$ and $K_3$ are constants independent of $T$.
\end{theorem}

\begin{proof}
We start by simplifying the notation for the quantity to bound. Note that $\alpha_T$ is a function of $T$ and independent of a given $t \in [T]$. For simplicity, we use $\alpha$ as a shorthand for $\alpha_T$. We also define $C_{N}(t)$ and $S(t)$ to compactly express $\lambda_{T}$ as follows
\begin{align}
    C_{N}(t) &\triangleq \sqrt{\frac{\lmax{\bV_{N,a_t}(t)}}{\lmin{\bV_{N,a_t}(t)}}} \overset{\eqref{eq:V_ita_D_ita}}{=} \sqrt{\frac{\smaxx{\bD_{N,a_t}(t)}{2}+1}{\sminn{\bD_{N,a_t}(t)}{2}+1}} \ , \label{eq:C_Nt_general} \\
    S(t) &\triangleq \frac{1}{\sqrt{\min_{i \in [N]} \lmin{\bV_{i,a_t}(t)}}} \overset{\eqref{eq:V_ita_D_ita}}{=} \frac{1}{\sqrt{\min_{i \in [N]} \sminn{\bD_{i,a_t}(t)}{2} + 1}}  \ , \label{eq:S_t_general}  
\end{align}
based on which we have,
\begin{align}
    \lambda_{T} &= \E \left[ \sum_{t=1}^{T} C_{N}(t) S(t) \right] \ . \label{eq:def_express_lambda_T} 
\end{align}
For bounding $\E\left[\sum_{t=1}^T C_{N}(t) S(t) \right]$, we will use upper and lower bounds for the maximum and minimum singular values of $\bD_{i,a_t}(t)$. However, such bounds depend on the number of non-zero rows of $\bD_{i,a_t}(t)$ matrices, which equals to values of the random variable $N_{i,a_t}(t)$. To start, define the constants
\begin{align}
    \varepsilon_n &\triangleq  \max\left\{\alpha m^2 \sqrt{n},\alpha^2 m^2 \right\} \ , \quad \forall n \in [T] \ .  \label{eq:def_varepsilon_n}
\end{align}
Then, for each triplet of $i \in [N], t \in [T]$, and $n \in [t]$, we define the error events $\mcE_{i,n}(t), \mcE^{*}_{i,n}(t)$ as:
\begin{align}
    \mcE_{i,n}(t) &\triangleq \Bigg\{ N_{i}(t) = n \quad \text{and} \notag \\ 
    & \hspace{-0.25in} \left\{\smin{\bD_{i}(t)}\leq \max\left\{0,\sqrt{n \kappa_{\min}} - \frac{\varepsilon_n}{\sqrt{n \kappa_{\min}}}\right\} \  \text{or} \   \smax{\bD_{i}(t)}\geq \sqrt{n \kappa_{\max}} + \frac{\varepsilon_n}{\sqrt{n\kappa_{\min}}} \right\} \Bigg\} \ , \label{eq:def_error_int_obs} \\
    \mcE^{*}_{i,n}(t) &\triangleq \Bigg\{ N^{*}_{i}(t) = n \quad \text{and} \notag \\ 
    & \hspace{-0.25in} \left\{\smin{\bD^{*}_{i}(t)}\leq \max\left\{0,\sqrt{n \kappa_{\min}} - \frac{\varepsilon_n}{\sqrt{n \kappa_{\min}}}\right\} \  \text{or} \   \smax{\bD^{*}_{i}(t)}\geq \sqrt{n \kappa_{\max}} + \frac{\varepsilon_n}{\sqrt{n\kappa_{\min}}} \right\} \Bigg\}  \ . \label{eq:def_error_int_int}
\end{align}
In other words, the event $\mcE_{i,n}(t)$ specifies the condition under which at least one of the terms $\smin{\bD_{i}(t)}$ and $\smax{\bD_{i}(t)}$ does not conform the lower and upper bounds that we construct. $\mcE^{*}_{i,n}(t)$ has the counterpart implications for singular values of $\bD^{*}_{i}(t)$. The next result shows that events $\mcE_{i,n}(t)$ and $\mcE^{*}_{i,n}(t)$ occur with low probability.
\begin{lemma} \label{lm:error_event_int}
The probability of the events $\mcE_{i,n}(t)$ and $\mcE^{*}_{i,n}(t)$ defined in \eqref{eq:def_error_int_obs} and \eqref{eq:def_error_int_int} are upper bounded as
\begin{align}
    \P(\mcE_{i,n}(t)) &\leq (d+1) \exp \left( -\frac{3\alpha^2}{16} \right) \ , \\
    \mbox{and} \qquad \P(\mcE^{*}_{i,n}(t)) &\leq (d+1) \exp \left( -\frac{3\alpha^2}{16} \right) \ .
\end{align}
\end{lemma}
\begin{proof}
We will prove the analysis for bounding $\P(\mcE_{i,n}(t))$ and analysis for $\P(\mcE^{*}_{i,n}(t))$ follows similarly. The core of the proof is using Freedman's concentration inequality for matrix martingales. We define the martingale sequence $\bY_{i}(k)$, with difference sequence $\bZ_i(k)$, and the predictable quadratic variation of the process $\bW_i(k)$ as follows
\begin{align}
    \bZ_{i}(s) &\triangleq \mathbbm{1}_{\{i \notin a_s\}} \left( X_{\Pa(i)}(s) X_{\Pa(i)}^\top(s) - \Sigma_{i,a_s} \right) \ , \ \forall s \in [T] \ , \\
    \bY_{i}(k) &\triangleq \sum_{s=1}^k \bZ_{i}(s) \ , \ \forall k \in [T]   \ , \\
    \mbox{and} \qquad \bW_{i}(k) &\triangleq \sum_{s=1}^k \E[\bZ_{i}^2(s) \med \mcF_{s-1}] \ , \ \forall k \in [T]   \ ,
\end{align}
where $\mcF_{s-1} \triangleq \sigma(a_1,X(1),\dots,a_{s-1},X(s-1),a_s)$. Under the event $\mcE_{i,n}(t)$ we have $N_{i}(t)=n$. We will show that, given $N_{i}(t)=n$, we have $\norm{\bW_{i}(t)}\leq 2m^4 n$. Subsequently, given the event $\mcE_{i,n}(t)$, we will show that $\smax{\bY_{i}(t)}\geq \varepsilon_n$. The probability of these two events occurring together will be bounded by the matrix Freedman inequality. Finally, $\P(\mcE_{i,n}(t))$ will be upper bounded by the same probability. Detailed analysis is provided in Appendix~\ref{appendix:proofs}. 
\end{proof}
Now that we have bounds on the probability of error events $\mcE_{i,n}(t)$ and $\mcE^{*}_{i,n}(t)$, we define the union error event $\mcE_{\cup}$ as 
\begin{align} \label{eq:def_union_error_singular}
    \mcE_{\cup} \triangleq \{ \exists\ (i,t,n) : i \in [N], t \in [T], n \in [t], \ \mcE_{i,n}(t) \ \text{or} \  \mcE^{*}_{i,n}(t)  \} \ .
\end{align}
By taking a union bound and using Lemma \ref{lm:error_event_int} we have
\begin{align}
    \P(\mcE_{\cup}) &\leq \sum_{i=1}^{N} \sum_{t=1}^{T} \sum_{n=1}^{t} (\P(\mcE_{i,n}(t)) + \P(\mcE^{*}_{i,n}(t))) \\
    &\leq N T (T+1) (d+1) \exp \left( -\frac{3\alpha^2}{16} \right) \label{eq:def_union_error_singular_prob}
\end{align}
Now we turn back to $\E\left[ \sum_{t=1}^T C_{N}(t)S(t) \right]$ to analyze it under the complementary events $\mcE_{\cup}$ and $\mcE_{\cup}^{\C}$.
\begin{align}
    \E \left[ \sum_{t=1}^T C_{N}(t)S(t) \right] &= \E \left[\mathbbm{1}_{\mcE_{\cup}} \sum_{t=1}^T C_{N}(t)S(t) \right] + \E \left[\mathbbm{1}_{\mcE_{\cup}^\C} \sum_{t=1}^T C_{N}(t)S(t) \right] \ . \label{eq:decompose_good_bad_events}
\end{align}
Analyzing the second term will be more involved. Let us start with the first one.
\paragraph{Bounding $\E \left[\mathbbm{1}_{\mcE_{\cup}} \sum_{t=1}^T C_{N}(t)S(t) \right]$.} 
Since $\lmin{\bV_{i,a_t}(t)}\geq 1$, we have the following unconditional upper bound
\begin{align}
    C_{N}(t)S(t) &= \sqrt{\frac{\lmax{\bV_{N,a_t}(t)}}{\lmin{\bV_{N,a_t}(t)}}} \cdot \frac{1}{\sqrt{\min_{i \in [N]} \lmin{\bV_{i,a_t}(t)}}} \leq \sqrt{\lmax{\bV_{N,a_t}(t)}} \ .  \label{eq:big_bounding_error_1}
\end{align}
For finding an unconditional upper bound on $\lmax{\bV_{N,a_t}(t)}$, we leverage $\norm{X}\leq m$ as follows
\begin{align}
    \lmax{\bV_{N,a_t}(t)} &= \lmax{I_{N+1} + \sum_{s=1}^{t} \mathbbm{1}_{\{N \in a_s\}} X_{\Pa(i)}(s) X_{\Pa(i)}^{\top}(s)} \\
    &\leq 1 + \sum_{s=1}^{t} \mathbbm{1}_{\{N \in a_s\}} \lmax{X_{\Pa(i)}(s) X_{\Pa(i)}^{\top}(s)} \\
    &\leq 1 + \sum_{s=1}^{t} \lmax{X_{\Pa(i)}(s) X_{\Pa(i)}^{\top}(s)} \\
    &=  1 + \sum_{s=1}^{t} \norm{X_{\Pa(i)}(s)}^2 \\
    &\leq m^2t + 1 \ .  \label{eq:big_bounding_error_2}
\end{align}
Therefore, the desired quantity is bounded by
\begin{align}
    \E \left[\mathbbm{1}_{\mcE_{\cup}} \sum_{t=1}^T C_{N}(t)S(t) \right] \overset{\eqref{eq:big_bounding_error_1}}&{\leq} \E \left[\mathbbm{1}_{\mcE_{\cup}} \sum_{t=1}^T \sqrt{\lmax{\bV_{N,a_t}(t)}}  \right] \\
    \overset{\eqref{eq:big_bounding_error_2}}&{\leq} \E \left[\mathbbm{1}_{\mcE_{\cup}} \sum_{t=1}^T \sqrt{ m^2t + 1}  \right] \\
    &= \E[\mathbbm{1}_{\mcE_{\cup}}] \sum_{t=1}^T \sqrt{ m^2t + 1}   \\
    &= \P(\mcE_{\cup}) \sum_{t=1}^T \sqrt{ m^2t + 1} \ . \label{eq:big_bounding_error_3}
\end{align}
We have derived a bound for $\P(\mcE_{\cup})$ at \eqref{eq:def_union_error_singular_prob}. The sum term is bounded as
\begin{align}
    \sum_{t=1}^T \sqrt{ m^2t + 1} &\leq (m\sqrt{T}+1) +  \sum_{t=1}^{T-1} (m\sqrt{t}+1) \\
    &\leq m\sqrt{T} + T + \int_{t=1}^{T} m\sqrt{t} dt \\
    &= m\sqrt{T} + T + \frac{2m}{3}(T^{3/2}-1) \label{eq:big_bounding_error_4}
\end{align}
By setting $\alpha = \sqrt{\frac{16}{3}\log((d+1) N T^{5/2}(T+1))}$, we obtain
\begin{align}
     \E \left[\mathbbm{1}_{\mcE_{\cup}} \sum_{t=1}^T \sqrt{m^2 t + 1}  \right] \overset{\eqref{eq:big_bounding_error_3}}&{\leq} \P(\mcE_{\cup}) \sum_{t=1}^{T}\sqrt{m^2t+1} \\
     \overset{\eqref{eq:def_union_error_singular_prob}}&{\leq} \underset{= T^{-3/2}}{\underbrace{\frac{N T (T+1) (d+1)}{\exp(\log((d+1) N T^{5/2} (T+1)))}}}  \sum_{t=1}^{T}\sqrt{m^2t+1}  \\
     \overset{\eqref{eq:big_bounding_error_4}}&{\leq} T^{-3/2} \left(m\sqrt{T} + T+ \frac{2m}{3}(T^{3/2}-1)\right) \\
     & \leq \frac{m}{T} + \frac{2m}{3} +   1 \ . \label{eq:error_event_bound}
\end{align}
\paragraph{Bounding $\E \left[\mathbbm{1}_{\mcE_{\cup}^\C} \sum_{t=1}^T C_{N}(t)S(t) \right]$.} Given the event $\mcE_{\cup}^\C$, all the events $\{\mcE_{i,n}^{\C}(t), {\mcE^{*}_{i,n}}^{\C}(t) : i \in [N], t \in [T], n \in [t]\}$ hold. Therefore, we can use the following bounds on the singular values
\begin{align}
    \smax{\bD_{i,a_t}(t)} &\leq \sqrt{N_{i,a_t}(t) \kappa_{\max}} + \frac{\alpha m^2}{\sqrt{\kappa_{\min}}} \max\left\{1,\frac{\alpha}{\sqrt{N_{i,a_t}(t)}}\right\} \ , \label{eq:smax_D_ita} \\
    \smin{\bD_{i,a_t}(t)} &\geq \max\left\{0, \sqrt{N_{i,a_t}(t)\kappa_{\min}} - \frac{\alpha m^2}{\sqrt{\kappa_{\min}}}\max\left\{1,\frac{\alpha}{\sqrt{N_{i,a_t}(t)}}\right\} \right\} \ . \label{eq:smin_D_ita}
\end{align}
Note that for values of $N_{i,a_t}(t)$ that are smaller than a certain threshold, the right-hand side of \eqref{eq:smin_D_ita} becomes zero. The threshold above which this lower bound becomes non-zero will be critical in the following steps. Hence, we define the constant
\begin{align}
    \tau &\triangleq  \frac{\alpha^2 m^4}{\kappa_{\min}^2} \ . \label{eq:tau_ucb} 
\end{align}
When $N_{i,a_t}(t) \geq \tau$, we have $\sqrt{N_{i,a_t}(t)} \geq \frac{\alpha m^2}{\kappa_{\min}} \geq \alpha$ since $\kappa_{\min}\leq m^2$, in which case \eqref{eq:smin_D_ita} reduces to
\begin{align}
        \smin{\bD_{i,a_t}(t)} &\geq \max\left\{0, \sqrt{N_{i,a_t}(t)\kappa_{\min}} - \sqrt{\tau \kappa_{\min}} \right\} \ . \label{eq:smin_D_ita2}
\end{align}
To facilitate the analysis, we dispense with the square-root terms via using the following bounds
\begin{align}
    C_{N}(t) \overset{\eqref{eq:C_Nt_general}}&{=} \sqrt{\frac{\smaxx{\bD_{N,a_t}(t)}{2}+1}{\sminn{\bD_{N,a_t}(t)}{2}+1}} \leq \sqrt{2}\frac{\smax{\bD_{N,a_t}(t)}+1}{\smin{\bD_{N,a_t}(t)}+1} \ , \label{eq:C_Nt_root2} \\
    S(t) \overset{\eqref{eq:S_t_general}}&{=} \frac{1}{\sqrt{\min_{i \in [N]} \sminn{\bD_{i,a_t}(t)}{2} + 1}} \leq \frac{\sqrt{2}}{\min_{i \in [N]} \smin{\bD_{i,a_t}(t)}+1} \ . \label{eq:S_t_root2} 
\end{align}
Note that \eqref{eq:C_Nt_root2} follows from $\frac{x^2+1}{y^2+1}\leq 2(\frac{x+1}{y+1})^2$ when $x\geq y \geq 0$, and \eqref{eq:S_t_root2} follows from $2(x^2+1)>(x+1)^2$. Next, we define the following two functions of $x \in \mathbb{R}^{+}$:
\begin{align}
    g(x) &\triangleq \sqrt{2}\frac{\sqrt{x \kappa_{\max}}+\sqrt{\tau \kappa_{\min}}\max\left\{1,\frac{\alpha}{\sqrt{x}}\right\}+1}{\max\left\{0,\sqrt{x \kappa_{\min}}-\sqrt{\tau \kappa_{\min}} \right\}+1} \ , \label{eq:g_function} \\
    h(x) &\triangleq \frac{\sqrt{2}}{\max\left\{0,\sqrt{x \kappa_{\min}}-\sqrt{\tau \kappa_{\min}}\right\}+1}  \label{eq:h_function} \ .
\end{align}
Given the event $\mcE_{\cup}^{\C}$, we bound $C_{N}(t)$ and $S(t)$ in terms of the newly defined $g$ and $h$ functions as
\begin{align}
     \mathbbm{1}_{\mcE_{\cup}^\C} C_N(t) \overset{\eqref{eq:C_Nt_root2}}&{\leq}  \sqrt{2}\frac{\smax{\bD_{N,a_t}(t)}+1}{\smin{\bD_{N,a_t}(t)}+1} \\ \overset{\eqref{eq:smax_D_ita},\eqref{eq:smin_D_ita2}}&{\leq}  \sqrt{2}\frac{\sqrt{N_{N,a_t}(t) \kappa_{\max}}+\sqrt{\tau \kappa_{\min}}\max\left\{1,\frac{\alpha}{\sqrt{N_{N,a_t}(t)}}\right\}+1}{\max\left\{0,\sqrt{N_{N,a_t}(t) \kappa_{\min}}-\sqrt{\tau \kappa_{\min}}\right\}+1}  \\
    &= g(N_{N,a_t}(t)) \ , \label{eq:reduce_to_g} \\
    \mathbbm{1}_{\mcE_{\cup}^\C} S(t) \overset{\eqref{eq:S_t_root2}}&{\leq}  \frac{\sqrt{2}}{\min_{i \in [N]} \smin{\bD_{i,a_t}(t)}+1} \\
    \overset{\eqref{eq:smin_D_ita2}}&{\leq} \frac{\sqrt{2}}{\min_{i \in [N]} \max\left\{0,\sqrt{N_{i,a_t}(t) \kappa_{\min}}-\sqrt{\tau \kappa_{\min}}\right\}+1} \\
    &= \max_{i \in [N]} h(N_{i,a_t}(t)) \ . \label{eq:reduce_to_h} 
\end{align}
Plugging inequalities in \eqref{eq:reduce_to_g} and \eqref{eq:reduce_to_h} into $\mathbbm{1}_{\mcE_{\cup}^\C} \sum_{t=1}^T C_{N}(t)S(t)$, we have
\begin{align}
    \mathbbm{1}_{\mcE_{\cup}^\C} \sum_{t=1}^T C_{N}(t)S(t)  \leq \sum_{t=1}^T g(N_{N,a_t}(t)) \max_{i \in [N]} h(N_{i,a_t}(t)) \ .  \label{eq:sum_nice_event}
\end{align}
Note that $h(x)$ is a non-increasing function of $x$: it is equal to $\sqrt{2}$ for $x \leq \tau$, and it is decreasing for $x \geq \tau$. Furthermore, for $n \geq \tau $ values, we have $\sqrt{n} \geq \frac{\alpha m^2}{\kappa_{\min}} \geq \alpha$, and $\max\left\{ 1,\frac{\alpha}{\sqrt{n}}\right\}=1$. Hence,
\begin{align}
    g(n) &= \sqrt{2} \frac{\sqrt{n \kappa_{\max}}+\sqrt{\tau \kappa_{\min}}+1}{\sqrt{n \kappa_{\min}}-\sqrt{\tau \kappa_{\min}}+1} \ , \quad \forall n\geq \tau \ ,  \label{eq:g_definition}
\end{align}
which is also a decreasing function of $n$ for $n\geq \tau$. To use this behavior of $g(n)$, we split the \eqref{eq:sum_nice_event} into two parts as follows.
\begin{align}
     \sum_{t=1}^T g(N_{N,a_t}(t)) \max_{i \in [N]} h(N_{i,a_t}(t))  &=  \sum_{t=1}^T \mathbbm{1}_{\{N_{N,a_t}(t) < \tau\}} g(N_{N,a_t}(t)) \max_{i \in [N]} h(N_{i,a_t}(t))  \notag \\
     & + \sum_{t=1}^T \mathbbm{1}_{\{N_{N,a_t}(t) \geq \tau\}} g(N_{N,a_t}(t)) \max_{i \in [N]} h(N_{i,a_t}(t)) \ . \label{eq:sum_g_h}
\end{align}
We will bound each of the two summands next.
\paragraph{Bounding $\sum_{t=1}^T \mathbbm{1}_{\{N_{N,a_t}(t)< \tau\}} g(N_{N,a_t}(t)) \max_{i \in [N]} h(N_{i,a_t}(t))$.} Note that if $n < \tau$, $g(n)$ becomes
\begin{align}
    g(n) \overset{\eqref{eq:g_function}}&{=} \sqrt{2}\left( \sqrt{n \kappa_{\max}} + \sqrt{\tau \kappa_{\min}}\max\left\{1,\frac{\alpha}{\sqrt{n}}\right\} + 1 \right) \ , \quad n < \tau \ . \label{eq:g_function_small}
\end{align}
By noting that $\max_{i \in [N]}h(N_{i,a_t}(t))\leq \sqrt{2}$, we obtain
\begin{align}
       \sum_{t=1}^T \mathbbm{1}_{\{N_{N,a_t}(t)< \tau\}} g(N_{N,a_t}(t)) \max_{i \in [N]} h(N_{i,a_t}(t)) &\leq \sqrt{2} \sum_{t=1}^{T} \mathbbm{1}_{\{N_{N,a_t}(t)< \tau\}} g(N_{N,a_t}(t)) \\
       &\leq \sqrt{2} \sum_{n=1}^{\tau-1} g(n) \ . \label{eq:sum_of_g}
\end{align}
Substituting the expression of $g(n)$ in \eqref{eq:g_function_small} into \eqref{eq:sum_of_g}, and splitting it into two sums for $n \leq \floor{\alpha^2} $ and $n \geq \floor{\alpha^2}+1$ cases, we obtain
\begin{align}
 \sqrt{2}\sum_{n-1}^{\tau-1} g(n) \overset{\eqref{eq:g_function_small}}&{=} 2 \sum_{n=1}^{\tau-1} \left( \sqrt{n \kappa_{\max}} + \sqrt{\tau \kappa_{\min}}\max\left\{1,\frac{\alpha}{\sqrt{n}}\right\} + 1 \right) \\
    &= 2 \Bigg(\tau-1 + \sqrt{\kappa_{\max}}\sum_{n=1}^{\tau-1}\sqrt{n} + \sqrt{\tau \kappa_{\min}}\bigg(\sum_{n=1}^{\floor{\alpha^2}}\frac{\alpha}{\sqrt{n}}+ \sum_{n=\floor{\alpha^2}+1}^{\tau-1}1 \bigg) \Bigg) \ . \label{eq:sum_g_first} 
\end{align}
We bound the sum terms in \eqref{eq:sum_g_first} as follows:
\begin{align}
    \sum_{n=1}^{\tau-1}\sqrt{n} &\leq \int_{n=1}^{\tau} \sqrt{n} dn = \frac{2}{3}(\tau^{3/2}-1) \leq \frac{2}{3}\tau^{3/2} \ , \label{eq:sum_g_first_1} \\
    \mbox{and} \qquad \sum_{n=}^{\floor{\alpha^2}} \frac{\alpha}{\sqrt{n}} &= 2\alpha \sum_{n=1}^{\floor{\alpha^2}}(\sqrt{n}-\sqrt{n-1}) = 2\alpha (\sqrt{\floor{\alpha^2}}) \leq 2\alpha^2 \ . \label{eq:sum_g_first_2}
\end{align}
Plugging these results back to \eqref{eq:sum_g_first}, and using $\tau \geq \alpha^2$, we obtain
\begin{align}
    \sqrt{2}\sum_{n=1}^{\tau-1} g(n) &\leq 2 \Bigg(\tau-1 + \sqrt{\kappa_{\max}}\sum_{n=1}^{\tau-1}\sqrt{n} + \sqrt{\tau \kappa_{\min}}\bigg( \underset{\leq 2\alpha^2}{\underbrace{\sum_{n=1}^{\floor{\alpha^2}}\frac{\alpha}{\sqrt{n}}}} + \underset{\leq \tau - \alpha^2}{\underbrace{\sum_{n=\floor{\alpha^2}+1}^{\tau-1} 1}} \bigg) \Bigg) \\
    \overset{\eqref{eq:sum_g_first_1}}&{\leq}  2 \Bigg(\tau +  \frac{2}{3}\sqrt{\kappa_{\max}}\tau^{3/2} + \sqrt{\tau \kappa_{\min}}(\tau + \underset{\leq \tau}{\underbrace{\alpha^2}}) \Bigg) \\
    &\leq 2 \Bigg(\tau +  \frac{2}{3}\sqrt{\kappa_{\max}}\tau^{3/2} + 2\sqrt{\kappa_{\min}}\tau^{3/2} \Bigg)    \\
    &< 4\tau(\sqrt{\kappa_{\max}\tau} + \sqrt{\kappa_{\min} \tau} + 1) \\
    &= 2\sqrt{2}\tau g(\tau) \ . \label{eq:sum_g}
\end{align}
Hence, we have the following bound for the first summand
\begin{align}
    \sum_{t=1}^T \mathbbm{1}_{\{N_{N,a_t}(t)< \tau\}} g(N_{N,a_t}(t)) \max_{i \in [N]} h(N_{i,a_t}(t)) & \leq  2\sqrt{2}\tau g(\tau) \ , \label{eq:first_summand}
\end{align}
which is a constant term.

\paragraph{Bounding $\sum_{t=1}^T \mathbbm{1}_{\{N_{N,a_t}(t)\geq \tau\}} g(N_{N,a_t}(t)) \max_{i \in [N]} h(N_{i,a_t}(t))$.} Using the fact that $g(n)$ is a decreasing function for $n \geq \tau$, we have
\begin{align}
    \mathbbm{1}_{\{N_{N,a_t}(t) \geq \tau \}}  g(N_{N,a_t}(t)) &\leq g(\tau) \label{eq:sum_g_h_ita_pre1} \ , \\ 
    \mbox{and} \qquad \sum_{t=1}^T \mathbbm{1}_{\{N_{N,a_t}(t)\geq \tau\}} g(N_{N,a_t}(t)) \max_{i \in [N]} h(N_{i,a_t}(t)) \overset{\eqref{eq:sum_g_h_ita_pre1}}&{\leq} g(\tau) \sum_{t=1}^T \max_{i \in [N]} h(N_{i,a_t}(t)) \ . \label{eq:sum_g_h_ita}
\end{align}
The sum $\sum_{t=1}^T \max_{i \in [N]} h(N_{i,a_t}(t))$ is the final critical piece in the proof. $h(n)$ is a non-increasing function, and a decreasing one for $n \geq \tau$. However, the argument of $h$ in \eqref{eq:sum_g_h_ita} is changing due to taking maximum over $N$ possible values. We will prove in the following lemma that this can be compensated by having a $\sqrt{2N}$ factor on top of the optimal scaling behavior $\sqrt{T}$.

\begin{lemma}\label{lm:sum_h_ita}
The term $\sum_{t=1}^{T} \max_{i \in [N]} h(N_{i,a_t}(t))$ is bounded by
\begin{align}
    \sum_{t=1}^{T} \max_{i \in [N]} h(N_{i,a_t}(t)) < 2N  \left(\sqrt{2}\tau + \sqrt{\frac{2}{\kappa_{\min}}}\left(\sqrt{\frac{2T}{N}} + \sqrt{\tau} \log\left(\frac{T}{2N}\right) \right) \right) \ . \label{eq:h_side_1}
\end{align}
\end{lemma}
\begin{proof}
See Appendix~\ref{appendix:proofs}.
\end{proof}
We are ready to combine the pieces to reach the final result. We apply Lemma \ref{lm:sum_h_ita} and \eqref{eq:sum_g} in \eqref{eq:sum_g_h}, and then on \eqref{eq:sum_nice_event} to obtain
\begin{align}
     \mathbbm{1}_{\mcE_{\cup}^\C} \sum_{t=1}^T C_{N}(t)S(t)  &\leq \sum_{t=1}^T g(N_{N,a_t}(t)) \max_{i \in [N]} h(N_{i,a_t}(t)) \\ 
     &<  \frac{4g(\tau)}{\sqrt{\kappa_{\min}}}\sqrt{NT} + 2\sqrt{2}(N+1)\tau g(\tau) + \frac{2\sqrt{2}N\sqrt{\tau} g(\tau)}{\sqrt{\kappa_{\min}}} \log\left(\frac{T}{2N}\right) \ ,
\end{align}
where $g(\tau) = \sqrt{2}(\sqrt{\tau \kappa_{\max}}+\sqrt{\tau \kappa_{\min}}+1)$ since $\tau$ makes the denominator in \eqref{eq:g_definition} equal to 1. Note that the upper bound we have just found is not a random variable but a constant. Then, $\E\left[ \mathbbm{1}_{\mcE_{\cup}^\C} \sum_{t=1}^T C_{N}(t)S(t)\right]$ is immediately upper bounded by this result. Also recall the result in \eqref{eq:error_event_bound}, based on which we have 
\begin{align}
    \E \left[\mathbbm{1}_{\mcE_{\cup}} \sum_{t=1}^T C_{N}(t)S(t) \right]  \overset{\eqref{eq:error_event_bound}}&{\leq}  \frac{m}{T} + \frac{2m}{3} + 1 \ .
\end{align}
Therefore, the final result is 
\begin{align}
    \lambda_{T} &= \E \left[\mathbbm{1}_{\mcE_{\cup}^{\C}} \sum_{t=1}^T C_{N}(t)S(t) \right] + \E \left[\mathbbm{1}_{\mcE_{\cup}} \sum_{t=1}^T C_{N}(t)S(t) \right] \\
    &< \frac{4g(\tau)}{\sqrt{\kappa_{\min}}}\sqrt{NT} + 2\sqrt{2}(N+1)\tau g(\tau) + \frac{2\sqrt{2}N\sqrt{\tau} g(\tau)}{\sqrt{\kappa_{\min}}} \log\left(\frac{T}{2N}\right) + \frac{m}{T} + \frac{2m}{3} +  1 \ .
\end{align}
Finally, note that $\alpha = \mcO(\sqrt{\log(dNT)})$ and $\tau = \mcO(\log(dNT))$. Also, $\kappa_{\max}$ and $\kappa_{\min}$ are independent of $T$. Hence, ignoring the logarithmic terms and constants,
\begin{align}
    \lambda_{T} &= \tilde \mcO(\sqrt{NT}) \ . 
\end{align}
\end{proof} 
Finally, we combine Theorem \ref{th:regret_ucb_part1} and \ref{th:regret_ucb_part2} to obtain the regret bound of our algorithm, formalized next.
\begin{theorem}\label{th:regret_ucb_final}
Under Assumption~\ref{assumption:boundedness}, the regret of LinSEM-UCB is
\begin{align}
    \E[R(T)] = \tilde \mcO( d^{L+\frac{1}{2}} \sqrt{NT}) \ .
\end{align}
\end{theorem}
\begin{proof}
Since $\beta_{T} = m_B + \sqrt{2\log(2NT)+(d+1)\log(1+m^2T/(d+1))}$ and $m_B$ is constant, ignoring the poly-logarithmic factors, $(\beta_{T}+m_B)^{L+1}$ contributes $(d+1)^{\frac{L+1}{2}}$ factor to the result of Theorem \ref{th:regret_ucb_part1}. Factoring the result of Theorem \ref{th:regret_ucb_part2} for $\lambda_{T}$, we obtain $\E[R(T)] = \tilde \mcO( d^{L+\frac{1}{2}} \sqrt{NT})$.
\end{proof}

\section{\algonameTS{} Algorithm}\label{sec:TS}
\begin{algorithm}[t]
\caption{\algonameTS{}}
\label{alg:ts_algorithm}
\begin{algorithmic}[1]
\State \textbf{Input:} Causal graph $\mcG$, action set $\mcA$, hyperparameter $\sigma$, prior distribution~$\pi_0$.  
\State \textbf{Initialization:} Initialize parameters for $2N$ linear problems:
\State  $[\bB(0)]_{i} = \mathbf{0}_{(N+1)\times 1},  [\bB^{*}(0)]_{i} = \mathbf{0}_{(N+1) \times 1}, \ \forall i \in [N] \ $ \Comment{initiliaze estimates for parameter vectors}
\State  $\bV_{i}(0) = I_{N+1}, \ \bV^{*}_{i}(0) = I_{N+1}, \ g_{i}(0) = \mathbf{0}_{(N+1) \times 1}, \ g^{*}_{i}(0) = \mathbf{0}_{(N+1) \times 1}$ \Comment{initiliaze auxiliary parameters}
\For {$t = 1,2,\ldots T$}
    \State ${\breve \bW} \sim \pi_{t-1}(\bW \med X(1),\dots, X(t-1))$  
    \For {$a \in \mcA$}
        \State Construct $\breve \bB_{a}$ from rows of ${\breve \bW}$ similarly to \eqref{eq:Ba_construct}. \label{line:Ba_construct}
        \State $\hat \mu_a \leftarrow f(\breve \bB_{a})$ \Comment{expected reward under action $a$} \label{line:expected_rewards} 
    \EndFor
    \State $a_t = \argmax_{a} \hat \mu_a$ \Comment{select the action that maximizes expected reward} \label{line:select_best_arm}
    \State Pull $a_t$, observe $X(t) = (1,X_1(t),\dots, X_N(t))^\top$. 
    \State Update the posterior to $\pi_{t}$ under the linear model, similarly to Lines \ref{line:ucb_update_start}-\ref{line:ucb_update_end} of Algorithm \ref{alg:ucb_algorithm}. \label{line:ts_update}
\EndFor
\end{algorithmic}
\end{algorithm}%

Thompson Sampling ({\rm TS})-based algorithms gradually refine the posterior distributions of the reward for each action and select actions by sampling from their posterior distributions. The actions are selected sequentially in a way that they balance the exploitation and the exploration processes. 
To this end, a {\rm TS} algorithm needs to update the posterior distributions of all arms. Similar to \algonameUCB{}'s improvement over {\rm UCB}, the known causal structure can be leveraged to improve the performance of ${\rm TS}$ in causal bandits.

\begin{assumption}[Bounded prior]\label{assumption:boundedness:prior}
The domain of the prior distribution of parameters is bounded as $\mcW \triangleq \{\bW \in R^{2N \times N} : \norm{[\bW]_{i}}\leq m_B \ \ \forall i \in [2N]\}$.
\end{assumption}
We denote the prior distribution of $\bW$ by $\pi_0$, and denote its posterior given the filtration generated until time $t$ by $\pi_t$ for $t\in\N$. The \algonameTS{} algorithm at time $t$ samples from the posterior distribution $\pi_{t-1}$. Subsequently, it constructs the weight matrix for each intervention action and computes the corresponding expected reward. Next, the intervention with the highest expected reward is selected, and the graph instance $X(t)$ is observed. Finally, the posterior distributions of the $2N$ independent weight vectors that constitute $\bW$ are updated according to the linear model.

We note that \algonameTS{} has two advantages over \algonameUCB{}. First, the constrained optimization problem for computing ${\rm UCB}_a(t)$ has no closed-form solution to our knowledge. Hence, \algonameUCB{} is computationally expensive to implement in practical settings. \algonameTS{} does not suffer from this problem. Secondly, \algonameUCB{} requires to know bound $m$ on $\norm{X}$ and horizon $T$ since the radius of the confidence intervals $\beta_T$ depends on $m$ and $T$. In contrast, \algonameTS{} does not require to know $m$ or $T$. Hence, \algonameTS{} is an anytime algorithm.
Furthermore, we can ensure a performance guarantee similar to that of Theorem \ref{th:regret_ucb_final} for the Bayesian setting. This guarantee is formalized in the next theorem.
\begin{theorem}\label{th:regret_ts}
Under Assumption~\ref{assumption:boundedness} and Assumption~\ref{assumption:boundedness:prior}, the Bayesian regret of \algonameTS{} is
\begin{align}
    {\rm BR}(T) = \tilde \mcO( d^{L+\frac{1}{2}} \sqrt{NT}) \ .
\end{align}
\end{theorem}
\begin{proof}
The main tools of the proofs, e.g., Lemma \ref{lm:bound_l_paths}, are similar to that proof of frequentist result. However, the proof for Thompson sampling critically relies on the property that conditional distributions of different actions are almost surely the same given proper filtration. Importantly, this filtration is different than the one used for concentration inequalities in the proof of Lemma~\ref{lm:error_event_int}. Finally, the ranges of the parameters regarding the intervention space, e.g., $\kappa_{\min}$, are different from that of the frequentist analysis and are examined carefully. For detailed proof, see Appendix~\ref{appendix:proof_TS}.
\end{proof}

\subsection{\algonameTS{-Gaussian} Algorithm}\label{sec:TS_Gaussian}
We have provided \algonameTS{} and its performance guarantees under the bounded noise and bounded prior distributions. Sampling from the posteriors under bounded priors can be done via different numerical techniques yet they can be time-consuming. For instance, commonly used Gibbs sampling has a long convergence time, especially in the high-dimensional regime. To circumvent this, in this section, we present a modified algorithm that leverages Gaussian priors, which we refer to as \algonameTS{-Gaussian}. We note that the theoretical guarantees presented in Section \ref{sec:TS} do not apply to \algonameTS{-Gaussian} since boundedness assumptions are not satisfied.
\begin{algorithm}[t]
\caption{\algonameTS{-Gaussian}}
\label{alg:ts_algorithm_gaussian}
\begin{algorithmic}[1]
\State \textbf{Input:} Causal graph $\mcG$, action set $\mcA$, hyperparameter $\sigma$.
\State \textbf{Initialization:} Initialize parameters for $2N$ linear problems:
\State  $[\bB(0)]_{i} = \mathbf{0}_{N\times 1},  [\bB^{*}(0)]_{i} = \mathbf{0}_{N \times 1}, \ \forall i \in [N] \ $ \Comment{initiliaze estimates for parameter vectors}
\State  $\bV_{i}(0) = I_{N+1}, \ \bV^{*}_{i}(0) = I_{N+1}, \ g_{i}(0) = \mathbf{0}_{N \times 1}, \ g^{*}_{i}(0) = \mathbf{0}_{N \times 1}$ \Comment{initiliaze auxiliary parameters}
\For {$t = 1,2,\ldots T$}
    \For {$i \in \{1,\dots,N\}$}
        \State $[\breve \bB(t)]_{i} \sim \mcN([\bB(t-1)]_{i} , \sigma^2 [\bV_{i}(t-1)]^{-1})$ 
        \label{line:sample_weights_obs}
        \State $[\breve \bB^{*}(t)]_{i} \sim \mcN([\bB^{*}(t-1)]_{i} , \sigma^2 [\bV^{*}_{i}(t-1)]^{-1})$
        \label{line:sample_weights_int}
    \EndFor
    \For {$a \in \mcA$}
        \State Construct $\breve \bB_{a}$ from rows of ${\breve \bW \triangleq [{\breve \bB} \ \ {\breve \bB^{*}}]}$ similarly to \eqref{eq:Ba_construct}.
        \State $\hat \mu_a \leftarrow f(\breve \bB_{a})$
    \EndFor
    \State $a_t = \argmax_{a} \hat \mu_a$ \Comment{select the action that maximizes the expected reward}
    \State Pull $a_t$, observe $X(t) = (X_1(t),\dots, X_N(t))^\top$. 
    \State Update the Gaussian posteriors under the linear model. 
\EndFor
\end{algorithmic}
\end{algorithm}%

\paragraph{Gaussian posterior computation.} We use a Gaussian distribution for the prior distribution of the parameters and a Gaussian likelihood model for the data. We follow the standard Bayesian posterior inference \citep{agrawal2013thompson}. To summarize, consider the posterior computation of parameters $[\bB]_{i}$. The likelihood function of $X_{i}(t)$ is given by $\mcN(\langle [\bB]_{i}, X_{\Pa(i)}(t) \rangle , \sigma_{i}^2)$ according to the linear SEM. We start with the prior $\mcN(\mathbf{0}_{N\times 1},I_{N+1})$ for $[\bB]_{i}$. If the prior of $[\bB]_{i}$ at round $t$ is $\mcN([\bB(t-1)]_{i},\sigma^2 [\bV_{i}(t-1)]^{-1})$, then the posterior distribution at round $t$ is computed as $\mcN([\bB(t)]_{i}, \sigma^2 [\bV_{i}(t)]^{-1})$. A similar computation shows that the posterior distribution of $[\bB^{*}]_{i}$ at round $t$  is $\mcN([\bB^{*}(t)], \sigma^2 [\bV^{*}_{i}(t)]^{-1})$.

\paragraph{Algorithm details.} Algorithm \ref{alg:ts_algorithm_gaussian} presents the \algonameTS{-Gaussian} algorithm. Since we have exact forms of the posteriors of weight vectors, we explicitly write the sampling of parameters in lines \ref{line:sample_weights_obs}-\ref{line:sample_weights_int}. Estimated weight matrices and intervention selection are all formed similarly to those of Algorithm \ref{alg:ts_algorithm}. Finally, interventional and observational posteriors are updated according to the linear model specified earlier.

\section{Lower Bounds}\label{sec:lower-bound}
In this section, we provide minimax lower bounds on the regret. The main observation is that the dependence of these lower bounds on the graph parameters (maximum degree $d$ and longest causal path length $L$)  conforms to that of the achievable regret bounds. For this purpose, we consider a general class of graphs and prove a minimax lower bound by showing that there exists a causal bandit environment parameterized by a linear SEM, for which the regret of any algorithm is at least a constant factor of $d^{\frac{L}{2}-2}\sqrt{T}$.
\paragraph{A general class of graphs.} Note that given values of $d$ and $L$ necessitate that the graph size $N$ is at least $\max\{d+1,L+1\}$, and at most $\sum_{\ell=0}^{L}d^{\ell} < 2d^{L}$. On the other hand, $d$ and $L$ can vary independently without imposing restrictions on one another. Therefore, we aim to prove a lower bound for fixed $d$ and $L$ whereas $N$ can vary freely. In the proof of the following theorem, we construct  graph structures and bandit instances that are flexible enough to accommodate all possible choices of $d$ and $L$, and the graph size $N$ that can have varying sizes in the range $N\in\{(dL-d+1),\dots, d^L+(dL-d+1)\}$. Since $N < 2d^{L}$, this class contains large enough graphs that have the same order as the largest possible graphs for fixed $d$ and $L$.
\begin{theorem}\label{th:lower-bound}
Given the pair $(d,L)$, for any $N \in \{d(L-1)+2, \dots, d^L+d(L-1)+1\}$, there exists a causal bandit instance with graph parameters $(d,L,N)$ such that expected regret of any algorithm is at least
    \begin{align}
        \E[R(T)] \geq c d^{\frac{L}{2}-2}\sqrt{T} \ ,
    \end{align}
    where $c > 0$ is a constant.
\end{theorem}
\emph{Sketch of the proof.} 
The central idea of the proof is as follows. Two linear SEM causal bandit instances that differ by only one edge parameter are hard to distinguish. At the same time, we can construct them to have different optimal actions, indicating that a selection policy cannot incur small regret for both at the same time under the same data realization. 
Note that, the difference of the rewards, or equivalently the regrets, observed by these two bandit instances under the same action can be computed by tracing the effect of that differing edge parameter over all the paths that end at the reward node. We use a hierarchical structure that consists of $L$ layers and $d$ nodes at each layer to maximize the number of such paths. We also use $d$-ary trees to control the number of nodes in the graph without affecting the lower bound analysis. See Appendix~\ref{sec:proof-lower-bound} for the detailed proof.

\begin{remark}\label{remark:lower-bound}
We remark that optimism in the face of uncertainty linear bandit algorithm (OFUL) of \citep{abbasi-yadkori2011} has $\mcO(d\sqrt{T})$ regret, whereas the lower bound for linear bandit problem with finite arms is $\mcO(\sqrt{dT})$. Note that the linear bandit problem can be considered a special case of our framework with $L=1$. Since our confidence intervals rely on the same principles as OFUL, the gap between lower and upper bounds established in the paper, i.e., $d^{\frac{L}{2}}$ factor, is due to the aggregated effect of $L$ layers. Tighter regret results would have been possible by using a more sophisticated approach for building the confidence intervals along with elimination, such as the SupLinUCB algorithm of \citep{chu2011contextual}. However, this investigation is out of the scope of this paper.
\end{remark}

\section{Simulations}\label{sec:simulations}

The main novelty of our work is to use causal graph knowledge without knowing the interventional distributions. Furthermore, we use soft interventions on continuous variables. To the best of our knowledge, there does not exist any causal bandit algorithm that suits this setting. The standard linear bandit algorithms are also not applicable to causal bandits for linear SEMs due to the feature vectors being unknown. Therefore, we compare the performance of our method to that of the non-causal {\rm UCB} algorithm. For practical implementation, we use Algorithm \ref{alg:ts_algorithm_gaussian} for the simulations in this section.~\footnote{The codebase for reproducing the simulations are available at \url{https://github.com/bvarici/causal-bandits-linear-sem}.}
\paragraph{Parameterization.} We consider linear SEMs with soft interventions under the Bayesian setting. For a given causal graph structure, the prior distributions of the non-zero entries of observational parameters $\bB$ are sampled independently at random according to the uniform distribution on $[-1,-0.25] \cup [0.25, 1]$. Priors for interventional weights are set as negatives of observational weights, $\bB^{*} = - \bB$, so that the effect of an intervention is strong. To approximate the Bayesian setting, the true weights $[\bB]_{i}= -[\bB^{*}]_{i}$ are sampled 10 times from a Gaussian prior with small variances. The distribution of the individual noise terms $\epsilon_i$ is set to $\mcN(0,1)$ for all nodes. We note all parameters are unknown to the learner. Simulations for each sampled parameter set are repeated 10 times. We simulate causal graphs under two structural families.

\paragraph{Hierarchical Graphs:} We consider hierarchical graph structures with $L$ layers and degree $d$ constructed as follows. Each layer $\ell \in \{1,\dots,L\}$ contains $d$ nodes, each node at layer $\ell < L$ is a parent of the nodes in the next layer $(\ell+1)$, and the nodes in the final layer $L$ are the parents of the reward node. The total number of nodes is $N=dL+1$. Figure \ref{fig:hierarchical_sample} depicts the hierarchical structure for $d=3$ and $L=2$. Every node is set to be intervenable, leading to $|\mcA|=2^{dL+1}$ number of possible interventions. We construct this structure by observing that both our regret upper bound and the minimax lower bound are a function of $d$ and $L$.

Figure \ref{fig:hierarchical_d_3_l_2_compare} compares the cumulative regret of \algonameTS{-Gaussian} and non-causal {\rm UCB} algorithms for $d=3$ and $L=2$. We observe that \algonameTS{-Gaussian} algorithm significantly outperforms the standard bandit algorithm by exploiting the known causal structure. Next, we look into the graph structure's effect on the regret of \algonameTS{-Gaussian}. In Figure \ref{fig:hierarchical_vary}, we plot the cumulative regrets at $T=5000$ for different values of $d$ and $L$. We note that the curves in Figure \ref{fig:hierarchical_vary} become steeper as $L$ increases. For instance, the green curve ($L=4$) has a faster-growing regret than the blue curve ($L=2$). On the other hand, the growth of regret is much slower than the growth of the cardinality $|\mcA|=2^{dL+1}$. These observations imply that our theoretical upper bounds that scale with $d^{L}$ instead of $|\mcA|$ are realized empirically.

\begin{figure}[t]
    \centering
    \begin{subfigure}[t]{0.32\textwidth}
        \centering
        \includegraphics[width=0.6\linewidth]{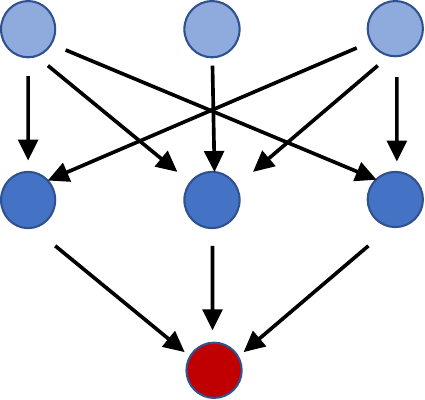}
        \caption{Hierarchical structure with \\ $L=2$ layers and degree $d=3$.}
        \label{fig:hierarchical_sample}
    \end{subfigure}
    \begin{subfigure}[t]{0.32\textwidth}
        \centering
        \includegraphics[width=0.9\linewidth]{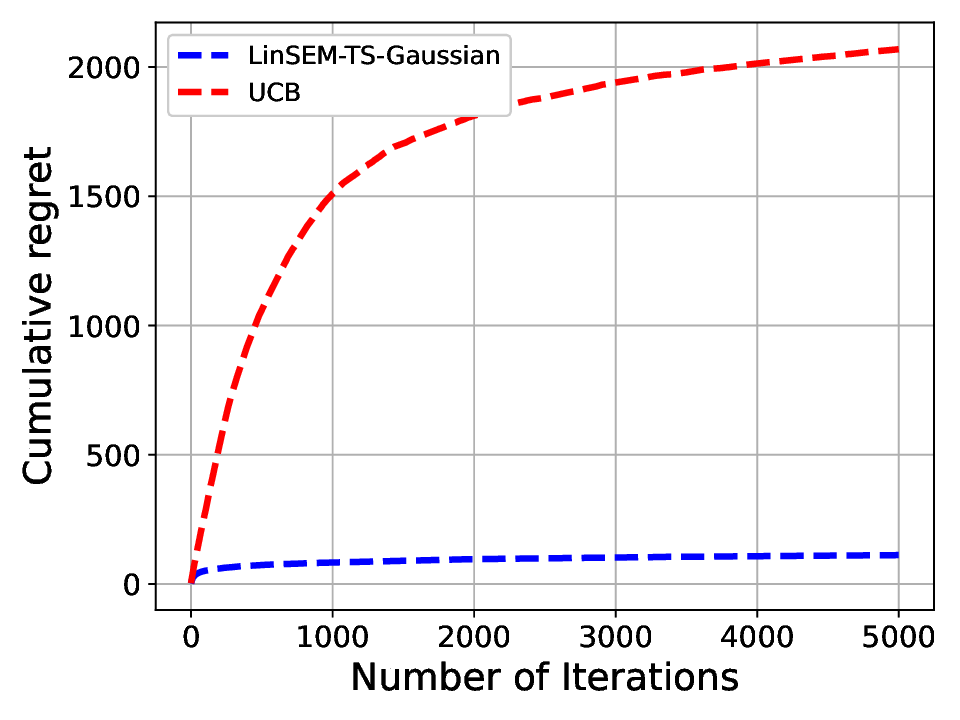}
        \caption{Cumulative regret comparison for $d=3$, $L=2$.}
        \label{fig:hierarchical_d_3_l_2_compare}
    \end{subfigure}    
    \begin{subfigure}[t]{0.32\textwidth}
        \centering
        \includegraphics[width=0.9\linewidth]{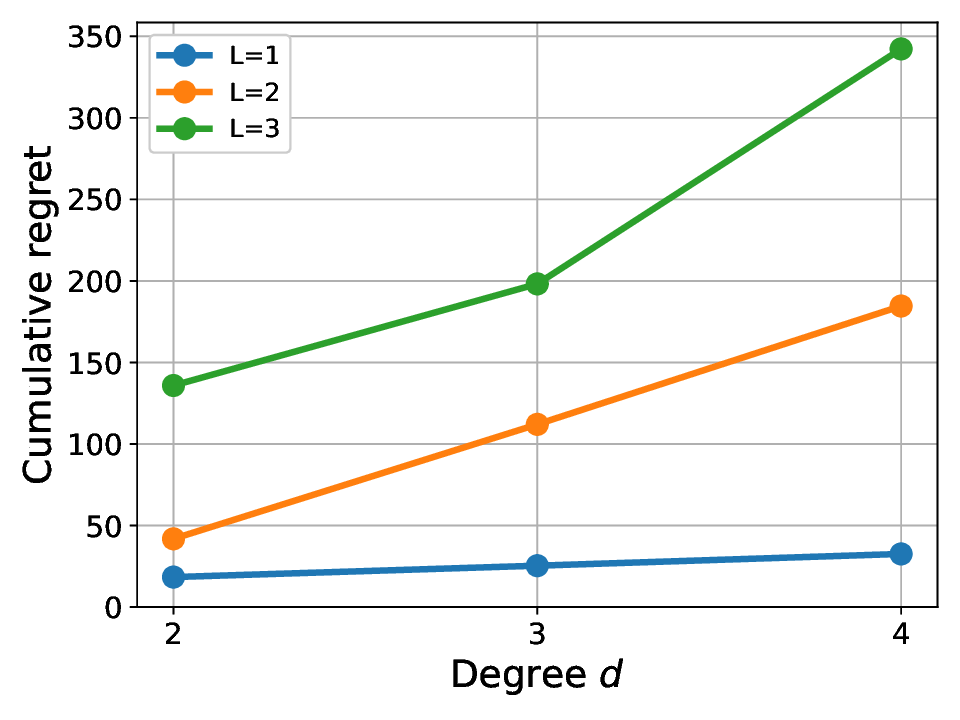}
        \caption{Cumulative regret at $T=5000$.}
        \label{fig:hierarchical_vary}
    \end{subfigure}   
    \caption{Simulation results for hierarchical causal graphs.}
\end{figure}

\paragraph{Enhanced parallel bandits:} We construct a structure similar to the parallel bandits model of \citet{lattimore2016causal} with additional edges. In this model, each node $\{1,\dots,N-1\}$ is a parent of the reward node $N$. In addition, we assign one parent to each node $i \in \{2,\dots,N-1\}$, which is chosen uniformly at random from the nodes $\{1,\dots,i-1\}$. Every node is intervenable, which gives $|\mcA|=2^{N}$ number of possible interventions. One example of this structure with 5 nodes is given in Figure \ref{fig:enhanced_parallel_sample}. We randomly generate 5 structures for each value of $N \in \{5,6,7,8,9\}$ and run \algonameTS{-Gaussian} and {\rm UCB} for each graph. 

Figure~\ref{fig:enhanced_parallel_N_5_compare} compares the cumulative regret of \algonameTS{-Gaussian} and {\rm UCB} algorithms for $N=5$. Similarly to the previous setting,  \algonameTS{-Gaussian} outperforms {\rm UCB} that does not use the causal relationships among the arms. Secondly, we fix horizon $T$ at 5000 and compare the performance of  \algonameTS{-Gaussian} as $N$ grows. We note that the maximum degree $d$ is equal to $(N-1)$ for enhanced parallel bandits. The largest path length $L$ is not strictly controlled by $N$ and can take values between $2$ and $(N-1)$. Figure \ref{fig:enhanced_vary_N_compare_ucb} shows that ${\rm UCB}$ algorithm’s regret grows exponentially as $N$ grows, as $|\mcA|$ does, whereas \algonameTS{-Gaussian} incurs a much smaller regret.
\begin{figure}[t]
    \centering
    \begin{subfigure}[t]{0.32\textwidth}
        \centering
        \includegraphics[width=0.9\linewidth]{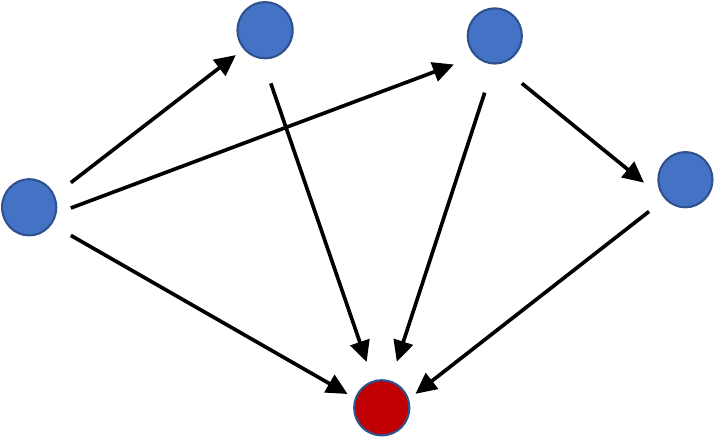}
        \caption{Sample structure with $N=5$}
        \label{fig:enhanced_parallel_sample}
    \end{subfigure}
    \begin{subfigure}[t]{0.32\textwidth}
        \centering
        \includegraphics[width=0.9\linewidth]{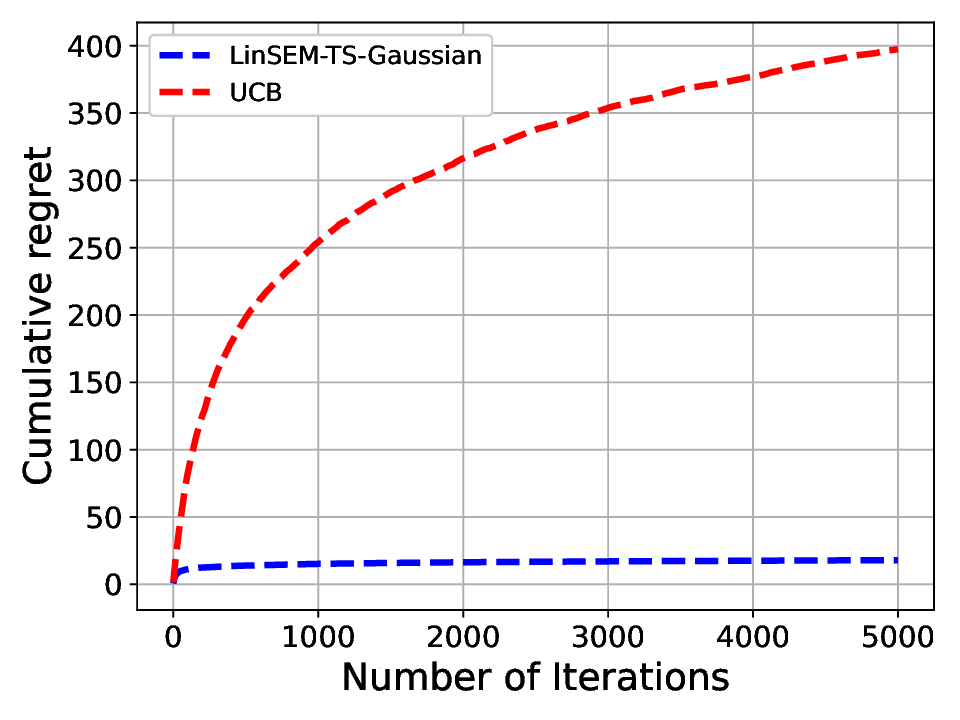}
        \caption{Cumulative regret comparison for $N=5$.}
        \label{fig:enhanced_parallel_N_5_compare}
    \end{subfigure}    
    \begin{subfigure}[t]{0.32\textwidth}
        \centering
        \includegraphics[width=0.9\linewidth]{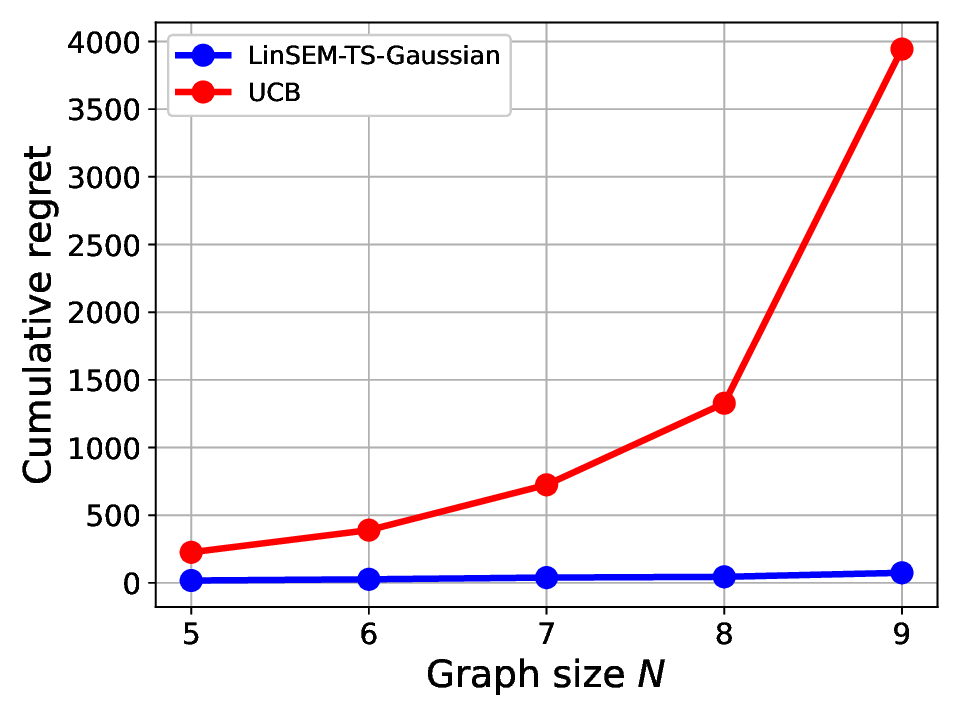}
        \caption{Cumulative regret at $T=5000$}
        \label{fig:enhanced_vary_N_compare_ucb}
    \end{subfigure}   
    \caption{Simulation results for enhanced parallel bandits.}
\end{figure}

\paragraph{Graph misspecification.}
In the simulations above, the learner is given the perfect knowledge of the causal structure following the theoretical analysis in previous sections. In the following simulations, we investigate a possible scenario in which the structure given to the learner deviates from the true structure by a few edges. This can be a practical scenario where the domain expert generously assigns edges from possible causes to effects. In other words, the learner can be given a graph that is a supergraph of the true graph with some additional edges. We simulate this graph misspecification setting in hierarchical graphs by randomly adding incorrect edges to the graph structure provided to the learner. The incorrect edges are chosen in a way that preserves acyclicity. The number of incorrect edges is set to $2$ for each graph instance, which is reasonable given that the graph structures have moderate sizes.

Figure~\ref{fig:hierarchical_vary_d_graph_mis} shows that our method still works when the knowledge of graph structure is perturbed by a small number of edges. We interpret this result by recalling our theoretical analysis. Since the true parents of a node $i$ is a subset of the assumed parents of it, the guarantees of estimating $[\bB]_i$ and $[\bB^*]_i$ would still follow through, albeit for a greater degree than the true degree of node $i$. The overall effect is a potential increase in the graph's $d$ and $L$ values. The results in Figure~\ref{fig:hierarchical_vary_d_graph_mis} align with our expectations as the regret under graph misspecification is only slightly worse than the case with perfect graph knowledge.

\begin{figure}[t]
    \centering
    \includegraphics[width=0.6\linewidth]{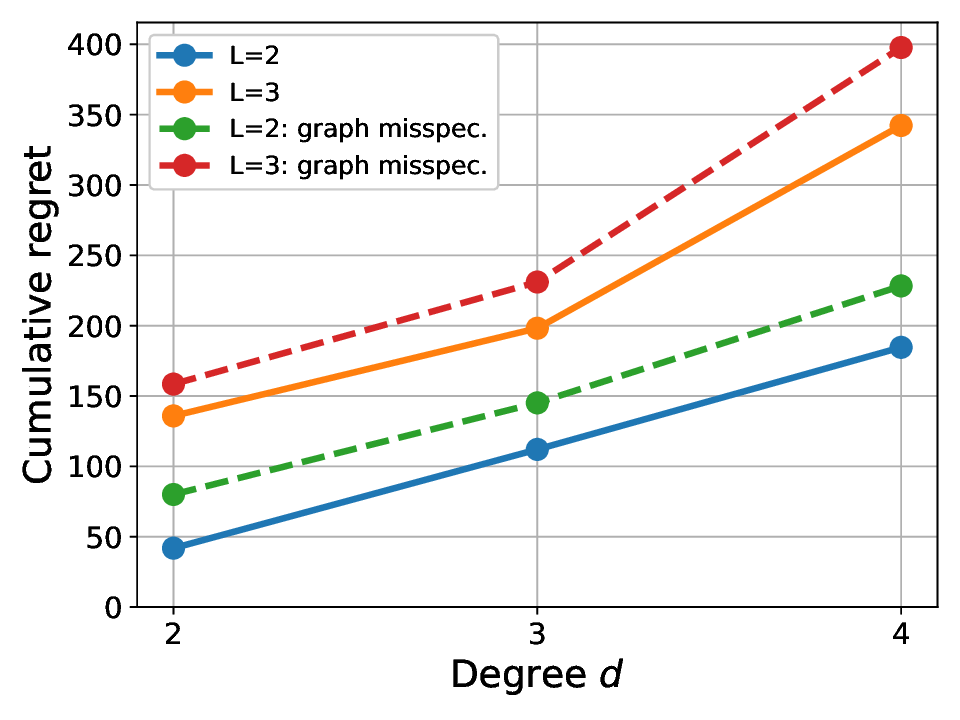}
    \caption{Cumulative regret comparison (at $T=5000$) of the main setting with perfect knowledge of the graph and graph misspecification setting for varying $d$ and $L$ values. Our algorithm incurs sublinear regret even when there are misspecified edges.}
    \label{fig:hierarchical_vary_d_graph_mis}
\end{figure}

\paragraph{Knowledge of interventional distributions}
Finally, we assess the effect of knowing or lacking the interventional distributions through simulations. The interventional distributions of the reward's parents $\{\P_{a}(\Pa(N)):a\in\mcA\}$ in linear SEMs correspond to knowing all of the weights $[\bB_a]_{i}$ for non-reward nodes $i \in [N-1]$. In this case, the estimation problem reduces to estimating only the two weight vectors $[\bB]_{N}$ and $[\bB^{*}]_{N}$ corresponding to the reward node. Hence, the causal bandit problem becomes significantly easier when the interventional distributions are given a priori.

Figure~\ref{fig:hierarchical_vary_d_compare_known_dist} shows that the effect of $L$ with a fixed value of $d$ on the hierarchical graph’s cumulative regret is more significant when the distributions $\{\P_{a}(\Pa(N)):a\in\mcA\}$ are unknown. This observation can be explained by the fact that there is no compounded estimation error through $L$-length paths when the knowledge of the distributions is given, and the problem essentially reduces to a linear bandit. These empirical results illustrate the relative difficulty of our problem setting, i.e., lacking the knowledge of interventional distributions.
\begin{figure}[t]
    \centering
    \includegraphics[width=0.6\linewidth]{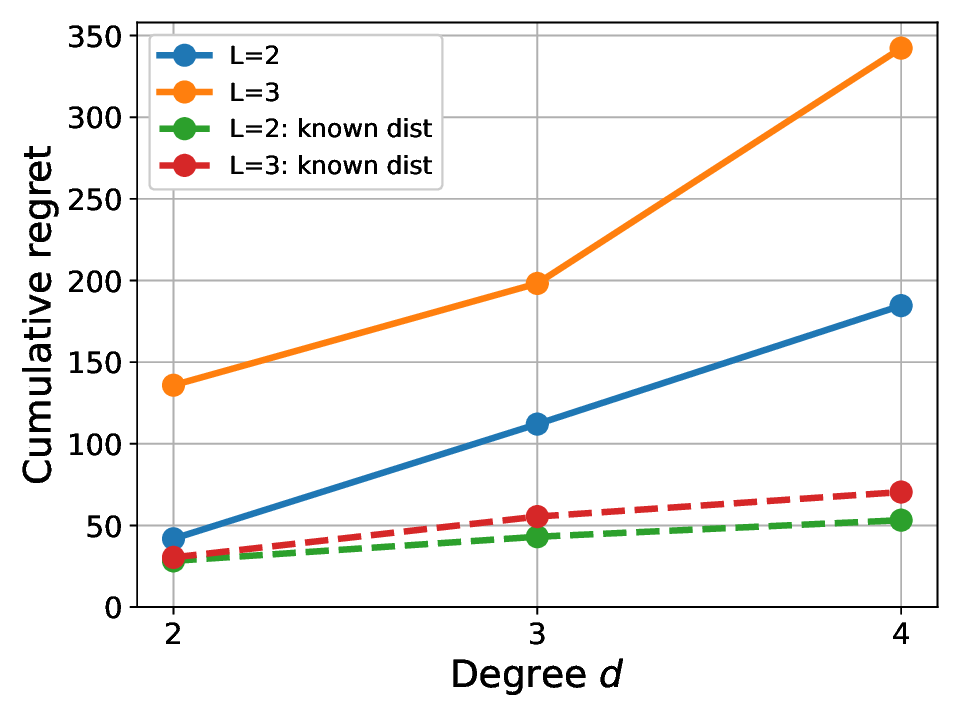}
    \caption{Cumulative regret at $T=5000$ with knowledge and without the knowledge of the interventional distributions for varying $d$ and $L$ values.}
    \label{fig:hierarchical_vary_d_compare_known_dist}
\end{figure}

\section{Discussion}

\paragraph{Background.} There have been considerable recent advances in understanding the fundamental performance limits of causal bandit problems under various settings. The hallmark of the results is establishing theoretical advantages over non-causal methods. The extent of such advantages critically hinges on the extent of information known about the topology of the causal graph and the interventional distributions. Recent studies have advanced in the direction of having minimal assumptions on these two key factors. Some of these studies, nevertheless, rely on other tacit or simplifying assumptions such as adopting single-node interventions and Boolean variables. Motivated by the applications in which domain knowledge suffices to construct a causal graph (e.g., the computational advertising model given in Figure~\ref{fig:soft-motivation}), an important direction of studies, which is also the focus of this paper, 
is assuming that the topology is known while the intervention distributions are unknown.

\paragraph{Contributions.} Our focus has been on designing an optimal sequence of interventions when the intervention space can grow with the graph. Specifically, mainly for analytical clarity, we have assumed one observational and one interventional model per node. That, however, can be readily extended to multiple interventional models per node, albeit at the expense of a higher degree of parameterization. We have addressed the open theoretical question of finding an optimal sequence of interventions for optimizing a graph reward value. This is done for the class of linear systems in which the causal relationships follow linear models. This renders that the nodes' conditional distributions have locally linear properties. We have allowed flexible intervention types (soft interventions that can be seamlessly extended to hard interventions), continuous variables, and exponentially large intervention space. Under these settings, we have characterized upper and lower bounds on the hitherto unknown cumulative regret under appropriate boundedness assumptions. We have shown that the achievable regret bounds depend on the properties of the causal graph but do not grow exponentially with the graph size, whereas the cardinality of the intervention space does. Furthermore, we have established a lower bound to show that this dependence on the graph properties is information-theoretic.

\paragraph{Future directions.} There are several open theoretical questions in the setting we have investigated. One pertains to relaxing the linear structural equation model. This can be done, for instance, by using reproducing kernel Hilbert space with proper approximation methods. Investigating the graphs with partially known structures, e.g., topological orders, is another important direction. Investigating the sensitivity of the results to model mismatch, contamination, or fluctuation is necessary to understand these algorithms' robustness. Finally, there is a scope for improving the dependence of our regret guarantees on the properties of intervention space. Given that regret scales with $\sqrt{T}$, a more sophisticated analysis would indicate that actions with autocorrelation matrix ($\Sigma_{i,a}$) that are closer to that of the optimal arm ($\Sigma_{i,a^*}$) are played overwhelmingly often. Then, $\frac{1}{\sqrt{\kappa_{\rm \min}}}$ factor, in which $\kappa_{\rm \min}$ is a lower bound over the intervention space, might be too conservative.

\newpage

\appendix

\section{Notations}\label{appendix:notations}
The following table summarizes the commonly used notations throughout the paper.
\begin{align*}
    \mcA, \mcE  &: \text{calligraphic letters for sets and events} \\
    X_{i} &: \text{random variables of the causal graph} \\
    X_{i}(t) &: \text{value of node $i$ at time $t$} \\
    X &: (X_1,\dots,X_N)^\top \\
    X(t) &: (1,X_{1}(t),\dots,X_{N}(t))^\top \quad \text{data vector at time $t$} \\    
    X_{\mcS}(t) &: X(t) \odot \mathbf{1}(\mcS) \quad \forall \mcS \subseteq \mcV \\
    X_{\Pa(i)}(t) &: \text{follows the previous line} \\
    \epsilon_i &: \text{noise variable for node $i$} \\
    \epsilon &: (1,\epsilon_1,\dots,\epsilon_N)^\top \\
     \epsilon(t) &: (1,\epsilon_{1}(t),\dots,\epsilon_{N}(t))^\top \quad \text{noise vector at time $t$} \\
     \bA, [\bA]_{i}, [\bA]_{i,j} &: \text{a matrix $\bA$, its column $i$, and its entry at row $i$ and column $j$, respectively} \\
    \bB_{a}, [\bB_{a}]_{i} &: \text{weight matrix for intervention $a$, and its $i$-th column} \\
    \bB_{a}(t), [\bB_{a}(t)]_{i} &: \text{estimate of $\bB_{a}$ at time $t$, and its $i$-th column} \\
    \bB_{a}^{\ell}, \big[\bB_{a}^{\ell}\big]_{i} &: \text{$\ell$-th power of $\bB_{a}$, and its $i$-th column} \\
    \bB_{a}^{\ell}(t), \big[\bB_{a}^{\ell}(t)\big]_{i} &: \text{$\ell$-th power of $\bB_{a}(t)$, and its $i$-th column}. \\ 
    \bDelta_{a}(t) &: \bB_{a}(t) - \bB_{a} \\
    [\bDelta_{a}(t)]_{i} &: [\bB_{a}(t)]_{i} - [\bB_{a}]_{i} \\
    \bDelta_{a}^{(\ell)}(t) &: \bB_{a}^{\ell}(t) - \bB_{a}^{\ell} \\
    [\bDelta_{a}^{(\ell)}(t)]_i &: \big[\bB_{a}^{\ell}(t)\big]_{i} - \big[\bB_{a}^{\ell}\big]_{i} \\ 
    \bV_{i}(t), \bV^{*}_{i}(t) &: \text{regularized gram matrices for node $i$, observational and interventional} \\
    \bD_{i}(t), \bD^{*}_{i}(t) &: \text{data matrices for parents of node $i$} \\
    \bV_{i,a}(t), \bD_{i,a}(t) &: \text{relevant gram and data matrices for node $i$ under intervention $a$} \\
    N^{*}_{i}(t) &: \text{number of times node $i$ is intervened up to time $t$} \\
    N_{i}(t) &: \text{number of times node $i$ is non-intervened up to time $t$} \\
    \Sigma_{i,a} &: \text{autocorrelation matrix for parents of node $i$ under intervention $a$} 
\end{align*}

\section{Proofs of Our Regret Results} \label{appendix:proofs}

\paragraph{Proof of Lemma~\ref{lm:bound_l_paths}.} \label{appendix:proof_lemma2}
Note that $\beta_{T}$ is independent of $t \in [T]$. For simplicity, we use $\beta$ as a shorthand for $\beta_{T}$. We will prove the desired result in three steps. Note that matrix $\bV_{i,a}(t)$ has $(d+1)$-sparse vectors in its rows and columns, based on its definition in \eqref{eq:V_ita}. This is a key property that will be used in all the following steps.

\paragraph{Step 1:} First, we derive the bounds for vectors $[\bB_{a}]_{i}$, $[\Delta_{a}(t)]_{i}$, and their norms. For any valid vector $[\bB_{a}]_{i}$ and matrix $\bV_{N,a}(t)$, we use the CS inequality to obtain
\begin{align}
    \norm{[\bB_{a}]_{i}}_{\bV_{N,a}(t)} &\leq \underset{\leq m_B}{\underbrace{\norm{[\bB_{a}]_{i}}}} \lmaxx{\bV_{N,a}(t)}{1/2} \leq m_B \lmaxx{\bV_{N,a}(t)}{1/2}   \label{eq:lm2_aux1_max} \ . 
\end{align}
Also, noting that $[\Delta_{a}(t)]_{i} = [\Delta_{a}(t)]_{i} \odot \mathbf{1}(\Pa(i))$, we have
\begin{align}
        \norm{[\Delta_{a}(t)]_{i}}_{\bV_{i,a}(t)} &\geq \norm{[\Delta_{a}(t)]_{i} \odot \mathbf{1}(\Pa(i))} \lminn{\bV_{i,a}(t)}{1/2} = \norm{[\Delta_{a}(t)]_{i}} \lminn{\bV_{i,a}(t)}{1/2} \ . \label{eq:lm2_aux1_min} 
\end{align}
Using the conditions of the lemma and the inequalities above, we also have
\begin{align}
     \norm{[\Delta_{a}(t)]_{i}} \overset{\eqref{eq:lm2_aux1_min}}&{\leq} \underset{\leq \beta}{\underbrace{\norm{[\Delta_{a}(t)]_{i}}_{\bV_{i,a}(t)}}} \lminn{\bV_{i,a}(t)}{-1/2} \leq \beta \lminn{\bV_{i,a}(t)}{-1/2} \ . \label{eq:lm2_aux3}  
\end{align}
Note that this lemma provides the result for the $\bV_{N,a}(t)$-norm. Using \eqref{eq:lm2_aux1_max}, \eqref{eq:lm2_aux1_min}, and the lemma conditions, we have
\begin{align}
   \norm{[\Delta_{a}(t)]_{i}}_{\bV_{N,a}(t)} \overset{\eqref{eq:lm2_aux1_max}}&{\leq}  \norm{[\Delta_{a}(t)]_{i}} \lmaxx{\bV_{N,a}(t)}{1/2} \\
    \overset{\eqref{eq:lm2_aux1_min}}&{\leq} \underset{\leq \beta}{\underbrace{\norm{[\Delta_{a}(t)]_{i}}_{\bV_{i,a}(t)}}} \lminn{\bV_{i,a}(t)}{-1/2} \lmaxx{\bV_{N,a}(t)}{1/2} \\
    &\leq \beta \lminn{\bV_{i,a}(t)}{-1/2} \lmaxx{\bV_{N,a}(t)}{1/2} \ . \label{eq:lm2_aux5}
\end{align}
Next, we use the CS inequality, $|\Pa(i)|\leq d+1$, and $\norm{[\bB_{a}]_{i}}\leq m_B$ to obtain
\begin{align}
    \sum_{j \in \Pa(i)} [\bB_{a}]_{j,i} \overset{{\rm (CS)}}&{\leq} \sqrt{|\Pa(i)| \sum_{j \in \Pa(i)} \big([\bB_{a}]_{j,i}\big)^2} \leq  \norm{[\bB_{a}]_{i}} \sqrt{d+1} \leq  m_B \sqrt{d+1}\ . \label{eq:lm2_aux4}
\end{align}
For the error vector $\Delta_{a}(t)$, we use \eqref{eq:lm2_aux1_min} and \eqref{eq:lm2_aux3} to obtain
\begin{align}
    \sum_{j \in \Pa(i)} [\Delta_{a}(t)]_{j,i} \overset{{\rm (CS)}}&{\leq} \sqrt{(d+1) \sum_{j \in \Pa(i)} \big([\Delta_{a}(t)]_{j,i}\big)^2}  = \sqrt{d+1} \norm{ [\Delta_{a}(t)]_i} \\ 
    \overset{\eqref{eq:lm2_aux1_min}}&{\leq} \sqrt{d+1} \norm{[\Delta_{a}(t)]_{i}}_{\bV_{i,a}(t)} \lminn{\bV_{i,a}(t)}{-1/2}  \\
    \overset{\eqref{eq:lm2_aux3}}&{\leq} \sqrt{d+1} \beta \lminn{\bV_{i,a}(t)}{-1/2} \ . \label{eq:lm2_aux6} 
\end{align}

\paragraph{Step 2:} In \eqref{eq:lm2_aux5}, we have the term $\lminn{\bV_{i,a}(t)}{-1/2}$. Let us define 
\begin{align}
     S_{a}(t) \triangleq \max_{i \in [N]} \lminn{\bV_{i,a}(t)}{-1/2} = \frac{1}{\min_{i \in [N]} \lminn{\bV_{i,a}(t)}{1/2}} \ . \label{eq:S_at}
\end{align}
By the definitions in \eqref{eq:def_delta_ait} and \eqref{eq:def_delta_alit}, $\bB_a(t) = \Delta_{a}(t) + \bB_{a}$, and $\Delta_{a}^{(\ell)}(t) = [\Delta_{a}(t) + \bB_{a}]^{\ell} - \bB_{a}^{\ell}$. Therefore, each term in the binomial expansion of $\Delta_{a}^{(\ell)}(t)$ is a product that consists of $\Delta_{a}(t)$ and $\bB_{a}$ factors. For any $\ell \in [L+1]$ and $k \in [\ell] \cup \{0\}$, there are $\binom{\ell}{k}$ terms that contains the $\Delta_{a}(t)$ factor $k$ times and the $\bB_{a}$ factor $(\ell-k)$ times. We denote the set of such product terms by
\begin{align}
    \mcH_{\ell,k} \triangleq \{ H : \text{$H$ has $\Delta_{a}(t)$ factor $k$ times and $\bB_{a}$ factor $\ell-k$ times} \} \ .
\end{align}
For instance, set $\mcH_{3,1}$ consists of $\binom{3}{1}=3$ terms:
\begin{align}
    \mcH_{3,1} = \{\Delta_{a}(t) \bB_{a} \bB_{a}, \bB_{a} \Delta_{a}(t) \bB_{a}, \bB_{a} \bB_{a} \Delta_{a}(t) \} \ .
\end{align}
Note that $\mcH_{\ell,0}=\{\bB_{a}^{\ell}\}$, which cancels out the $\bB_{a}^{\ell}$ term in the expansion of $\Delta_{a}^{(\ell)}(t)$. Therefore, sets $\mcH_{\ell,1},\dots,\mcH_{\ell,\ell}$ contain all valid products consisting of $\bB_{a}$ and $\Delta_{a}(t)$. Hence,
\begin{align}
    \Delta_{a}^{(\ell)}(t) = \sum_{k=1} \sum_{H \in \mcH_{\ell,k}} H \label{eq:delta_expansion_sets} \ .
\end{align}
In this step, by induction, we show that for any $\ell\geq 1$, and $k \in [\ell] \cup \{0\}$, 
\begin{align}
    \norm{[H]_{i}}_{\bV_{N,a}(t)} \leq  (d+1)^{\frac{\ell-1}{2}} m_B^{\ell-k} \beta^k S_{a}^{k}(t) \lmaxx{\bV_{N,a}(t)}{1/2} \ , \quad  \forall H \in \mcH_{\ell,k} \ . \label{eq:bound_lk}
\end{align}
Consider $\ell=1$. For $k=0$, we have $\mcH_{1,0}=\{\bB_{a}\}$, and from \eqref{eq:lm2_aux1_max} we have $\norm{[\bB_{a}]_{i}}_{\bV_{N,a}(t)} \leq m_B \lmaxx{\bV_{N,a}(t)}{1/2}$. For $k=1$, $\mcH_{1,1} = \{\Delta_{a}(t)\}$, and 
\begin{align}
    \norm{[\Delta_{a}(t)]_{i}}_{\bV_{N,a}(t)} \overset{\eqref{eq:lm2_aux5}}&{\leq} \beta \lminn{\bV_{i,a}(t)}{-1/2} \lmaxx{\bV_{N,a}(t)}{1/2} \overset{\eqref{eq:S_at}}{\leq} \beta S_{a}(t) \lmaxx{\bV_{N,a}(t)}{1/2} \ .
\end{align}
Therefore, \eqref{eq:bound_lk} is correct for $\ell=1$. Suppose that it is correct for $1,\dots,\ell-1$ values, for $\ell \geq 2$. Consider a product term $H \in \mcH_{\ell,k}$, for some $ k \in [\ell] \cup \{0\}$. The first factor of $H$ is either $\bB_{a}$ or $\Delta_{a}(t)$, and we analyze the induction step for each of these possibilities separately. 
\begin{enumerate}
    \item If $H$ starts with $\bB_{a}$, represent it by $H = \bB_{a} {\bar H}$, where $\bar H \in \mcH_{\ell-1,k}$ and $k \in [\ell-1] \cup \{0\}$. Using the induction assumption for the elements of set $\mcH_{\ell-1,k}$ we obtain
    \begin{align}
        \norm{[H]_i}_{\bV_{N,a}(t)}^2 &= \norm{(\bB_{a} \bar H)_i}_{\bV_{N,a}(t)}^2 \\
        &=  \sum_{u,v \in \Pa(i)}[\bB_{a}]_{u,i}[\bB_{a}]_{v,i} {\bar H}_{v}^{\top} \bV_{N,a}(t) {\bar H}_{u}  \\
        \overset{\rm (CS)}&{\leq} (d+1) \sum_{u \in \Pa(i)} ([\bB_{a}]_{u,i})^2 \norm{{\bar H}_{u}}_{\bV_{N,a}(t)}^2 \\
        \overset{\eqref{eq:bound_lk}}&{\leq} (d+1)^{\ell-1} m_B^{2(\ell-k-1)}  \beta^{2k} S_{a}^{2k}(t) \lmaxx{\bV_{N,a}(t)}{} \sum_{u \in \Pa(i)} ([\bB_{a}]_{u,i})^2   \\
        &= (d+1)^{\ell-1} m_B^{2(\ell-k-1)} \beta^{2k} S_{a}^{2k}(t) \lmaxx{\bV_{N,a}(t)}{} \underset{\leq m_B^2}{\underbrace{\norm{[\bB_{a}]_{i}}^2}}   \\
        &\leq  (d+1)^{\ell-1} m_B^{2(\ell-k)} \beta^{2k} S_{a}^{2k}(t) \lmaxx{\bV_{N,a}(t)}{} \ . \label{eq:bound_lk_induction_case1}
    \end{align}
    \item If $H$ starts with $\Delta_{a}(t)$ represent it by $H = \Delta_{a}(t) {\bar H}$, where $\bar H \in \mcH_{\ell-1,k-1}$ and $k \in [\ell]$. Similarly to the first case, we have
    \begin{align}
        \norm{[H]_i}_{\bV_{N,a}(t)}^2 &= \norm{[\Delta_{a}(t) \bar H]_{i}}_{\bV_{N,a}(t)}^2 \\
        &=  \sum_{u,v \in \Pa(i)} [\Delta_{a}(t)]_{u,i} [\Delta_{a}(t)]_{v,i} [\bar H]_{v}^{\top} \bV_{N,a}(t) [\bar H]_{u}  \\
        \overset{\rm (CS)}&{\leq} (d+1) \sum_{u \in \Pa(i)} (\Delta_{a}(t)]_{u,i})^2 \norm{[\bar H]_{u}}_{\bV_{N,a}(t)}^2 \\
        \overset{\eqref{eq:bound_lk}}&{\leq} (d+1)^{\ell-1}  m_B^{2(\ell-k)} \beta^{2k-2}  S_{a}^{2k-2}(t) \lmaxx{\bV_{N,a}(t)}{} \sum_{u \in \Pa(i)} ([\Delta_{a}(t)]_{u,i})^2  \\
        &= (d+1)^{\ell-1} m_B^{2(\ell-k)} \beta^{2k-2} S_{a}^{2k-2}(t) \lmaxx{\bV_{N,a}(t)}{} \norm{[\Delta_{a}(t)]_{i}}^2  \\
        \overset{\eqref{eq:lm2_aux3}}&{\leq} (d+1)^{\ell-1} m_B^{2(\ell-k)} \beta^{2k-2} S_{a}^{2k-2}(t)  \lmaxx{\bV_{N,a}(t)}{} \beta^2 \underset{\leq S_{a}^{2}(t)}{\underbrace{\lminn{\bV_{i,a}(t)}{-1}}}   \\
        &\leq (d+1)^{\ell-1} m_B^{2(\ell-k)} \beta^{2k} S_{a}^{2k}(t) \lmaxx{\bV_{N,a}(t)}{} \ . \label{eq:bound_lk_induction_case2} 
    \end{align}
\end{enumerate}
Taking the square-roots of both sides in \eqref{eq:bound_lk_induction_case1} and \eqref{eq:bound_lk_induction_case2} yields
\begin{align}
     \norm{[H]_i}_{\bV_{N,a}(t)}&\leq  (d+1)^{\frac{\ell-1}{2}} m_B^{\ell-k} \beta^{k} S_{a}^{k}(t)  \lmaxx{\bV_{N,a}(t)}{1/2} \ , \label{eq:bound_lk_final}
\end{align}
which is the desired inequality for all $k \in [\ell] \cup \{0\}$. This completes the proof of induction.

\paragraph{Step 3:} Recall the binomial expansion of $\Delta_{a}^{(\ell)}(t)$ and focus on its $i$-th column:
\begin{align}
    \big[\Delta_{a}^{(\ell)}(t)\big]_{i} \overset{\eqref{eq:delta_expansion_sets}}&{=} \sum_{k=1}^{\ell} \sum_{H \in \mcH_{\ell,k}} [H]_i \ , \label{eq:delta_expansion_sets_i} 
\end{align}
in which we aim to bound $\bV_{N,a}(t)$-norm of each $[H]_{i}$ term in \eqref{eq:bound_lk_final}. The eigenvalues of $\bV_{i,a}(t)$ are always at least $1$, which makes $S_{a}(t)\leq 1$ by definition. Therefore, for all $H \in \bigcup_{1\leq k \leq \ell} \mcH_{\ell,k}$, we can replace $S_{a}^{k}(t)$ in \eqref{eq:bound_lk_final} by $S_{a}(t)$, 
\begin{align}
    S_{a}^{k}(t) &\leq S_{a}(t) \ , \quad \forall k \in [\ell] \ , \\
    \mbox{and} \qquad \norm{[H]_i}_{\bV_{N,a}(t)} \overset{\eqref{eq:bound_lk_final}}&{\leq} (d+1)^{\frac{\ell-1}{2}} m_B^{\ell-k} \beta^{k} S_{a}(t)  \lmaxx{\bV_{N,a}(t)}{1/2} \  .  \label{eq:H_i_bound}
\end{align}
The final result follows by using \eqref{eq:H_i_bound} for each of the $2^{\ell}-1$ term in the sum \eqref{eq:delta_expansion_sets_i} as follows
\begin{align}
    \norm{\big[\Delta_{a}^{(\ell)}(t)\big]_{i}}_{\bV_{N,a}(t)} &= \norm{\sum_{k=1}^{\ell} \sum_{H \in \mcH_{\ell,k}} [H]_i}_{\bV_{N,a}(t)} \\
    &\leq \sum_{k=1}^{\ell} \sum_{H \in \mcH_{\ell,k}} \norm{[H]_i}_{\bV_{N,a}(t)} \\
    \overset{\eqref{eq:H_i_bound}}&{\leq} (d+1)^{\frac{\ell-1}{2}}  S_{a}(t) \lmaxx{\bV_{N,a}(t)}{1/2} \sum_{k=1}^{\ell} |\mcH_{\ell,k}|  m_B^{\ell-k}\beta^{k}  \\
    &= (d+1)^{\frac{\ell-1}{2}}  S_{a}(t) \lmaxx{\bV_{N,a}(t)}{1/2} \sum_{k=1}^{\ell} \binom{\ell}{k}  m_B^{\ell-k}\beta^{k}  \\
    &= (d+1)^{\frac{\ell-1}{2}}  S_{a}(t) \lmaxx{\bV_{N,a}(t)}{1/2} m_B^{\ell} \left(\left(\frac{\beta}{m_B}+1\right)^{\ell}-1\right) \\
    &< (d+1)^{\frac{\ell-1}{2}}  S_{a}(t) \lmaxx{\bV_{N,a}(t)}{1/2} (\beta+m_B)^{\ell} \\
    &= (d+1)^{\frac{\ell-1}{2}} (\beta+m_B)^{\ell} \big[\lmaxx{\bV_{N,a}(t)}{1/2} \max_{i \in [N]} \lminn{\bV_{i,a}(t)}{-1/2} \big] \ .
\end{align}
\endproof

{\noindent \textbf{Proof of Corollary \ref{corollary_lemma2}}.} \label{appendix:proof_corollary3}
Note that while proving $\eqref{eq:bound_lk_final}$, we have only used the CS inequality, along with the properties $\norm{[\bB_{a}]_{i}} \leq~1$, $\norm{[\Delta_{a}(t)]_{i}}_{\bV_{i,a}(t)}\leq \beta$, $\norm{[\bB_{a}]_{i}}_0 \leq d+1$, and $\norm{[\Delta_{a}(t)]_{i}}_0 \leq d+1$. Therefore, for a matrix $\bA$ that satisfies the same conditions that $\bB_{a}$ does, and matrix $\Delta_{\bA}(t)$ that satisfies the same conditions that $\Delta_{a}(t)$ does, the result in \eqref{eq:corollary_lemma2} holds for $[\Delta_{\bA}(t)]_{N}$ following similar steps. \endproof
The next result relates the singular values of a convex combination of positive semidefinite matrices. The subsequent corollary will be useful when we consider the singular values of the sum of the autocorrelation matrices later.
\begin{lemma}
\label{lm:singular_bounds}
 Let $\bA_1, \bA_2 \ldots \bA_n \in \R^{d \times d}$ be positive semidefinite matrices, and denote the minimum and maximum singular values of $\bA_i$ by $\smin{\bA_i}$ and $\smax{\bA_i}$ respectively. Also, let $\alpha_1, \alpha_2 \dots \alpha_n \geq 0: \sum_{i=1}^{n} \alpha_i =1$. Then,
 \begin{align}
    \smax{\sum_{i=1}^{n} \alpha_i \bA_i} &\leq \max_{i \in [n]} \smax{\bA_i} \ , \\
    \mbox{and} \qquad \smin{\sum_{i=1}^{n} \alpha_i \bA_i} &\geq \min_{i \in [n]} \smin{\bA_i} \ .
\end{align}
\end{lemma} 
\begin{proof}
For the positive definite matrix $\bA \in \mathbb{R}^{d\times d}$, its maximum singular value is
\begin{align}
    \smax{\bA} = \max_{x \in \mathbb{R}^{d}: \norm{x}=1} x^\top \bA x \ .
\end{align}
Therefore, we have
\begin{align}
      \smax{\sum_{i=1}^{n} \alpha_i \bA_i} &= \max_{x \in \mathbb{R}^{d}: \norm{x}=1} x^\top \left(\sum_{i=1}^{n} \alpha_i \bA_i \right)x  \leq \sum_{i=1}^{n} \alpha_i \smax{\bA_i}  \leq \max_{i \in [n]} \smax{\bA_i} \ .
\end{align}
Similar arguments with inequalities flipped and using the definition of minimum singular value yields the second result of the lemma.
\end{proof}
\begin{corollary}\label{corollary1}
Consider a sequence of interventions $\{a_s: s \in [t]\}$. Then, 
\begin{align}
    \smax{\sum_{s=1}^t \mathbbm{1}_{\{i \in a_s\}}\Sigma_{i,a_s}} &\leq \left(\sum_{s=1}^t \mathbbm{1}_{\{i \in a_s\}}\right)\kappa_{i,\max} \ , \label{eq:small_corr_1} \\
    \mbox{and} \qquad  \smin{\sum_{s=1}^t \mathbbm{1}_{\{i \in a_s\}}\Sigma_{i,a_s}} &\geq \left(\sum_{s=1}^t \mathbbm{1}_{\{i \in a_s\}}\right)\kappa_{i,\min} \ . \label{eq:small_corr_2}
\end{align}
\end{corollary}
\begin{proof} 
Recall that $N_{i}(t) \overset{\eqref{eq:def_N_it}}{=} \sum_{s=1}^{t} \mathbbm{1}_{\{i \in a_s\}}$. If $N_{i}(t)=0$, each of the quantities in the corollary is zero. If $N_{i}(t)>0$, the coefficients $\left\{ \frac{\mathbbm{1}_{\{i \in a_1\}}}{N_{i}(t)},\dots, \frac{\mathbbm{1}_{\{i \in a_t\}}}{N_{i}(t)} \right\}$ constitute a sequence that sums up to $1$.
Then, we can apply Lemma \ref{lm:singular_bounds} and the definitions of $\kappa_{i,\max}, \kappa_{i,\min}$ to obtain
\begin{align}
     \smax{\sum_{s=1}^t \frac{\mathbbm{1}_{\{i \in a_s\}}}{N_{i}(t)} \Sigma_{i,a_s}}  &\leq \max_{s \in [t]} \smax{\Sigma_{i,a_s}} \leq \kappa_{i,\max} \ , \label{eq:small_corr_proof_1} \\
     \mbox{and} \qquad \smin{\sum_{s=1}^t \frac{\mathbbm{1}_{\{i \in a_s\}}}{N_{i}(t)} \Sigma_{i,a_s}}  &\geq \min_{s \in [t]} \smin{\Sigma_{i,a_s}} \geq \kappa_{i,\min} \ . \label{eq:small_corr_proof_2}
\end{align}
By dividing both sides of these two inequalities by $N_{i}(t)$, \eqref{eq:small_corr_proof_1} and \eqref{eq:small_corr_proof_2} imply \eqref{eq:small_corr_1} and \eqref{eq:small_corr_2}.
\end{proof}

\begin{lemma} \label{lm:smin_simple_bound}
Consider matrices $\bD$ and $\bA$ that satisfy 
\begin{align}\label{eq:smin_simple_bound_condition}
    \norm{\bD^\top \bD - \bA} \leq \zeta \ . 
\end{align}
Then we have,
\begin{align}
    \smax{\bD} &\leq \sqrt{\smax{\bA}} + \frac{\zeta}{\sqrt{\smax{\bA}}}  \ , \\
    \mbox{and} \qquad \smin{\bD} &\geq \max\left\{0, \sqrt{\smin{\bA}} - \frac{\zeta}{\sqrt{\smin{\bA}}}  \right\}   \ .
\end{align}
Equivalently, if
\begin{align}
    \smax{\bD} \geq \sqrt{\smax{\bA}} + \frac{\zeta}{\sqrt{\smax{\bA}}} \ \  \text{or} \ \ \smin{\bD} \leq \max\left\{0, \sqrt{\smin{\bA}} - \frac{\zeta}{\sqrt{\smin{\bA}}}  \right\} \ ,
\end{align}
then $ \norm{\bD^\top \bD - \bA} \geq \zeta$.
\end{lemma}
\begin{proof}
We prove it via bounding $\norm{\bD x}^2$. For vector $x$ that satisfies $\norm{x}=1$ we have
\begin{align}
        |\norm{\bD x}^2 - x^\top \bA x| &= |\langle(\bD^\top \bD - \bA)x, x \rangle| \overset{\rm (CS)}{\leq} \smax{\bD^\top \bD - \bA} \overset{\eqref{eq:smin_simple_bound_condition}}{\leq} \zeta \ .
\end{align}
We have the following immediate conclusions for \eqref{eq:smin_simple_bound_condition}:
\begin{align}
    \smin{\bA}-\zeta \leq \min_{x} x^\top \bA x - \zeta &\leq \norm{\bD x}^2 \leq \max_{x} x^\top \bA x + \zeta \leq \smax{\bA} + \zeta \ ,  \\
    \norm{\bD x} &\leq \sqrt{\smax{\bA}+\zeta} \ , \qquad  \forall x ,  \\
    \smax{\bD} &\leq  \sqrt{\smax{\bA}+\zeta} \leq \sqrt{\smax{\bA}} + \frac{\zeta}{\sqrt{\smax{\bA}}} \ , \\
    \norm{\bD x}^2 &\geq \max\left\{0, \smin{\bA}-\zeta\right\} \ ,  \\
    \norm{\bD x} &\geq \sqrt{ \max\left\{0, \smin{\bA} - \zeta\right\}}  \ , \quad  \forall x  ,  \\
    \smin{\bD} &\geq \sqrt{ \max\left\{0, \smin{\bA} - \zeta\right\}}  \\
    &\geq \max\left\{0, \sqrt{\smin{\bA}} - \frac{\zeta}{\sqrt{\smin{\bA}}} \right\} \ .
\end{align}
For the second statement of the lemma, denote the events in the lemma by $\mcZ_1 \triangleq \{ \norm{\bD^\top \bD - \bA} \leq \zeta \}$ and 
\begin{align} 
    \mcZ_2 \triangleq \Biggl\{ &\smax{\bD} \leq \sqrt{\smax{\bA}} + \frac{\zeta}{\sqrt{\smax{\bA}}} \quad \mbox{and} \notag \\
    & \quad \smin{\bD} \geq \max\biggl\{0, \sqrt{\smin{\bA}} - \frac{\zeta}{\sqrt{\smin{\bA}}}  \biggr\} \Biggr\} \ .
\end{align}
In the first step we showed $\mcZ_1 \subseteq \mcZ_2$, which implies $\mcZ_2^{\C} \subseteq \mcZ_1^{\C}$.
\end{proof}

\paragraph{Freedman's inequality.} \citet[Theorem 1.6]{freedman1975tail} is a martingale extension of Bernstein-type concentration inequalities. An earlier extension of Freedman's inequality to matrix martingales is given by \citet[Theorem 1.2]{oliveira2009concentration}. The result of \citet{oliveira2009concentration} requires a uniform bound $\smax{\bZ(k)}\leq R$ while the final result similarly involves $\lmax{\bY(k)}$. \citet[Theorem 1.2]{tropp2011freedman} achieves a similar concentration result for $\lmax{\bY(k)}$ while requiring $\lmax{\bZ(k)}\leq R$, which is a weaker condition compared to $\smax{\bZ(k)} \leq R$. However, when we define our matrix martingale sequence $\bY(k)$ in our proofs, we seek bounds for $\smax{\bY(k)}$. Therefore, we derive a parallel result to \citet[Theorem 1.2]{tropp2011freedman} with a uniform bound condition on $\smax{\bZ(k)}$ and a final result with $\smax{\bY(k)}$.

\begin{lemma}[Matrix Freedman.]\label{lm:modified_freedman} 
 Consider a matrix martingale $\{\bY(k) : k = 0,1,\dots\}$ whose values are self-adjoint matrices with dimension $n$, and let $\{ \bZ(k) : k =1,2,\dots \}$ be the difference sequence. Assume that the difference sequence is uniformly bounded in the sense that
\begin{align}
    \smax{\bZ(k)} \leq R \ , \quad \text{almost surely} \quad \forall k \in \mathbb{N}^{+} \ .
\end{align}
Define the predictable quadratic variation process of the martingale as
\begin{align}
    \bW(k) \triangleq  \sum_{j=1}^k \E[\bZ(j)^2 \med \mcF_{j-1}] \ , \quad \forall k \in \mathbb{N}^{+} \ .
\end{align}
Then, for all $\varepsilon \geq 0$ and $\sigma^2 >0$,
\begin{align} 
    \P\{ \exists k :  \smax{\bY(k)} \geq \varepsilon \ \ \text{and} \ \ \norm{\bW(k)}\leq \sigma^2 \} &\leq n \exp \left\{ \frac{-\varepsilon^2/2}{\sigma^2 + R\varepsilon/3} \right\} \\
    &\leq n \exp \left\{ -\frac{3}{8} \min\left(\frac{\varepsilon^2}{\sigma^2}, \frac{\varepsilon}{R}\right) \right\} \ . 
\end{align}
\end{lemma}
\begin{proof}
The proof for the most part follows from the steps of \citet[Theorem 1.2]{tropp2011freedman}, and we present only the necessary changes to obtain the desired result. For a $c: (0,\infty) \rightarrow [0,\infty]$ function and positive number $\theta$, the real-valued function of two self-adjoint matrices is defined as
\begin{align}
    G_\theta(\bY,\bW) \triangleq {\mathrm tr} \exp(\theta \bY - c(\theta)\bW) \ .
\end{align}
\citet[Lemma 2.2]{tropp2011freedman} shows that if $\lmax{\bY}\geq t$ and $\lmax{\bW} \leq w$, then for all $\theta > 0$, $G_\theta(\bY,\bW) \geq \exp(\theta t - c(\theta)w)$. We change the conditions to $\smax{\bY} \geq t$ and $\smax{\bW}\leq w$, and the proof of \citet[Lemma 2.2]{tropp2011freedman} follows through to show $G_\theta(\bY,\bW) \geq \exp(\theta t - c(\theta)w)$ for all $\theta > 0$. By changing all instances of $\lmax{\bY(k)}$ and $\lmax{\bW(k)}$ to $\smax{\bY(k)}$ and $\smax{\bW(k)}$ in the proof of \citet[Theorem 1.2]{tropp2011freedman}, the desired result in Lemma \ref{lm:modified_freedman} follows directly.

Finally, note that zero-padding a $n \times n$ matrix with additional $(N-n)$ zero rows and $(N-n)$ zero columns leaves its maximum singular value and maximum eigenvalue unchanged. Therefore, the final result will hold for a matrix martingale $\bY(k)$ with dimension $N$ that has only a $n \times n$ non-zero submatrix. 
\end{proof}

{\noindent \textbf{Proof of Lemma \ref{lm:error_event_int}}.} \label{appendix:proof_lemma7}
We prove the result for $\P(\mcE_{i,n}(t))$ and the result for $\P(\mcE^{*}_{i,n}(t))$ will follow similarly. The core of the proof is using Freedman's concentration inequality for matrix martingales.

\paragraph{Martingale construction.} Let us consider node $i$. We define two interventional and observational martingale sequences for node~$i$. Let $\mcF_{s-1} \triangleq \sigma(a_1,X(1),\dots,a_{s-1},X(s-1),a_s)$ denote the filtration for $s \in [T]$. Define $\Sigma_{i,a_t}$ as the autocorrelation matrix of $X_{\Pa(i)}$, which is distributed according to $\mathbb{P}_{a_t}$. Furthermore, define
\begin{align}
    \bZ_{i}(s) &\triangleq \mathbbm{1}_{\{i \notin a_s\}} \left( X_{\Pa(i)}(s) X_{\Pa(i)}^\top(s) - \Sigma_{i,a_s} \right) \ , \ \forall s \in [T] \ , \\
    \bY_{i}(k) &\triangleq \sum_{s=1}^k \bZ_{i}(s) \ , \ \forall k \in [T]  \label{eq:martingale_seq_obs} \ , \\
    \mbox{and} \qquad \bW_{i}(k) &\triangleq \sum_{s=1}^k \E[\bZ_{i}^2(s) \med \mcF_{s-1}] \ , \ \forall k \in [T]  \label{eq:W_ik_obs} \ .
\end{align}
$\bZ_i(k)$ is the difference sequence and $\bW_i(k)$ is the predictable quadratic variation of the process. We show that the $\bY_{i}(k)$ sequence is a martingale, i.e., $\E[\bY_{i}(k) \med \mcF_{k-1}] = \bY_{i}(k-1)$.
\begin{align}
    \E[\bY_{i}(k) \med \mcF_{k-1}] &= \E[\bY_{i}(k-1) + \bZ_{i}(k) \med \mcF_{k-1}] \\
    &= \bY_{i}(k-1) + \E\left[\mathbbm{1}_{\{i \notin a_k\}} \left( X_{\Pa(i)}(k) X_{\Pa(i)}^\top (k) - \Sigma_{i,a_k} \right) \med \mcF_{k-1}\right] \ . \label{eq:martingale_mid1}
\end{align}
Action $a_k$ is $\mcF_{k-1}$-measurable. Hence, the randomness in the expectation after conditioning on $\mcF_{k-1}$ is induced by $X(k) \sim \mathbb{P}_{a_k}$. Therefore, the expected value in \eqref{eq:martingale_mid1} is zero by definition of $\Sigma_{i,a}$ in \eqref{eq:second_moment_definition}, and the sequence $\bY_{i}(k)$ defined in \eqref{eq:martingale_seq_obs} is a martingale. 

\paragraph{Defining the events.} We will use $n$ to denote realizations of $N_{i}(t)$. For all $n \in [T]$, recall $\varepsilon_n$ in \eqref{eq:def_varepsilon_n}, and also define $\sigma_n^2$ as
\begin{align}
    \sigma_n^2 &\triangleq 2m^4 n \ , \quad \forall n \in [T] \label{eq:def_sigma2_n} \ . 
\end{align}  
Then, we define the following events for each triplet $\{(i,t,n) :i \in [N], t \in [T], n \in [t] \}$:
\begin{align}
    \mcU_{i,n}(t) &\triangleq \{N_{i}(t) = n\} \ , \\
    \mcY_{i,n}(t) &\triangleq \{\smax{\bY_{i}(t)} \geq \varepsilon_n\} \ , \\
    \mcQ_{i,n}(t) &\triangleq \{\norm{\bW_{i}(t)} \leq \sigma_n^2\} \ , \\
    \mcD_{i,n}(t) &\triangleq \biggl\{ \smin{\bD_{i}(t)}\leq \max\left\{0,\sqrt{n \kappa_{\min}} - \frac{\varepsilon_n}{\sqrt{n \kappa_{\min}}}\right\} \quad \text{or} \notag  \\  & \qquad \quad  \smax{\bD_{i}(t)}\geq \sqrt{n \kappa_{\max}} + \frac{\varepsilon_n}{\sqrt{n\kappa_{\min}}}  \biggr\} \ .
\end{align}
We will show the desired result, that is $ \P(\mcD_{i,n}(t), \mcU_{i,n}(t)) \leq (d+1) \exp \left( -\frac{3\alpha^2}{16} \right)$ in four steps.
\paragraph{Step 1: Show that $\P(\mcQ_{i,n}(t) \med \mcU_{i,n}(t))=1$.}
The summands of $\bW_{i}(t)$ are the conditional expectations of the following $\bZ_{i}^2(s)$ terms
\begin{align}
    \bZ_{i}^2(s) &= \mathbbm{1}_{\{i \notin a_s\}}\left([X_{\Pa(i)}(s) X_{\Pa(i)}^{\top}(s)]^2 - 2X_{\Pa(i)}(s) X_{\Pa(i)}^{\top}(s)\Sigma_{i,a_s} + \Sigma_{i,a_s}^2  \right) \ . \label{eq:err_step1_3}
\end{align}
Note that $\Sigma_{i,a_s}$ is $\mcF_{s-1}$-measurable. Hence,
\begin{align}
    \E[X_{\Pa(i)}(s) X_{\Pa(i)}^{\top}(s)\Sigma_{i,a_s} \med \mcF_{s-1}] &= \Sigma_{i,a_s} \E[X_{\Pa(i)}(s) X_{\Pa(i)}^{\top}(s)] = \Sigma_{i,a_s}^2 \ . \label{eq:err_step1_4}
\end{align}
Using Assumption~\ref{assumption:boundedness}, we also have
\begin{align}
    \norm{[X_{\Pa(i)}(s) X_{\Pa(i)}^{\top}(s)]^2} &\leq \norm{X_{\Pa(i)}(s) X_{\Pa(i)}^{\top}(s)}^2 \leq m^4 \ . \label{eq:err_step1_5}
\end{align}
Taking the norm of the expected values on both sides in \eqref{eq:err_step1_3}, we obtain
\begin{align}
    \norm{\E[\bZ_{i}^2(s) \med \mcF_{s-1}]} \overset{\eqref{eq:err_step1_4}}&{=} \norm{\E[\mathbbm{1}_{\{i \notin a_s\}} (X_{\Pa(i)}(s) X_{\Pa(i)}^{\top}(s))^2 \med \mcF_{s-1}] - \Sigma_{i,a_s}^2 \E[\mathbbm{1}_{\{i \notin a_s\}} \med \mcF_{s-1} ]} \ . \label{eq:err_step1_6} \\
    &\leq \norm{\E[\mathbbm{1}_{\{i \notin a_s\}} (X_{\Pa(i)}(s) X_{\Pa(i)}^{\top}(s))^2 \med \mcF_{s-1}]} + \underset{\leq m^4}{\underbrace{\Sigma_{i,a_s}^2}} \E[\mathbbm{1}_{\{i \notin a_s\}} \med \mcF_{s-1}] \\
    \overset{\eqref{eq:err_step1_5}}&{\leq} 2m^4 \E[\mathbbm{1}_{\{i \notin a_s\}} \med \mcF_{s-1}]  \label{eq:err_step1_7} \\
    &= 2m^4 \mathbbm{1}_{\{i \notin a_s\}} \ , \label{eq:err_step1_7x}
\end{align}
Note that \eqref{eq:err_step1_7x} is correct since $ \mathbbm{1}_{\{i \notin a_s\}}$ is $\mcF_{s-1}$-measurable.
Subsequently, $\norm{\bW_{i}(t)}$ is bounded by
\begin{align}
    \norm{\bW_{i}(t)} &=  \norm{\sum_{s=1}^t \E[\bZ_{i}^2(s) \med \mcF_{s-1}]} \\ 
    &\leq \sum_{s=1}^{t} \norm{ \E[\bZ_{i}^2(s) \med \mcF_{s-1}]} \\
    \overset{\eqref{eq:err_step1_7x}}&{\leq} \sum_{s=1}^{t}2m^4 \mathbbm{1}_{\{i \notin a_s\}} \label{eq:norm_w_1}\\
    &= 2m^4  N_{i}(t) \ . \label{eq:err_step1_8}
\end{align}
Given that under the event $\mcU_{i,n}(t)$ we have $N_{i}(t) = n$, we obtain
\begin{align}
    \norm{\bW_{i}(t)} \overset{\eqref{eq:err_step1_8}}&{\leq} 2m^4 N_{i}(t) = 2m^4n = \sigma_n^2 \ , \\
    \mbox{and} \qquad \P(\mcQ_{i,n}(t) \med \mcU_{i,n}(t)) &= 1 \label{eq:err_step1_final} \ ,
\end{align}
since event $\mcU_{i,n}(t)$ implies event $\mcQ_{i,n}(t)$.

\paragraph{Step 2: Show that $\P(\mcY_{i,n}(t) \med \mcD_{i,n}(t), \mcU_{i,n}(t)) = 1$.}
Let us define $\bA = \sum_{s=1}^{t} \mathbbm{1}_{\{i \notin a_s\}} \Sigma_{i,a_s}$. From the definition of the martingale sequence $\bY_{i}(k)$ in \eqref{eq:martingale_seq_obs} we have
\begin{align}
    \bY_{i}(t) \overset{\eqref{eq:martingale_seq_obs}}&{=} \sum_{s=1}^{t} \mathbbm{1}_{\{i \notin a_s\}} \left( X_{\Pa(i)}(s) X_{\Pa(i)}^\top(s) - \Sigma_{i,a_s} \right) \\
    \overset{\eqref{eq:D_it_obs}}&{=}  \bD_{i}^{\top}(t) \bD_{i}(t) - \sum_{s=1}^{t} \mathbbm{1}_{\{i \notin a_s\}} \Sigma_{i,a_s} \\ 
    &=   \bD_{i}^{\top}(t) \bD_{i}(t) - \bA \ .
\end{align}
Given that under the event $\mcU_{i,n}(t)$ we have $N_{i}(t) = n$, Corollary \ref{corollary1} indicates that
\begin{align}
    \smax{\bA} = \smax{ \sum_{s=1}^{t} \mathbbm{1}_{\{i \notin a_s\}} \Sigma_{i,a_s}}  &\leq N_{i}(t) \kappa_{\max} = n \kappa_{\max} \ , \label{eq:err_step2_1} \\
    \mbox{and} \qquad \smin{\bA} = \smin{ \sum_{s=1}^{t} \mathbbm{1}_{\{i \notin a_s\}} \Sigma_{i,a_s}}  &\geq N_{i}(t) \kappa_{\min} = n \kappa_{\min} \ . \label{eq:err_step2_2}
\end{align}
Therefore, the event $\mcD_{i,n}(t)$ implies that at least one of the following two inequalities is correct:
\begin{align}
    \smin{\bD_{i}(t)} \leq \max\left\{0,\sqrt{n \kappa_{\min}} - \frac{\varepsilon_n}{\sqrt{n \kappa_{\min}}}\right\} \overset{\eqref{eq:err_step2_1}}&{\leq} \max\left\{0, \sqrt{\smin{\bA}} - \frac{\varepsilon_n}{\sqrt{\smin{\bA}}} \right\} \ , \label{eq:err_step2_3} \\
    \mbox{and} \quad  \smax{\bD_{i}(t)} \geq \sqrt{n \kappa_{\max}} + \frac{\varepsilon_n}{\sqrt{n \kappa_{\min}}} \overset{\eqref{eq:err_step2_2}}&{\geq} \sqrt{\smax{\bA}} + \frac{\varepsilon_n}{\smin{\bA}}  \label{eq:err_step2_4} \ . 
\end{align}
Given the events $\mcU_{i,n}(t)$ and $\mcD_{i,n}(t)$, we invoke the second statement of Lemma \ref{lm:smin_simple_bound} to obtain
\begin{align}
    \norm{\bD_{i}^{\top}(t) \bD_{i}(t) - \bA} \geq \varepsilon_n \ ,
\end{align}
which is the event $\mcY_{i,n}(t)$. Therefore, we have 
\begin{align}
    \P(\mcY_{i,n}(t) \med \mcD_{i,n}(t), \mcU_{i,n}(t)) = 1 \ . \label{eq:err_step2_5}
\end{align}

\paragraph{Step 3: Show that $\P(\mcY_{i,n}(t),\mcQ_{i,n}(t)) \leq (d+1) \exp\left(-\frac{3\alpha^2}{16}\right)$.} The norm of the difference sequence $\bZ_{i}(s)$ for martingale $\bY_{i}(k)$ is bounded as
\begin{align}
    \norm{\bZ_{i}(s)} &= \norm{\mathbbm{1}_{\{i \notin a_s\}}\left( X_{\Pa(i)}(s) X_{\Pa(i)}^{\top}(s) - \Sigma_{i,a_s} \right)} \\
    &\leq \norm{ X_{\Pa(i)}(s) X_{\Pa(i)}^{\top}(s) - \Sigma_{i,a_s} } \\
    &\leq \norm{ X_{\Pa(i)}(s) X_{\Pa(i)}^{\top}(s)} + \norm{\Sigma_{i,a_s} } \\
    &\leq \norm{X_{\Pa(i)}(s)}^2 + \underset{\leq m^2}{\underbrace{\kappa_{i,\max}}} \\
    & \leq 2m^2 \ . \label{eq:err_step1_2}
\end{align}
Next, we apply Lemma \ref{lm:modified_freedman} (matrix Freedman) with $R=2m^2$ to obtain
\begin{align}
\P(\mcY_{i,n}(t),\mcQ_{i,n}(t))&= \P\{\smax{\bY_{i}(t)} \geq \varepsilon_n \ \ \text{and} \ \ \norm{\bW_{i}(t)}  \leq \sigma_n^2 \}  \\
    &\leq  \P\{ \exists k \in [T] :  \smax{\bY_{i}(k)} \geq \varepsilon_n \ \ \text{and} \ \ \norm{\bW_{i}(k)}\leq \sigma_n^2 \} \\
    &\leq  (d+1) \exp \left( -\frac{3}{8} \min\left\{\frac{\varepsilon_n^2}{\sigma_n^2}, \frac{\varepsilon_n}{2m^2}\right\} \right) \\
    &=  (d+1) \exp \left( -\frac{3}{8} \min\left\{\frac{\alpha^2 \max\left\{n,\alpha^2\right\}}{2n}, \frac{\alpha \max\left\{\sqrt{n},\alpha\right\}}{2}\right\} \right) \\
    &=  (d+1) \exp \left( -\frac{3\alpha^2}{16} \right) \ . \label{eq:err_step3_1}
\end{align}
Finally, note that it can be easily verified that in both cases of  $\alpha\geq~\sqrt{n}$ and $\alpha < \sqrt{n}$ we have
\begin{align}
    \frac{\alpha^2}{2}=\min\left\{\frac{\alpha^2 \max\left\{n,\alpha^2\right\}}{2n}, \frac{\alpha \max\left\{\sqrt{n},\alpha\right\}}{2}\right\}\ ,
\end{align}
which, in turn, implies \eqref{eq:err_step3_1}.

\paragraph{Step 4: Show that $\P(\mcD_{i,n}(t),\mcU_{i,n}(t)) \leq (d+1) \exp\left(-\frac{3\alpha^2}{16}\right)$.} Now we are ready to combine the last three steps and establish the desired result. Since $\P(\mcY_{i,n}(t),\mcQ_{i,n}(t))\leq (d+1) \exp\left(-\frac{3\alpha^2}{16}\right)$, it suffices to show that $\P(\mcD_{i,n}(t),\mcU_{i,n}(t)) \leq \P(\mcY_{i,n}(t),\mcQ_{i,n}(t))$. First, due to the Step 1 result $\P(\mcQ_{i,n}(t) \med \mcU_{i,n}(t)) = 1$, we have
\begin{align}
    \P(\mcQ_{i,n}(t) \med \mcY_{i,n}(t), \mcU_{i,n}(t)) &= \P(\mcQ_{i,n}(t) \med \mcU_{i,n}(t)) = 1 \ , \label{eq:err_step4_1} \\
    \mbox{and} \quad \P(\mcY_{i,n}(t),\mcU_{i,n}(t),\mcQ_{i,n}(t)) &= \P(\mcY_{i,n}(t), \mcU_{i,n}(t)) \P(\mcQ_{i,n}(t) \med \mcY_{i,n}(t), \mcU_{i,n}(t)) \\
    \overset{\eqref{eq:err_step4_1}}&{=} \P(\mcY_{i,n}(t), \mcU_{i,n}(t)) \ . \label{eq:err_step4_2} 
\end{align}
Furthermore, using the Step 3 result $\P(\mcY_{i,n}(t),\mcQ_{i,n}(t)) \leq (d+1) \exp\left(-\frac{3\alpha^2}{16}\right)$, we have
\begin{align}
    \P(\mcY_{i,n}(t),\mcU_{i,n}(t)) \overset{\eqref{eq:err_step4_2}}&{=} \P(\mcY_{i,n}(t),\mcU_{i,n}(t),\mcQ_{i,n}(t)) \\ 
    &\leq \P(\mcY_{i,n}(t),\mcQ_{i,n}(t)) \\
    &\leq (d+1) \exp \left( -\frac{3\alpha^2}{16} \right) \ . \label{eq:err_step4_3} 
\end{align}
Next, using the Step 2 result $\P(\mcY_{i,n}(t) \med \mcD_{i,n}(t), \mcU_{i,n}(t))=1$, we obtain
\begin{align}
    \P(\mcY_{i,n}(t),\mcD_{i,n}(t),\mcU_{i,n}(t)) &= \P(\mcD_{i,n}(t),\mcU_{i,n}(t)) \P(\mcY_{i,n}(t) \med \mcD_{i,n}(t),\mcU_{i,n}(t)) \\
    \overset{\eqref{eq:err_step2_5}}&{=} \P(\mcD_{i,n}(t),\mcU_{i,n}(t)) \ . \label{eq:step4_4} 
\end{align}
Finally, using $\P(\mcY_{i,n}(t),\mcU_{i,n}(t)) \leq (d+1) \exp\left(-\frac{3\alpha^2}{16}\right) $, we have
\begin{align}
    \P(\mcD_{i,n}(t),\mcU_{i,n}(t)) \overset{\eqref{eq:step4_4}}&{=}\P(\mcY_{i,n}(t),\mcD_{i,n}(t),\mcU_{i,n}(t)) \\ 
    &\leq \P(\mcY_{i,n}(t),\mcU_{i,n}(t)) \\
    \overset{\eqref{eq:err_step4_3}}&{\leq} (d+1) \exp \left( -\frac{3\alpha^2}{16} \right)\ ,
\end{align}
which is the desired result. The interventional counterpart result, i.e., $\P(\mcE^{*}_{i,n}(t)) \leq (d+1) \exp \left( -\frac{3\alpha^2}{16} \right)$, can be shown similarly.
\endproof

{\noindent \textbf{Proof of Lemma \ref{lm:sum_h_ita}}.} \label{appendix:proof_lemma8}
We will use the fact that $h$ is a non-increasing function. First, note that there may exist multiple nodes that achieve $\max_{i \in [N]} h(N_{i,a_t}(t))$. Without loss of generality, we select the smallest solution as $\argmax$ (or $\argmin$) when the $\max$ (or $\min$) of a function over set $[N]$ has more than one solution. Then, since $h$ is a non-increasing function, we have
\begin{align}
    \argmax_{i \in [N]} h(N_{i,a_t}(t)) &= \argmin_{i \in [N]} N_{i,a_t}(t) \ . \label{eq:h_lemma_1}
\end{align}
Also, by definition,
\begin{align}
    N_{i,a_t}(t) \overset{\eqref{eq:def_N_iat}}&{=} \mathbbm{1}_{\{i \in a_t\}} N^{*}_{i}(t) +  \mathbbm{1}_{\{i \notin a_t\}} (t-N^{*}_{i}(t)) \ .
\end{align}
Therefore, the argument of \eqref{eq:h_lemma_1} is the node that has the smallest relevant counter variable $N^{*}_{i}(t)$ for the nodes $i \in a_t$, and $(t-N^{*}_{i}(t))$ for the nodes $i \notin a_t$. We denote this node by $i_t$ as follows
\begin{align}
    i_t &\triangleq \argmax_{i \in [N]} h(N_{i,a_t}(t)) \overset{\eqref{eq:h_lemma_1}}{=} \argmin_{i \in [N]} N_{i,a_t}(t) \ .
\end{align}
Note that $i_t$ does not capture whether $i \in a_t$ or $i \notin a_t$. In other words, $i_t$ does not specify whether $N_{i_t,a_t}(t) = N^{*}_{i_t}(t)$ or $N_{i_t,a_t}(t) = N_{i_t}(t)$. Driven by addressing the challenge that it causes, define the set of time indices where $i_t=i$ for each of these two cases as follows
\begin{align}
    \mcS_{i} &\triangleq \{t \in [T] : i_t = i, \ i \notin a_t \} \ , \quad \forall i \in [N] \ , \\
    \mbox{and} \qquad \mcS^{*}_i &\triangleq \{t \in [T] : i_t = i, \ i \in a_t \} \ , \quad \forall i \in [N] \ . 
\end{align}
Subsequently, $h(N_{i_t,a_t}(t))$ becomes
\begin{align}
    \max_{i \in [N]}h(N_{i,a_t}(t)) &= h(N_{i_t,a_t}(t)) = \begin{cases} h(N^{*}_{i}(t)) \ , \quad \forall t \in \mcS^{*}_{i}(t) \\ h(N_{i}(t)) \ , \quad \forall t \in \mcS_{i}(t) \end{cases} \ .
\end{align}
Denote the elements of $\mcS_{i}$ by $S_{i,1},\dots, S_{i,|\mcS_{i}|}$. Note that until time $S_{i,n}$, the event $\{i_t=i, i \notin a_t\}$ occurs exactly $n$ times. Similarly $\{i_t=i, i \in a_t\}$ occurs $n$ times until time $S^{*}_{i,n}$. Then, 
\begin{align}
    n = \sum_{t=1}^{S_{i,n}} \mathbbm{1}_{\{i_t=i, i \notin a_t\}} &\leq \sum_{t=1}^{S_{i,n}} \mathbbm{1}_{\{i \notin a_t\}} = N_{i}(S_{i,n}) \ , \label{eq:h_side_2} \\
    \mbox{and} \qquad n = \sum_{t=1}^{S^{*}_{i,n}} \mathbbm{1}_{\{i_t=i, i \in a_t\}} &\leq \sum_{t=1}^{S^{*}_{i,n}} \mathbbm{1}_{\{i \in a_t\}} = N^{*}_{i}(S^{*}_{i,n}) \ . \label{eq:h_side_3}
\end{align}
Using the results above and noting that $h$ is a non-increasing function, we obtain
\begin{align}
    \sum_{t=1}^{T} \max_{i \in [N]}h(N_{i_t,a_t}(t)) &= \sum_{t=1}^{T} h(N_{i_t,a_t}(t)) \\
    &=\sum_{i=1}^{N} \sum_{t: t \in \mcS_{i}} h(N_{i}(t)) + \sum_{i=1}^{N} \sum_{t: t \in \mcS^{*}_i} h(N_{i}^{*}(t)) \\ 
    &= \sum_{i=1}^{N} \sum_{n=1}^{|\mcS_{i}|} \underset{\overset{\eqref{eq:h_side_2}}{\leq} h(n)}{\underbrace{h(N_{i}(S_{i,n}))}} + \sum_{i=1}^{N} \sum_{n=1}^{|\mcS^{*}_i|} \underset{\overset{\eqref{eq:h_side_3}}{\leq} h(n)}{\underbrace{h(N^{*}_{i}(S^{*}_{i,n}))}} \\
    &\leq \sum_{i=1}^{N} \sum_{n=1}^{|\mcS_{i}|} h(n) + \sum_{i=1}^{N} \sum_{n=1}^{|\mcS^{*}_i|} h(n) \ .
\end{align}
To bound the discrete sums through integrals, we define
\begin{align}
    H(y) &= \int_{x=0}^{y} h(x)dx \ , \quad y\geq 0 \ . \label{eq:Hx_definition}
\end{align}
Since $h(x)$ is a positive, non-increasing function, for any $k \in \mathbb{N}^{+}$ we have
\begin{align}
    \sum_{n=1}^{k} h(n) \leq \int_{x=1}^{k + 1} h(x)dx &\leq \int_{x=0}^{k} h(x)dx = H(k) \label{eq:h_side_4} \ ,
\end{align}
and subsequently,
\begin{align}
    \sum_{t=1}^{T} \max_{i \in [N]} h(N_{i,a_t}(t)) &\leq \sum_{i=1}^{N} \sum_{n=1}^{|\mcS_{i}|} h(n) + \sum_{i=1}^{N} \sum_{n=1}^{|\mcS^{*}_i|} h(n) \\ \overset{\eqref{eq:h_side_4}}&{\leq}  \sum_{i=1}^{N} H(|\mcS_{i}|) + \sum_{i=1}^{N} H(| \mcS^{*}_{i}|) \ . \label{eq:h_side_5}
\end{align}
We verify that $H$ is a concave function as follows
\begin{align}
    H'(n) &= h(n) - h(0) \\ 
    &= \sqrt{2}\left(\frac{1}{\max\left\{0,\sqrt{n \kappa_{\min}}-\sqrt{\tau \kappa_{\min}}\right\}+1}-1 \right) \  \\
     &= 
    \begin{cases}
    0 \ , \quad 0\leq n < \tau \\
    \sqrt{2}(\sqrt{n \kappa_{\min}}-\sqrt{\tau \kappa_{\min}}+1)^{-1} \ , \quad n \geq \tau
    \end{cases} \ , \\
    \mbox{and} \quad H''(n) &= 
    \begin{cases}
    0 \ , \quad 0 \leq n < \tau \\
    -\sqrt{\frac{\kappa_{\min}}{2n}}(\sqrt{n \kappa_{\min}}-\sqrt{\tau \kappa_{\min}}+1)^{-2} \ , \quad n \geq \tau
    \end{cases} \ .
\end{align}
Since $H''(n) \leq 0$ for the domain of $H$, $H$ is a concave function. Also, $\sum_{i=1}^{N} |\mcS_{i}| + \sum_{i=1}^{N} |\mcS^{*}_{i}| = T$. Then, we use Jensen's inequality for concave functions to obtain
\begin{align}
    \sum_{i=1}^{N} H(|\mcS_{i}|) + \sum_{i=1}^{N} H(|\mcS^{*}_{i}|) &\leq 2N \times H \left(\frac{1}{2N} \sum_{i=1}^{N} |\mcS_{i}| +  \frac{1}{2N} \sum_{i=1}^{N} |\mcS^{*}_{i}| \right) \\
    &= 2N \times H\left(\frac{T}{2N} \right) \ . \label{eq:h_side_6}
\end{align}
Next, we note that  $h(x)=\sqrt{2}$ for $x\leq \tau$ and decompose $H\left(\frac{T}{2N}\right)$ into two parts as
\begin{align}
    H\left(\frac{T}{2N}\right) &= \int_{x=0}^{\frac{T}{2N}} h(x) dx \\
    & = \int_{x=0}^{\tau} h(x) dx + \int_{x=\tau}^{\frac{T}{2N}} h(x) dx \\
    & = \sqrt{2}\tau + \int_{x=\tau}^{\frac{T}{2N}} h(x) dx \ . \label{eq:misc_late1}
\end{align}
Next, we need to bound $\int_{x=\tau}^{\frac{T}{2N}} h(x) dx$. Let us define $y \triangleq \sqrt{\tau}-\frac{1}{\sqrt{\kappa_{\min}}}$. Hence,
\begin{align}
    \int_{x=\tau}^{\frac{T}{2N}} h(x) dx &= \sqrt{2} \int_{x=\tau}^{\frac{T}{2N}} \frac{1}{\sqrt{x \kappa_{\min}}-\sqrt{\tau \kappa_{\min}}+1} dx \\
    &= \sqrt{\frac{2}{\kappa_{\min}}} \int_{x=\tau}^{\frac{T}{2N}} \frac{1}{\sqrt{x}-y} dx \\
    &= 2\sqrt{\frac{2}{\kappa_{\min}}} \left( \sqrt{\frac{T}{2N}}+y \log\left(\sqrt{\frac{T}{2N}}-y\right) - \sqrt{\tau} - y \log(\sqrt{\tau}-y) \right) \\
    &< \sqrt{\frac{2}{\kappa_{\min}}} \left(\sqrt{\frac{2T}{N}} + y\log \left(\frac{T}{2N}\right)\right) \\
    &<  \sqrt{\frac{2}{\kappa_{\min}}} \left(\sqrt{\frac{2T}{N}} +\sqrt{\tau} \log \left(\frac{T}{2N}\right)\right) \ .  \label{eq:h_side_7}
\end{align}
Finally, plugging \eqref{eq:h_side_7} into \eqref{eq:misc_late1}, we obtain the desired result
\begin{align}
    \sum_{t=1}^{T} \max_{i \in [N]} h(N_{i,a_t}(t)) &\leq 2N \times \left(\tau +  \int_{x=\tau}^{\frac{T}{2N}} h(x) dx\right) \\
    &< 2N  \left(\sqrt{2}\tau + \sqrt{\frac{2}{\kappa_{\min}}}\left(\sqrt{\frac{2T}{N}} + \sqrt{\tau} \log\left(\frac{T}{2N}\right) \right) \right) \ .
\end{align}
\endproof

{\noindent \textbf{Proof of Theorem \ref{th:regret_ts}}.} \label{appendix:proof_TS}
Before starting the proof, we comment on the changes from the frequentist setting.

\paragraph{Global lower and upper bounds for singular values.}
We have defined global bounds for singular values in \eqref{eq:kappa_min} and \eqref{eq:kappa_max} for the frequentist analysis. However, the probability measure $\P_a$ changes with respect to the sampled parameters $\bW$ for the Bayesian setting. Therefore, we need to expand the definition of lower and upper bounds for singular values to the domain of $\bW$, i.e., $\mcW$. We redefine the probability measure for an intervention $a$ by also accounting for the dependence on parameters $\bW$,
and denote it by $\P_{a}^{\bW}$. Accordingly, we define $\Sigma_{i,a}^{\bW} \triangleq \E_{X \sim \P_{a}^{\bW}}[X_{\Pa(i)} X_{\Pa(i)}^\top]$ and denote the lower and upper bounds of these moments' singular values by
\begin{align}
    \kappa_{\mathrm{max}}^{\bW} &\triangleq \max \limits_{i \in [N]} \max \limits_{a \in \mcA} \smax{\Sigma_{i,a}^{\bW}} \label{eq:kappa_W_max_ts} \ , \\
  \kappa_{\mathrm{max}}^{\mcW} &\triangleq \max_{\bW \in \mcW}  \kappa_{\mathrm{max}}^{\bW}  \label{eq:kappa_max_ts} \ ,  \\ 
     \kappa_{\mathrm{min}}^{\bW} &\triangleq \min \limits_{i \in [N]}\min \limits_{a \in \mcA} \smin{\Sigma_{i,a}^{\bW}} \ , \label{eq:kappa_W_min_ts} \\ 
    \mbox{and} \qquad \kappa_{\mathrm{min}}^{\mcW} &\triangleq \min_{\bW \in \mcW} \kappa_{\mathrm{min}}^{\bW} \label{eq:kappa_min_ts} \ .
\end{align}
Similar to the definition of $\tau$ in \eqref{eq:tau_ucb} for the frequentist setting, we define
\begin{align}\label{eq:tau_ts}
     \tau_{\bW} \triangleq \frac{\alpha^2 m^4}{(\kappa_{\mathrm{min}}^{\bW})^2} \ , \quad \mbox{and} \quad \tau_{\mcW} \triangleq \frac{\alpha^2 m^4}{(\kappa_{\mathrm{min}}^{\mcW})^2} \ .
\end{align}
We start by finding bounds on the expected regret for fixed $\bW$ using the tools that we have developed in Section \ref{sec:UCB}. Then, we analyze the terms in the result that are affected by the choice of $\bW,$ and obtain the final Bayesian regret. Consider
\begin{align}
     \E_{\epsilon} [R_{\bW}(T)] &= \E \left[ \sum_{t=1}^T \E[\mu_{a^*} - \mu_{a_t} \med \tilde\mcF_{t-1}] \right] \ , \label{eq:ts_proof_regret1}
\end{align}
where $\tilde\mcF_{t}= \sigma(a_1,X(1),\dots,a_{t},X(t))$. Note that $\tilde\mcF_{t}$ is different from $\mcF_{t}$ defined and used earlier. The next few steps are similar to those of the proof of Theorem \ref{th:regret_ucb_part1}. We use the same $\beta$ as before, i.e.,
\begin{align}
    \beta = m_B+ \sqrt{2\log(2NT)+(d+1)\log(1+m^2T/(d+1))} \ ,
\end{align}
and the upper confidence bound ${\rm UCB}(t)$ in \eqref{eq:ucb_definition}. Defining the event $\mcE_{\cap}$ similarly to \eqref{eq:ucb_conf_interval_event_union} in the proof of Theorem \ref{th:regret_ucb_part1}, we have $\P(\mcE_{\cap}^{\C})\leq \frac{1}{T}$. Next, we decompose the regret in \eqref{eq:ts_proof_regret1} as 
\begin{align}
    \sum_{t=1}^T \E[\mu_{a^*} - \mu_{a_t} \med \tilde\mcF_{t-1}]  &= \sum_{t=1}^T \E[\mathbbm{1}_{\mcE_{\cap}^{\C}}  \underset{\leq 2m}{\underbrace{(\mu_{a^*}-\mu_{a_t})}}  \med \tilde\mcF_{t-1}] + \sum_{t=1}^T \E[\mathbbm{1}_{\mcE_{\cap}} (\mu_{a^*}-\mu_{a_t}) \med \tilde\mcF_{t-1}]  \\
    &\leq 2m T \underset{\leq 1/T}{\underbrace{\P(\mcE_{\cap}^{\C}\med \tilde\mcF_{t-1})} }  + \sum_{t=1}^T \E[\mathbbm{1}_{\mcE_{\cap}} (\mu_{a^*}-\mu_{a_t}) \med \tilde\mcF_{t-1}]  \\
    &\leq 2m + \sum_{t=1}^T \E[\mathbbm{1}_{\mcE_{\cap}} (\mu_{a^*}-\mu_{a_t}) \med \tilde\mcF_{t-1}]  \ .  \label{eq:p8}
\end{align}
Thompson Sampling has the property that $\P(a^* = a \med \tilde\mcF_{t-1}) = \P(a_t = a \med \tilde\mcF_{t-1})$. Therefore, 
\begin{align}
    {\rm UCB}_{a^*}(t) &= {\rm UCB}_{a_t}(t) \ , \label{eq:ts_property_ucb} \ ,
\end{align}
and consequently,
\begin{align}
 \E[\mathbbm{1}_{\mcE_{\cap}} (\mu_{a^*}-\mu_{a_t}) \med \tilde\mcF_{t-1}] &= \E \left[\mathbbm{1}_{\mcE_{\cap}} \big( \underset{\leq 0}{\underbrace{\mu_{a^*} - {\rm UCB}_{a^*}(t)}} + {\rm UCB}_{a_t}(t) - \mu_{a_t}\big) \med \tilde\mcF_{t-1} \right] \label{eq:posterior_matching} \\
    &\leq \E \left[ \mathbbm{1}_{\mcE_{\cap}} \big({\rm UCB}_{a_t}(t) - \mu_{a_t}\big) \med \tilde\mcF_{t-1} \right] \ . \label{eq:posterior_matching_2}
\end{align}
Define ${\rm UCB}_{a}(t) = f(\tilde \bB_{a})$. By combining \eqref{eq:p8} and \eqref{eq:posterior_matching_2}, we obtain
\begin{align}
    \E_{\epsilon}[R_{\bW}(T)] &= \E \left[\sum_{t=1}^{T} \E[\mu_{a^*} - \mu_{a_t} \med \tilde\mcF_{t-1}] \right] \\
    &\leq 2m + \E\left[\sum_{t=1}^{T} \E \left[\mathbbm{1}_{\mcE_{\cap}} \big({\rm UCB}_{a_t}(t) - \mu_{a_t}\big) \med \tilde\mcF_{t-1} \right] \right] \\
    &= 2m + \E\left[\sum_{t=1}^{T} \E \left[\mathbbm{1}_{\mcE_{\cap}} (f(\tilde \bB_{a_t})-f(\bB_{a_t}))  \med \tilde\mcF_{t-1} \right] \right] \\
    &= 2m + \E\left[\sum_{t=1}^{T}  \mathbbm{1}_{\mcE_{\cap}} (f(\tilde \bB_{a_t})-f(\bB_{a_t}))  \right]  \ . \label{eq:p9}
\end{align}
Following similar steps to the proof of Theorem \ref{th:regret_ucb_part1}, we can bound the expected value in \eqref{eq:p9}, and obtain
\begin{align}
    \E_{\epsilon}[R_{\bW}(T)] &\leq 2m + 2(\beta + m_B)^{L+1} (d+1)^{\frac{L}{2}}\lambda_{\bW,T} \ , \label{eq:bound_ts_bW}
\end{align}
where
\begin{align}
    \lambda_{\bW,T} &< \frac{4g(\tau_{\bW})}{\sqrt{\kappa_{\bW,\min}}}\sqrt{NT} + 3(N+1)\tau_{\bW} g(\tau_{\bW}) + \frac{2\sqrt{2}N\sqrt{\tau_{\bW}} g(\tau_{\bW})}{\sqrt{\kappa_{\bW,\min}}} \log\left(\frac{T}{2N}\right)  \notag \\ 
    &\qquad + \frac{m}{T} + \frac{2m}{3} +  1 \ .
\end{align}
Finally, we replace $\bW$-specific terms with global $\mcW$ terms as follows
\begin{align}
    \kappa_{\mcW,\min} &\leq \kappa_{\bW,\min} \ , \\
    \tau_{\bW} \overset{\eqref{eq:tau_ts}}&{\leq} \tau_{\mcW}  \ , \\
    \sqrt{\tau_{\bW}}g(\tau_{\bW}) &\leq \sqrt{\tau_{\mcW}}g(\tau_{\mcW}) \ , \\
    \lambda_{\bW,T} &\leq \lambda_{\mcW,T} \ .
\end{align}
Since $\lambda_{\bW,T}$ is the only term in \eqref{eq:bound_ts_bW} that depends on $\bW$, we obtain
\begin{align}
     {\rm BR}(T) &= \E_{\mcW} \E_{\epsilon} [R_{\bW}(T)] \\
     &\leq 2m + 2(\beta + m_B)^{L+1} (d+1)^{\frac{L}{2}}\lambda_{\mcW,T} \ . 
\end{align}
Similar to the proof of Theorem \ref{th:regret_ucb_final}, using the same $\beta$ as in Theorem \ref{th:regret_ucb_part1}, 
and ignoring poly-logarithmic terms and constants, we obtain ${\rm BR}(T) = \tilde \mcO((d+1)^{L+\frac{1}{2}} \sqrt{NT})$.
\endproof

\section{Proof of Theorem~\ref{th:lower-bound}} \label{sec:proof-lower-bound}

\begin{figure}[t]
    \centering
    \begin{subfigure}[t]{0.4\textwidth}
        \centering
        \includegraphics[height=2in]{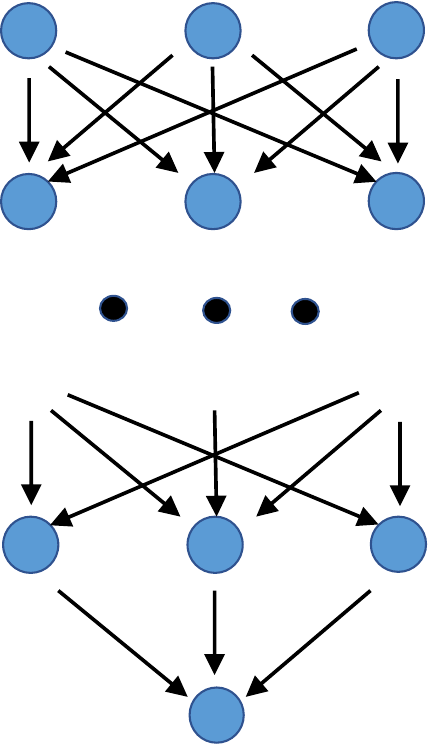}
        \caption{Hierarchical block with $d=3$.}
        \label{fig:block-hierarchical}
    \end{subfigure}
    \begin{subfigure}[t]{0.4\textwidth}
        \centering
        \includegraphics[height=2in]{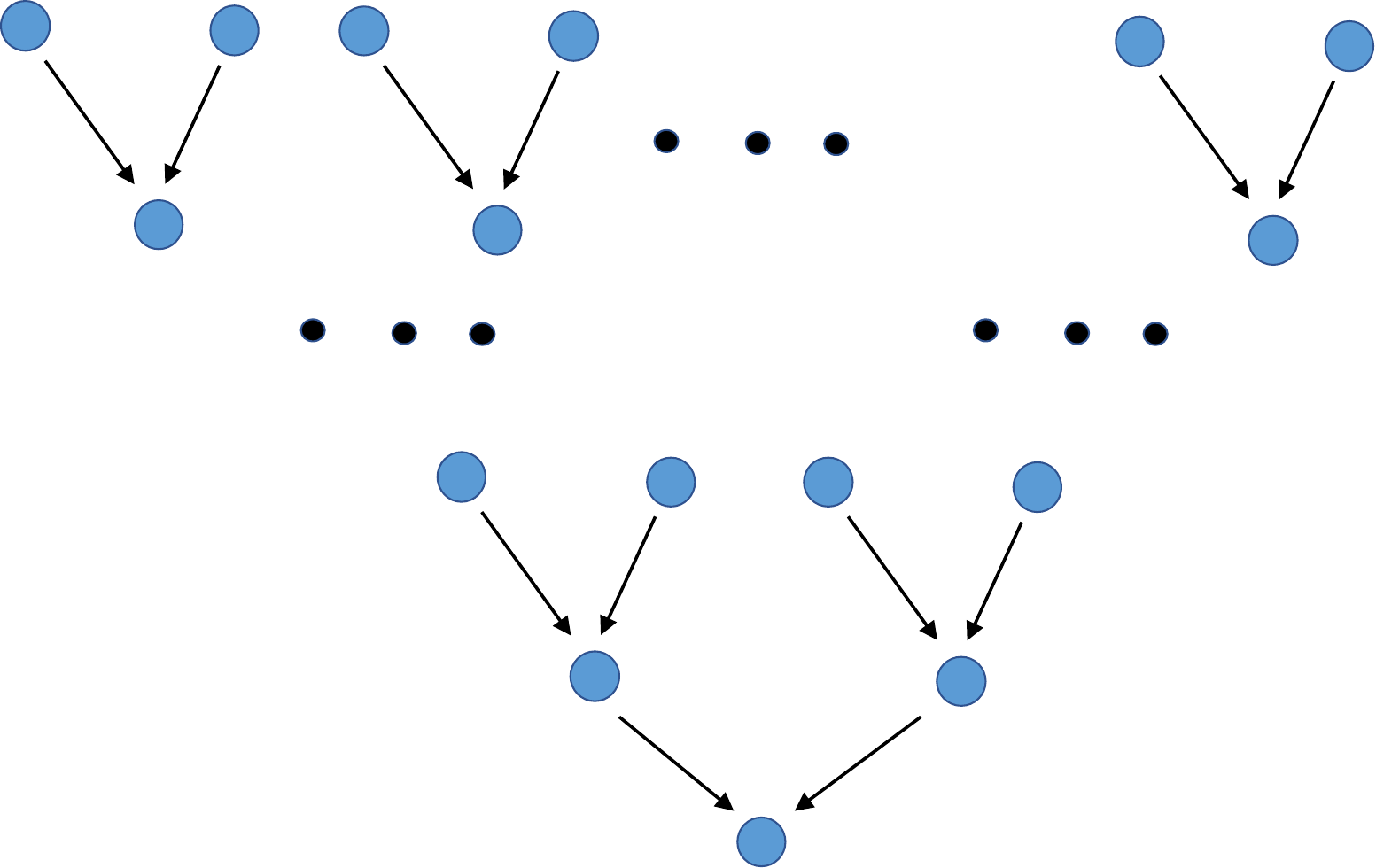}
        \caption{$d$-ary tree with $d=2$.}
        \label{fig:block-tree}
    \end{subfigure}
    \hfill
    \bigskip
    \vspace{0.1in}
    \begin{subfigure}[t]{0.5\textwidth}
        \centering
        \includegraphics[width=1\linewidth]{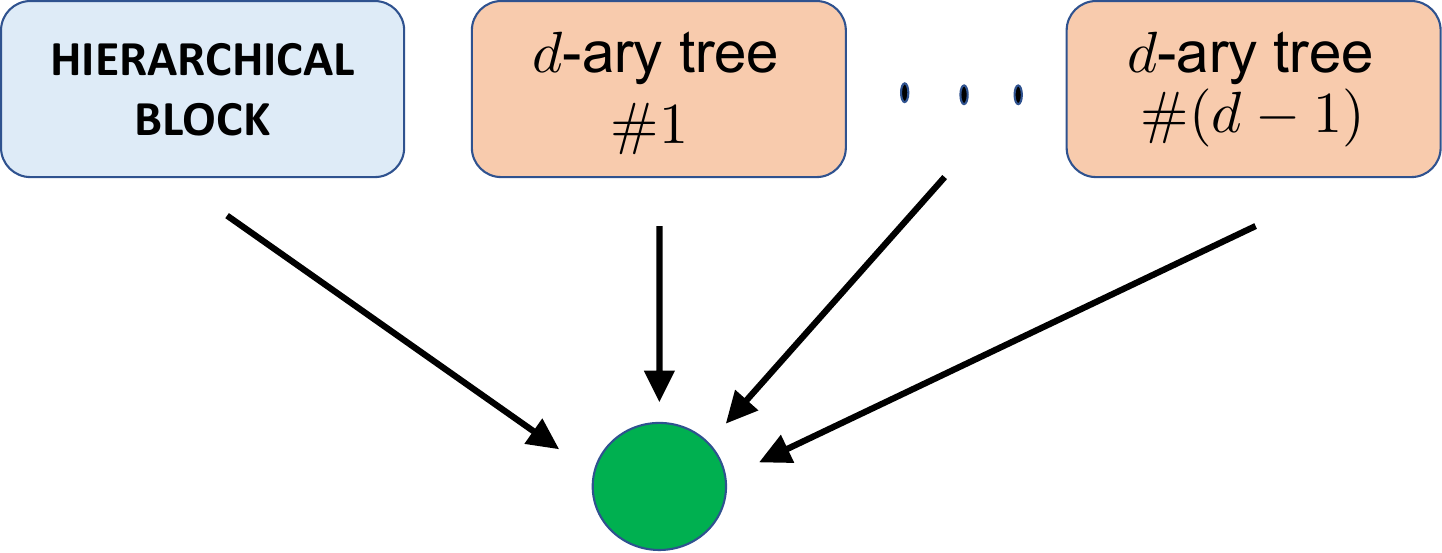}
        \caption{Causal graph instance.}
        \label{fig:lower-bound}
    \end{subfigure}    
    \caption{Sample diagrams for the graphs used in the proof of Theorem~\ref{th:lower-bound}.}
\end{figure}

\emph{Sketch of the proof.} To find a lower bound on the minimax regret, we provide a proof as follows. Let $\Pi$ be the set of all policies on set of stochastic bandit environments $\mcI$. The minimax regret is defined as
\begin{align}
    \inf_{\pi \in \Pi} \sup_{\mcI^* \in \mcI} \E_{\pi,\mcI^*}[R(T)] \ ,
\end{align}
in which $\E_{\pi,\mcI^*}[R(T)]$ denotes the expected regret of policy $\pi$ on the bandit instance $\mcI^*$. We will consider a set $\tilde \mcI$ that contains two bandit instances. By definition of minimax regret, a lower bound for the regret of any policy on $\tilde\mcI$ also is
a lower bound for the minimax regret since
\begin{align}
    \inf_{\pi \in \Pi} \sup_{\mcI^* \in \mcI} \E_{\pi,\mcI^*}[R(T)] \geq \inf_{\pi \in \Pi} \sup_{\mcI^* \in \tilde\mcI} \E_{\pi,\mcI^*}[R(T)] \ . 
\end{align}
After this intuition, the central idea of the proof is as follows. Two linear SEM causal bandit instances that differ by only one edge parameter are hard to distinguish. At the same time, we can construct them to have different optimal actions, indicating that a selection policy cannot incur small regret for both at the same time under the same data realization. Note that, the difference of the rewards, or equivalently the regrets, observed by these two bandit instances under the same action can be computed by tracing the effect of that differing edge parameter over all the paths that end at the reward node. We carefully build graphs to maximize the number of such paths for given $d$ and $L$. In this section, we provide details of these steps.

We consider two linear SEM causal bandit instances that share the same graph $\mcG$ and are parameterized by $\mcI_1\triangleq \{\bH,\bH^{*},\nu, \nu^{*},\epsilon\}$, and $\mcI_2\triangleq\{\bar \bH, \bar \bH^{*}, \bar \nu, \bar \nu^*, \epsilon\}$. For both instances, let the weights of all observational edges be $b$, and all interventional edges be $b-\delta$ such that $0 < \delta < b \leq 1$.  In other words, if for $i,j\in[N]$
\begin{align}
    (i \rightarrow j) \in \mcG\ ,
\end{align}
then 
\begin{align}
[\bH]_{i,j} = [\bar \bH]_{i,j} = b \ , \qquad \mbox{and} \qquad [\bH^*]_{i,j} = [\bar \bH^*]_{i,j} = b-\delta\ .     
\end{align}
We assume the noise terms have standard Gaussian distributions, i.e., $\epsilon_i \sim \mcN(0,1)$ for all $i \in [N]$. Furthermore, for all $k \in \{2,\dots,N\}$ set 
\begin{align}
    \nu_k = \nu^*_k = \bar \nu_k = \bar \nu^*_k = 0\ .
\end{align}
 The only difference between the two instances is the affine terms of node $1$, for which $\nu_1 = \bar \nu_1^{*} = 1$, and $\nu^*_1 = \bar \nu_1 = 1 - \delta$. Note that this parameterization can be applied to any graph $\mcG$. Next, consider a fixed bandit policy $\pi$ that generates the following filtration over time 
 \begin{align}
     \mcF_t\triangleq \{a_1,X(1),\dots,a_{t},X(t)\}\ .
 \end{align}
 The decision of $\pi$ at time $t$ is $\mcF_{t-1}$-measurable. 
 Accordingly, define $\P_t$ and $\bar\P_t$ as the probability measures induced by $\mcF_t$ by $t$ rounds of interaction between $\pi$ and the two bandit instances. When it is clear from context, we use the shorthand terms $\P$ and $\bar \P$ for $\P_T$ and $\bar\P_T$, respectively. We will show that $\pi$ cannot suffer small regret in both instances at the same time and under the same filtration $\mcF_T$.

Next, we construct our sample graph for fixed $d$ and $L$ values using the following two building blocks.
\begin{enumerate}
    \item One hierarchical block as depicted in Figure \ref{fig:block-hierarchical}, which consists of $(L-1)$ layers each with $d$ nodes. Adjacent layers are fully connected. There exists a final layer with one node fully connected to layer $L-1$.
    \item One $d$-ary tree with $L$ layers as depicted in Figure \ref{fig:block-tree}.
\end{enumerate}
  We use one hierarchical block and $(d-1)$ number of $d$-ary trees and connect their sink nodes to form a reward node. Hence, in this graph, the hierarchical block consists of $d(L-1)+1$ nodes, a $d$-ary tree consists of $\sum_{\ell=0}^{L-1} d^\ell$ nodes, and the total number of nodes is 
\begin{align}
    N &= d(L-1)+1 + (d-1)\sum_{\ell=0}^{L-1} d^\ell = d^L + d(L-1) \ .
\end{align}
The nodes in the hierarchical blocks are labeled by $\{1,\dots,d(L-1)+1\}$, beginning from the top layer. All weights are in the set $\{b,b-\delta\}$ and are positive in both bandit instances. Hence, by Lemma~\ref{lm:expected_reward}, the optimal action is the one that maximizes the value of each entry of $\bB_a$ and $\bar \bB_a$. As a result, the optimal actions are $a^* = \emptyset$ for the bandit instance $\mcI_1$ and $\bar a^* = \{1\}$ for bandit instance $\mcI_2$. Define $\mcE_{\rm lb}$ as the event in which node $1$ is intervened at least $\frac{T}{2}$ times after $T$ rounds, i.e.,
\begin{align} \label{eq:event}
    \mcE_{\rm lb} \triangleq \left\{ N^{*}_{1}(T) \geq \frac{T}{2}\right\} \ .
\end{align}
We note that the event $\mcE_{\rm lb}$ is defined on the sigma algebra defined by the filtration $\mcF_t$, that induces both $\P_t$ and $\bar \P_t$. We compute the expected instantaneous regret when node $1$ is intervened in the first bandit. Note that since $\nu_k = \nu^*_k$ and $\epsilon_k$  has a zero mean for $k \geq 2$, only the paths that start at node $1$ and end at the reward node $N$ contribute to the expected regret. Furthermore, since every weight is positive, in $\mcI_1$, we have $\mu_{\{1\}} \geq \mu_{a}$ for any $a$ that contains $1$. Therefore, for $\mcI_1$, the expected value of instantaneous regret is at least $\mu_{\emptyset}-\mu_{\{1\}}$ if $1$ is intervened. Then, by definition of $\mcE_{\rm lb}$, we have
\begin{align}
    \E_{\P}[R(t)] &= \E_{\P}\left[\sum_{t=1}^{T} r(t)\right] \\ &\geq   \E_{\P}\left[\sum_{t \in [T]: 1 \notin a_t} (\mu_{\emptyset} - \mu_{\{1\}})\right] \\
    &\geq \P(\mcE_{\rm lb}) \frac{T}{2}(\mu_{\emptyset} - \mu_{\{1\}})  \\ 
    &=  \P(\mcE_{\rm lb}) \frac{T}{2}  \delta b^L d^{L-2} \;  \ . \label{eq:lower_bound_regret_1}
\end{align}
The final equality holds because there exists $d^{L-2}$ paths from node $1$ to node $N$, and the difference between $\nu_1 - \nu^{*}_1 = \delta$ is multiplied with a $b$ factor for every edge along a path. Similarly, for $\mcI_2$, we have $\bar \mu_{\emptyset} \geq \bar \mu_{a}$ for any $a$ that does not contain $1$, and expected value of instantaneous regret is at least $\bar \mu_{\{1\}}-\bar\mu_{\emptyset}$ if $1$ is not intervened. Applying the same steps, we obtain
\begin{align}
    \E_{\bar\P}[R(t)] &= \E_{\bar\P}\left[\sum_{t=1}^{T} r(t)\right] \\ &\geq   \E_{\bar\P}\left[\sum_{t \in [T]: 1 \in a_t} (\bar\mu_{\{1\}} - \bar\mu_{\emptyset} )\right] \\
    &\geq \bar\P(\mcE_{\rm lb}^{\C}) \frac{T}{2}(\bar\mu_{\{1\}} - \bar\mu_{\emptyset} )  \\ 
    &=  \bar\P(\mcE_{\rm lb}^{\C}) \frac{T}{2}  \delta b^L d^{L-2} \;  \ . \label{eq:lower_bound_regret_2}
\end{align}
By combining \eqref{eq:lower_bound_regret_1} and \eqref{eq:lower_bound_regret_2} we have
\begin{align}
    \E_{\P}[R(t)] + \E_{\bar \P}[R(t)] \geq \frac{T}{2}\;  \delta  b^L d^{L-2} [\P(\mcE_{\rm lb})+\bar \P(\mcE_{\rm lb}^{\C})]  \ . \label{eq:lower_bound_regret_mid1}
\end{align}
By setting $b = \sqrt{1/d}$, we ensure that $m_B = 1$, and \eqref{eq:lower_bound_regret_mid1} becomes
\begin{align}
    \E_{\P}[R(t)] + \E_{\bar \P}[R(t)] \geq \frac{T}{2}\; \delta d^{\frac{L}{2}-2} [\P(\mcE_{\rm lb})+\bar \P(\mcE_{\rm lb}^{\C})] \ . \label{eq:lower_bound_regret_mid2}
\end{align}
Next, we characterize a lower bound  on $(\P(\mcE_{\rm lb})+\bar \P(\mcE_{\rm lb}^{\C}))$ that involves the Kullback-Leibler (KL) divergence between $\P$ and $\bar \P$, denoted by ${\rm D}_{\rm KL}(\P \kl \bar \P)$. For this purpose, we leverage the following theorem. 
\begin{theorem}[Bretagnolle-Huber inequality]
\label{th:pinsker}
Let $\P$ and $\bar \P$ be probability measures on the same measurable space $(\Omega,\mcF)$ and let $A \in \mcF$ be an arbitrary event. Then,
\begin{align}
    \P(A) + \bar \P(A^{\C}) \geq \frac{1}{2}\exp(-{\rm D}_{\rm KL}(\P \kl \bar \P)) \ .
\end{align}
\end{theorem}
By invoking Theorem~\ref{th:pinsker}, from \eqref{eq:lower_bound_regret_mid2} we obtain
\begin{align}
    \E_{\P}[R(t)] + \E_{\bar \P}[R(t)] \geq \frac{T}{4}\; \delta d^{\frac{L}{2}-2} \; \exp(-{\rm D}_{\rm KL}(\P \kl \bar \P)) \ . \label{eq:lower_bound_regret_mid3}
\end{align}
It remains to compute $ \exp(-{\rm D}_{\rm KL}(\P \kl \bar \P))$ to conclude our proof, for which we leverage the following result.
\begin{proposition}\label{prop:kl} The KL divergence between $\P$ and $\bar \P$ is equal to 
\begin{align}
    {\rm D}_{\rm KL}(\P \kl \bar \P) = T \delta^2 \ . 
\end{align}
\end{proposition}
\begin{proof}
Note that a Bayesian network factorizes as
\begin{align}
    \P(X_1,\dots,X_N) = \prod_{i=1}^{N} \P(X_i \med X_{\Par(i)}) \ . 
\end{align}
Additionally, the two bandit instances differ only in the mechanism of node $1$, which is a root node. Then, ${\rm D}_{\rm KL}(\P \kl \bar \P)$ can be simplified as
\begin{align}
    {\rm D}_{\rm KL}(\P \kl \bar \P) = \sum_{i=1}^{N}  {\rm D}_{\rm KL}(\P(X_i \med X_{\Par(i)})\; \| \;  \bar \P(X_i \med X_{\Par(i)})) = {\rm D}_{\rm KL}(\P(X_1) \kl \bar \P(X_1)) \ .
\end{align}
Hence, we only need to consider ${\rm D}_{\rm KL}(\P(X_1) \kl \bar \P(X_1))$ under two cases: when $1$ is observed, and $1$ is intervened. We have that
\begin{align}\label{eq:X_1}
    X_1 \sim \begin{cases}
        \mcN(\nu_1,1) \ , \quad & \mbox{ under } \P \mbox{ when } 1 \notin a \  \\
        \mcN(\nu_1-\delta,1) \ , \quad &\mbox{ under } \P \mbox{ when }  1 \in a \  \\   
        \mcN(\nu_1-\delta,1) \ , \quad &\mbox{ under } \bar\P \mbox{ when } 1 \notin a \   \\
        \mcN(\nu_1,1) \ , \quad & \mbox{ under } \bar\P \mbox{ when } 1 \notin a \ 
    \end{cases}\ .
\end{align}
By noting that
\begin{align}
    {\rm D}_{\rm KL}(\mcN(\nu_1,1) \kl \mcN(\nu_1-\delta,1)) = {\rm D}_{\rm KL}(\mcN(\nu_1-\delta,1) \kl \mcN(\nu_1,1)) = \frac{\delta^2}{2} \ ,
\end{align}
from~\eqref{eq:X_1} we obtain
\begin{align}
    & {\rm D}_{\rm KL}(\P(X_1) \kl \bar \P(X_1)) \notag \\ 
    &= \sum_{t \in [T]: 1 \notin a_t} {\rm D}_{\rm KL}(\mcN(\nu_1,1),\mcN(\nu_1-\delta,1))   + \sum_{t \in [T]: 1 \in a_t}  {\rm D}_{\rm KL}(\mcN(\nu_1-\delta,1),\mcN(\nu_1,1))   \\ 
    &= N^{*}_{1}(T) \; \frac{\delta^2}{2} +  (T-N^{*}_{1}(T)) \; \frac{\delta^2}{2} \\ 
    & =  \frac{T \delta^2}{2} \ . 
\end{align}
\end{proof}
By applying Proposition~\ref{prop:kl} on \eqref{eq:lower_bound_regret_mid3} and setting $\delta=\sqrt{2/T}$, we obtain
\begin{align}
\max\{\E_{\P}[R(t)], \E_{\bar \P}[R(t)]\} &\geq 
    \frac{1}{2}(\E_{\P}[R(t)] + \E_{\bar \P}[R(t)]) \\
    &\geq \frac{T}{8} \;  \delta d^{\frac{L}{2}-2} \; \exp(-T \delta^2 / 2 ) \\
    &= \frac{\exp(-1)}{8\sqrt{2}} \; d^{\frac{L}{2}-2} \; \sqrt{T} \ .
\end{align}
Hence, for $c = \frac{\exp(-1)}{8\sqrt{2}}$, the policy $\pi$ incurs a regret $c d^{\frac{L}{2}-2} \sqrt{T}$ in at least one of the two bandit instances. Finally, note that removing the nodes from the tree blocks of the constructed graph does not affect the analysis. Hence, $N$ can take any value in the range $\{(d(L-1)+2),\dots,d^L+d(L-1)+1\}$. \endproof

\section{Extension to Hard Interventions}\label{sec:hard}
In Section~\ref{sec:introduction}, we discussed that most of the prior work on causal bandits assumes hard interventions. Nevertheless, our formulation of soft interventions can be beneficial in practice. Furthermore, our results can be readily extended to hard interventions.

An intervention $a\in\mcA$ in this setting assigns pre-specified constant values to nodes in $a$. Denote the vector of these constant values by $Z \triangleq [z_1,\dots,z_N]$. Subsequently, a hard intervention is denoted by $a \triangleq do(X_a = Z_a)$, which means if $i \in a_t$, then $X_i(t) = Z_i$. Since the dependence of an intervened node on its parents is broken, a hard intervention on $i$ sets $[\bB^*]_{j,i}$ to zero for all $j \in [N]$, the term $[\bB^*]_{0,j}$ becomes $Z_i$, and $\epsilon_i(t)$ becomes zero. Then, \algonameUCB{} can be modified using this updated $\bB^*$ while constructing $\bB_a$'s. The analysis in the previous proofs is still valid by changing the upper bound on $\norm{X}$ and redefining $\kappa_{\min}$ and $\kappa_{\max}$ to accommodate the $Z$ terms.

\section{A Depiction of the Non-linearity}\label{appendix:non-linear}

In Lemma~\ref{lm:expected_reward}, we have shown that the reward $X_N$ is a linear function of $\epsilon$ variables but not a linear function of the edge weights. Specifically, the effect of a node on the reward is compounded via the edge weights along the paths from that node to the reward node. This renders our problem completely different from the linear bandit problem. Consider the example in Figure \ref{fig:expansion-example}. Since the dummy noise variable is $\epsilon_0=1$ and $\nu$ values are put into the dummy $0$-th row of $\bB$,  the reward node $X_5$ can be written as
\begin{figure}[b]
    \centering
    \includegraphics[width= 1.5 in]{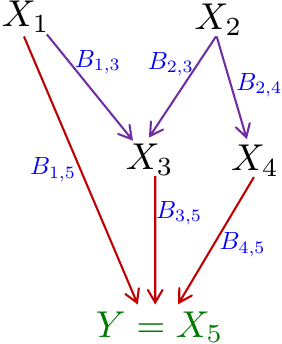}
    \caption{Sample graph with $N=5$ nodes and its edge weights.}
    \label{fig:expansion-example}
\end{figure}
\begin{align}
    X_5 &= [\bB]_{1,5} X_1 + [\bB]_{3,5} X_3 + [\bB]_{4,5} X_4 + (\nu_5 + \epsilon_5) \\
    &= [\bB]_{1,5} (\nu_1 + \epsilon_1) + [\bB]_{3,5} \left([\bB]_{1,3}(\nu_1 + \epsilon_1\right) + [\bB]_{2,3}(\nu_2+\epsilon_2) + \nu_3 + \epsilon_3) \notag \\
    &\quad + [\bB]_{4,5}\left([\bB]_{2,4}(\nu_2 + \epsilon_2) + \nu_4 + \epsilon_4\right)  + (\nu_5 + \epsilon_5)\\
    &= \underset{=1}{\underbrace{\epsilon_0}} \bigl([\bB]_{0,1}[\bB]_{1,5} + [\bB]_{0,1}[\bB]_{1,3}[\bB]_{3,5} + [\bB]_{0,2}[\bB]_{2,3}[\bB]_{3,5} + [\bB]_{0,2}[\bB]_{2,4}[\bB]_{4,5} \notag \\
    &\quad\quad+[\bB]_{0,3}[\bB]_{3,5} + [\bB]_{0,4}[\bB]_{4,5} +[\bB]_{0,5} \bigr) \notag \\
    &\quad + \epsilon_1\left([\bB]_{1,5} + [\bB]_{1,3} [\bB]_{3,5}\right)  \notag \\
    &\quad + \epsilon_2 \left([\bB]_{2,3}[\bB]_{3,5} + [\bB]_{2,4}[\bB]_{4,5}\right) \notag \\ 
    &\quad + \epsilon_3 [\bB]_{3,5} \notag \\ 
    &\quad + \epsilon_4 [\bB]_{4,5} \notag \\
    &\quad + \epsilon_5 \ .
\end{align}
We also demonstrate this non-linearity aspect via a simple simulation. Note that, we observe only vector $X$. Hence, if we are to estimate reward $X_N$ as a linear function, our only choice is to use $X$. To do so, we create a hierarchical graph with $N=10$ nodes, degree $d=3$, and $L=3$ layers. The parameters are randomly chosen similarly to the simulations in Section~\ref{sec:simulations}. We generate $5000$ training and $5000$ test samples. Then, we perform linear regression to estimate $X_N$ from $\{1,X_1,\dots,X_{N-1}\}$ using training data. We use the estimated parameters to predict $X_N$ on test data. Figure~\ref{fig:non-linear} shows the distribution of actual reward versus predicted reward on two different trials. Unsurprisingly, even for a simple model with $N=10$ nodes and $5000$ data samples, the reward is highly non-linear with respect to $X_N$, and equivalently to edge weight parameters.
\begin{figure}[t]
    \centering
    \begin{subfigure}[t]{0.4\textwidth}
        \centering
        \includegraphics[width=\linewidth]{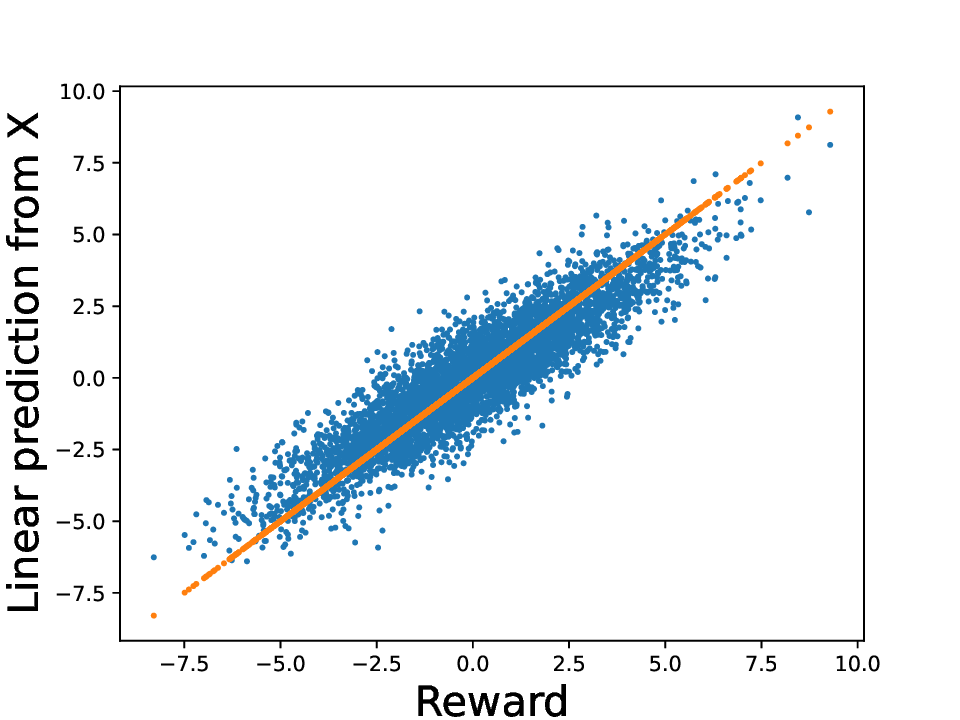}
        \caption{Example 1.}
        \label{fig:non-linear-1}
    \end{subfigure} 
    \begin{subfigure}[t]{0.4\textwidth}
        \centering
        \includegraphics[width=\linewidth]{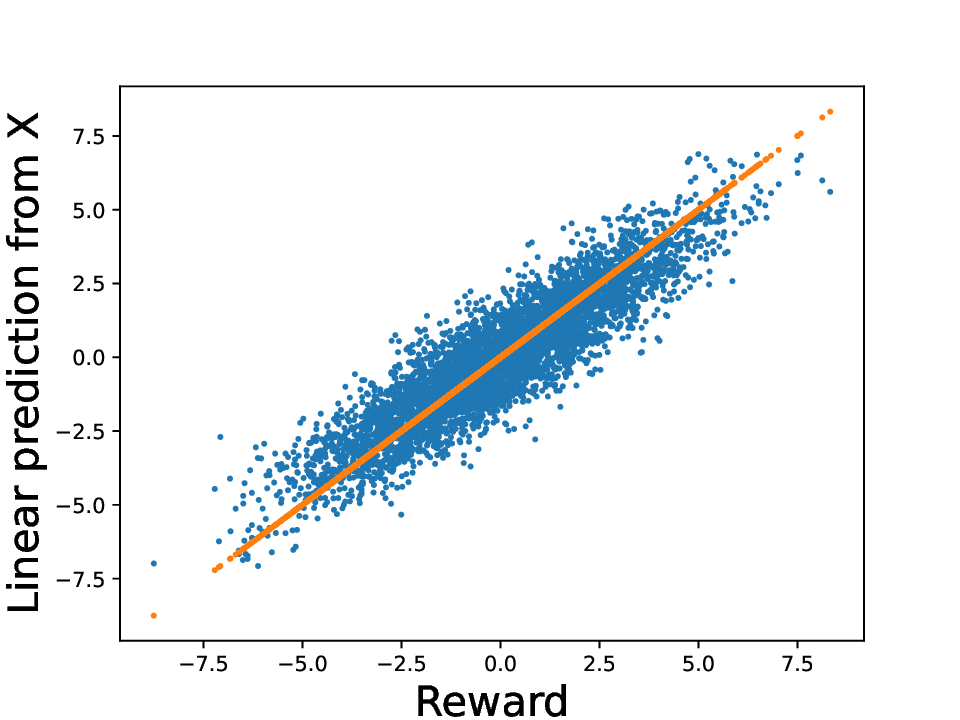}
        \caption{Example 2.}
        \label{fig:non-linear-2}
    \end{subfigure}   
    \caption{Sample results of estimating reward from $X$.}
    \label{fig:non-linear}
\end{figure}

\newpage

\bibliography{arxiv-v3}

\begin{thebibliography}{30}
\providecommand{\natexlab}[1]{#1}
\providecommand{\url}[1]{\texttt{#1}}
\expandafter\ifx\csname urlstyle\endcsname\relax
  \providecommand{\doi}[1]{doi: #1}\else
  \providecommand{\doi}{doi: \begingroup \urlstyle{rm}\Url}\fi

\bibitem[Abbasi{-}Yadkori et~al.(2011)Abbasi{-}Yadkori, P{\'{a}}l, and
  Szepesv{\'{a}}ri]{abbasi-yadkori2011}
Yasin Abbasi{-}Yadkori, D{\'{a}}vid P{\'{a}}l, and Csaba Szepesv{\'{a}}ri.
\newblock Improved algorithms for linear stochastic bandits.
\newblock In \emph{Proc. Advances in Neural Information Processing Systems},
  pages 2312--2320, Granada, Spain, December 2011.

\bibitem[Agrawal and Goyal(2013)]{agrawal2013thompson}
Shipra Agrawal and Navin Goyal.
\newblock Thompson sampling for contextual bandits with linear payoffs.
\newblock In \emph{Proc. International Conference on Machine Learning}, pages
  127--135, Atlanta, GA, June 2013.

\bibitem[Bareinboim et~al.(2015)Bareinboim, Forney, and
  Pearl]{bareinboim2015bandits}
Elias Bareinboim, Andrew Forney, and Judea Pearl.
\newblock Bandits with unobserved confounders: A causal approach.
\newblock In \emph{Proc. Advances in Neural Information Processing Systems},
  Montr{\'{e}}al, Canada, December 2015.

\bibitem[Bilodeau et~al.(2022)Bilodeau, Wang, and Roy]{bilodeau2022adaptively}
Blair Bilodeau, Linbo Wang, and Daniel~M Roy.
\newblock Adaptively exploiting $d$-separators with causal bandits.
\newblock \emph{arXiv:2202.05100}, 2022.

\bibitem[Bottou et~al.(2013)Bottou, Peters, Qui{\~n}onero-Candela, Charles,
  Chickering, Portugaly, Ray, Simard, and Snelson]{bottou2013counterfactual}
L{\'e}on Bottou, Jonas Peters, Joaquin Qui{\~n}onero-Candela, Denis~X Charles,
  D~Max Chickering, Elon Portugaly, Dipankar Ray, Patrice Simard, and
  Ed~Snelson.
\newblock Counterfactual reasoning and learning systems: The example of
  computational advertising.
\newblock \emph{Journal of Machine Learning Research}, 14\penalty0 (11), 2013.

\bibitem[Cesa-Bianchi and Lugosi(2012)]{cesa2012combinatorial}
Nicolo Cesa-Bianchi and G{\'a}bor Lugosi.
\newblock Combinatorial bandits.
\newblock \emph{Journal of Computer and System Sciences}, 78\penalty0
  (5):\penalty0 1404--1422, 2012.

\bibitem[Chu et~al.(2011)Chu, Li, Reyzin, and Schapire]{chu2011contextual}
Wei Chu, Lihong Li, Lev Reyzin, and Robert Schapire.
\newblock Contextual bandits with linear payoff functions.
\newblock In \emph{Proc. International Conference on Artificial Intelligence
  and Statistics}, Ft. Lauderdale, FL, April 2011.

\bibitem[Dani et~al.(2008)Dani, Hayes, and Kakade]{dani2008linear}
Varsha Dani, Thomas~P. Hayes, and Sham~M. Kakade.
\newblock Stochastic linear optimization under bandit feedback.
\newblock In \emph{Proc. Conference on Learning Theory}, pages 355--366,
  Helsinki, Finland, July 2008.

\bibitem[de~Kroon et~al.(2022)de~Kroon, Mooij, and Belgrave]{de2022causal}
Arnoud de~Kroon, Joris Mooij, and Danielle Belgrave.
\newblock Causal bandits without prior knowledge using separating sets.
\newblock In \emph{Conference on Causal Learning and Reasoning}, pages
  407--427, Eureka,CA, April 2022. PMLR.

\bibitem[Feng and Chen(2022)]{feng2022combinatorial}
Shi Feng and Wei Chen.
\newblock Combinatorial causal bandits.
\newblock \emph{arXiv:2206.01995}, 2022.

\bibitem[Freedman(1975)]{freedman1975tail}
David~A Freedman.
\newblock On tail probabilities for martingales.
\newblock \emph{The Annals of Probability}, 3\penalty0 (1):\penalty0 100--118,
  1975.

\bibitem[Jaber et~al.(2020)Jaber, Kocaoglu, Shanmugam, and
  Bareinboim]{jaber2020causal}
Amin Jaber, Murat Kocaoglu, Karthikeyan Shanmugam, and Elias Bareinboim.
\newblock Causal discovery from soft interventions with unknown targets:
  Characterization and learning.
\newblock In \emph{Proc. Advances in Neural Information Processing Systems},
  pages 9551--9561, December 2020.

\bibitem[Lattimore et~al.(2016)Lattimore, Lattimore, and
  Reid]{lattimore2016causal}
Finnian Lattimore, Tor Lattimore, and Mark~D Reid.
\newblock Causal bandits: Learning good interventions via causal inference.
\newblock In \emph{Proc. Advances in Neural Information Processing Systems},
  Barcelona, Spain, December 2016.

\bibitem[Lattimore and Szepesv{\'a}ri(2020)]{lattimore2020bandit}
Tor Lattimore and Csaba Szepesv{\'a}ri.
\newblock \emph{Bandit Algorithms}.
\newblock Cambridge University Press, Cambridge, UK, 2020.

\bibitem[Liu et~al.(2020)Liu, See, Ngiam, Celi, Sun, and
  Feng]{liu2020reinforcement}
Siqi Liu, Kay~Choong See, Kee~Yuan Ngiam, Leo~Anthony Celi, Xingzhi Sun, and
  Mengling Feng.
\newblock Reinforcement learning for clinical decision support in critical
  care: Comprehensive review.
\newblock \emph{Journal of Medical Internet Research}, 22\penalty0
  (7):\penalty0 e18477, July 2020.

\bibitem[Lu et~al.(2020)Lu, Meisami, Tewari, and Yan]{lu2020regret}
Yangyi Lu, Amirhossein Meisami, Ambuj Tewari, and William Yan.
\newblock Regret analysis of bandit problems with causal background knowledge.
\newblock In \emph{Proc. Conference on Uncertainty in Artificial Intelligence},
  pages 141--150, August 2020.

\bibitem[Lu et~al.(2021)Lu, Meisami, and Tewari]{lu2021causal}
Yangyi Lu, Amirhossein Meisami, and Ambuj Tewari.
\newblock Causal bandits with unknown graph structure.
\newblock In \emph{Proc. Advances in Neural Information Processing Systems},
  pages 24817--24828, December 2021.

\bibitem[Maiti et~al.(2022)Maiti, Nair, and Sinha]{maiti2022causal}
Aurghya Maiti, Vineet Nair, and Gaurav Sinha.
\newblock A causal bandit approach to learning good atomic interventions in
  presence of unobserved confounders.
\newblock In \emph{Proc. Conference on Uncertainty in Artificial Intelligence},
  Eindhoven, Netherlands, August 2022.

\bibitem[Nair et~al.(2021)Nair, Patil, and Sinha]{nair2021budgeted}
Vineet Nair, Vishakha Patil, and Gaurav Sinha.
\newblock Budgeted and non-budgeted causal bandits.
\newblock In \emph{Proc. International Conference on Artificial Intelligence
  and Statistics}, pages 2017--2025, April 2021.

\bibitem[Oliveira(2009)]{oliveira2009concentration}
Roberto~Imbuzeiro Oliveira.
\newblock {Concentration of the adjacency matrix and of the Laplacian in random
  graphs with independent edges}.
\newblock \emph{arXiv:0911.0600}, 2009.

\bibitem[Rusmevichientong and Tsitsiklis(2010)]{rusmevichientong2010linearly}
Paat Rusmevichientong and John~N Tsitsiklis.
\newblock Linearly parameterized bandits.
\newblock \emph{Mathematics of Operations Research}, 35\penalty0 (2):\penalty0
  395--411, 2010.

\bibitem[Sawant et~al.(2018)Sawant, Namballa, Sadagopan, and
  Nassif]{Sawant2018ContextualMB}
Neela Sawant, Chitti~Babu Namballa, Narayanan Sadagopan, and Houssam Nassif.
\newblock Contextual multi-armed bandits for causal marketing.
\newblock \emph{arXiv:1810.01859}, 2018.

\bibitem[Sen et~al.(2017)Sen, Shanmugam, Dimakis, and Shakkottai]{sen17}
Rajat Sen, Karthikeyan Shanmugam, Alexandros~G. Dimakis, and Sanjay Shakkottai.
\newblock Identifying best interventions through online importance sampling.
\newblock In \emph{Proc. International Conference on Machine Learning}, pages
  3057--3066, Sydney, Australia, August 2017.

\bibitem[Shen et~al.(2015)Shen, Wang, Jiang, and Zha]{shen2015portfolio}
Weiwei Shen, Jun Wang, Yu-Gang Jiang, and Hongyuan Zha.
\newblock Portfolio choices with orthogonal bandit learning.
\newblock In \emph{Proc. International Joint Conference on Artificial
  Intelligence}, Buenos Aires, Argentina, July 2015.

\bibitem[Tewari and Murphy(2017)]{tewari2017ads}
Ambuj Tewari and Susan~A. Murphy.
\newblock \emph{From Ads to Interventions: Contextual Bandits in Mobile
  Health}, pages 495--517.
\newblock Springer International Publishing, Cham, 2017.

\bibitem[Tropp(2011)]{tropp2011freedman}
Joel Tropp.
\newblock Freedman's inequality for matrix martingales.
\newblock \emph{Electronic Communications in Probability}, 16:\penalty0
  262--270, 2011.

\bibitem[Varici et~al.(2021)Varici, Shanmugam, Sattigeri, and
  Tajer]{varici2021scalable}
Burak Varici, Karthikeyan Shanmugam, Prasanna Sattigeri, and Ali Tajer.
\newblock Scalable intervention target estimation in linear models.
\newblock In \emph{Proc. Advances in Neural Information Processing Systems},
  pages 1494--1505, December 2021.

\bibitem[Xiong and Chen(2022)]{xiong2022pure}
Nuoya Xiong and Wei Chen.
\newblock Pure exploration of causal bandits.
\newblock \emph{arXiv:2206.07883}, 2022.

\bibitem[Yabe et~al.(2018)Yabe, Hatano, Sumita, Ito, Kakimura, Fukunaga, and
  Kawarabayashi]{yabe18causal}
Akihiro Yabe, Daisuke Hatano, Hanna Sumita, Shinji Ito, Naonori Kakimura,
  Takuro Fukunaga, and Ken{-}ichi Kawarabayashi.
\newblock Causal bandits with propagating inference.
\newblock In \emph{Proc. International Conference on Machine Learning}, pages
  5508--5516, Stockholm, Sweden, July 2018.

\bibitem[Zhou et~al.(2017)Zhou, Zhang, Xu, and Liang]{zhou2017large}
Qian Zhou, XiaoFang Zhang, Jin Xu, and Bin Liang.
\newblock Large-scale bandit approaches for recommender systems.
\newblock In \emph{Proc. International Conference on Neural Information
  Processing}, pages 811--821, Guangzhou, China, November 2017.

\end{thebibliography}
\bibliographystyle{plainnat}

\end{document}